\RequirePackage{fix-cm}
\documentclass[11pt]{article}

\usepackage{graphicx}
\usepackage{titling}
\usepackage[T1]{fontenc}
\usepackage[latin9]{inputenc}
\usepackage{geometry}
\usepackage{color}
\usepackage{algorithm2e}
\usepackage{amsmath}
\usepackage{multirow}
\usepackage{amsthm}
\usepackage{dsfont}
\usepackage{amssymb}
\usepackage{bbm}
\usepackage{float}
\usepackage{amsthm} 
\usepackage{undertilde}
\usepackage{booktabs}
\usepackage{setspace}
\usepackage{adjustbox}
\usepackage{graphicx}
\usepackage{subcaption}
\usepackage{float}

\usepackage{esint}
\usepackage{hyphenat}
\usepackage[authoryear]{natbib}
\usepackage[unicode=true,
 bookmarks=false,
 breaklinks=true,pdfborder={0 0 0},pdfborderstyle={},backref=section,colorlinks=true]
 {hyperref}
\hypersetup{
 citecolor=blue, linkcolor=blue}

\makeatletter
\usepackage{changebar}
\providecommand{\lyxadded}[3]{}

\renewcommand{\lyxadded}[3]{%
  {\protect\cbstart\color{lyxadded}{}#3\protect\cbend}%
}

\newtheorem{proposition}{Proposition}

\usepackage{verbatim}
\usepackage{array}
\usepackage{tabularx}
\usepackage{ragged2e}
\newcolumntype{P}[1]{>{\RaggedRight\arraybackslash}p{#1}}

\usepackage[labelfont={bf}, 
            textfont={it},
            margin=1cm,
            font=small]{caption}
\usepackage[final]{microtype}
\usepackage{amsthm}
\newtheorem{theorem}{Theorem}
\newtheorem{lemma}{Lemma}

\newtheorem{definition}{Definition}
\DeclareMathOperator*{\argmin}{arg\,min}

\newtheorem{example}{Example}

\usepackage[multiple]{footmisc}

\usepackage{titling}
\usepackage{xpatch}
\makeatletter
  {\begin{titlepage}\def\@thanks{}}%
  {\end{titlepage}}
\xpatchcmd\titlepage{\setcounter{page}\@ne}{}{}{}
\xpatchcmd\endtitlepage{\setcounter{page}\@ne}{}{}{}

\LinesNumbered
\RestyleAlgo{algoruled} 

\makeatother

\newcommand{\R}{\mathbb{R}}
\newcommand{\E}{\mathbb{E}}
\newcommand{\Pbb}{\mathbb{P}}

\newcommand{\D}{\mathcal{D}}

\newcommand{\prob}{\mathbbm{P}}

\begin{document}
\global\long\def\ddd{,\ldots,}%

\global\long\def\bB{\mathbf{B}}%

\global\long\def\bS{\mathbf{S}}%

\global\long\def\hS{\widehat{\mathbf{S}}}%

\global\long\def\hTheta{\widehat{\boldsymbol{\Theta}}}%

\global\long\def\bTheta{\boldsymbol{\Theta}}%

\global\long\def\bSigma{\boldsymbol{\Sigma}}%

\global\long\def\tr{\mathrm{tr}}%

\global\long\def\bX{\mathbf{X}}%

\global\long\def\bZ{\mathbf{Z}}%

\global\long\def\bx{\mathbf{x}}%

\global\long\def\bz{\mathbf{z}}%

\global\long\def\bs{\mathbf{s}}%

\global\long\def\bm{\mathbf{m}}%

\global\long\def\R{\mathbb{R}}%

\global\long\def\E{\mathbb{E}}%

\global\long\def\sn{\sum_{i=1}^{N}}%

\global\long\def\ind{\mathbb{I}}%

\global\long\def\diag{\operatorname{diag}}%

\global\long\def\sgn{\operatorname{sgn}}%

\global\long\def\prox{\operatorname{prox}}%

\global\long\def\Proj{\operatorname{Proj}}%

\global\long\def\argmin{\operatorname*{arg\,min}}%

\title{Isotonic Conformal Inference}

\author{
Daniel Bensimon \thanks{Equal contribution;   Department of Mathematics and Statistics,
McGill University; Mila}
\and
Zhiquan Liu \thanks{Equal contribution; School of Statistics and Data Science,
Nankai University}
\and
Sean Xiang Yu \thanks{Eli Lilly and Company}
\and
Eric D. Kolaczyk \thanks{Department of Mathematics and Statistics, McGill University; Mila}
\and
Archer Y. Yang 
\thanks{Department of Mathematics and Statistics,  McGill University; Mila; Corresponding author: archer.yang.yi@gmail.com}
}
\maketitle
\begin{abstract}
A predictor with small overall bias can still exhibit systematic bias
conditional on its predicted value, limiting its reliability for
downstream uncertainty quantification and candidate screening.
We propose Isotonic Conformal Inference, a unified framework for
prediction and multiple testing that separates calibration from
inference and exploits the level set structure of an isotonic
predictor. This structure enables the computation of localized
conformal quantiles and $p$-values while avoiding repeated isotonic
refitting across candidate responses. For prediction, we introduce
Split Isotonic Conformal Prediction (sICP) and Transductive Isotonic
Conformal Prediction (tICP), establish prediction-conditional coverage
guarantees for both methods, and characterize their finite-sample
and asymptotic self-calibration properties. For selection, we
introduce Isotonic Conformal Selection (ICS), which provides
finite-sample stratum-conditional false discovery rate control.
Experiments on simulated regression problems,
a healthcare-utilization dataset, and a high-throughput ADMET
drug discovery dataset demonstrate the practical performance of
the framework and substantially lower computational costs
relative to full transductive baselines.
\end{abstract}

\noindent \textbf{Keywords:} 

\section{Introduction}
Machine learning models are increasingly deployed in high-stakes domains such as drug discovery \citep{guFacilitatingStructurebasedDrug2026b,
vamathevanApplicationsMachineLearning2019}, clinical risk prediction
\citep{topolHighperformanceMedicineConvergence2019}, and large language model deployment \citep{singhalLargeLanguageModels2023}. In these applications, the numerical value of a prediction can directly influence downstream decision making. Systematic overprediction or underprediction can therefore lead to costly mistakes, even when overall prediction error is small. Reliable deployment therefore requires both calibrated point predictions and a principled assessment of the uncertainty surrounding them.

Conformal prediction (CP) \citep{vovkAlgorithmic2005} provides a general framework for constructing prediction intervals with finite-sample coverage guarantees. It is model-free and requires only exchangeability of the data. CP has been applied in various high-stakes settings such as drug discovery \citep{boscLargeScaleComparison2019}, LLM deployment \citep{quachConformalLanguageModeling2023}, and disease progression \citep{sreenivasanConformalPredictionEnables2025}. In these applications, both the point prediction and its associated prediction interval may inform decisions. In this work, we build on the CP framework to address the reliability of both. Specifically, we seek calibrated point predictions and prediction intervals that attain the desired coverage conditional on the calibrated prediction. More formally, we target the following objectives,

\begin{alignat}{2}
&\text{Self-Calibrated Prediction:}
&\qquad&
\mathbb{E}\!\left[Y_{n+1}\mid \hat f(X_{n+1})\right]
= \hat f(X_{n+1}) \quad \text{a.s.}
\tag{i}\label{eq:self-calibration}\\
&\text{Prediction-Conditional Validity:}
&\qquad&
\prob\!\left(
Y_{n+1}\in\widehat{C}(X_{n+1})
\,\middle|\,\hat f(X_{n+1})
\right)\ge 1-\alpha.
\tag{ii}\label{eq:prediction-conditional}
\end{alignat}
Objective~(\ref{eq:self-calibration}) requires the point prediction to be unbiased conditional on its own
value. For example, a model that assigns a prediction of ten to observations whose mean response is eight systematically overpredicts at that prediction level. Objective~(\ref{eq:prediction-conditional}) requires prediction intervals to attain the target coverage level conditional on the calibrated prediction. Together, these objectives motivate a two-step approach to reliable uncertainty quantification: first calibrate the point prediction, then construct prediction intervals with coverage conditional on that prediction. Calibration can also improve interval efficiency: when intervals are centered on point prediction, correcting systematic over or under prediction can reduce the size of the interval needed to attain the target coverage.

Standard CP provides a useful starting point for objective~(\ref{eq:prediction-conditional}) but its  coverage guarantees hold only marginally, that is on average over the calibration data and the test point. These intervals therefore achieve the target overall coverage while systematically undercovering among observations assigned particular predicted values. Marginal guarantees may therefore be insufficient. A stronger requirement would be coverage conditional on the full covariate vector, but nontrivial distribution-free guarantees of this kind are known to be impossible \citep{foygel2021limits}. This has motivated a line of work that relaxes exact conditional validity and instead targets weaker but attainable notions of conditional or local coverage
\citep{horeConformalPredictionLocal2024,gibbsConformalPredictionConditional2024}. 

None of these methods, however, achieve prediction-conditional coverage. Self-Calibrating Conformal
Prediction (SC-CP) \citep{laanSelfCalibratingConformalPrediction2024a} targets the two objectives above providing finite-sample guarantees
through a Venn-Abers multi-prediction \citep{vovkVennAbersPredictors2014a}. These guarantees come at a substantial computational cost; SC-CP refits the calibrator for every candidate outcome \(y\). To avoid this repeated refitting, we introduce \emph{Isotonic Conformal
Prediction} (ICP), a framework that decouples calibration from prediction interval construction. ICP fits a single isotonic calibration map from the calibration data and then constructs prediction intervals within strata sharing the same calibrated prediction. These within-stratum computations are the source of the speed-up over SC-CP. We propose two procedures, Split Isotonic Conformal Prediction (sICP) and Transductive Isotonic Conformal Prediction (tICP), that exploit this common level-set structure and differ in the strength of their self-calibration guarantees.

\begin{table}[!htbp]
\centering
\caption{Guarantees and computational cost of isotonic conformal methods.}
\label{tab:isotonic_method_comparison}
\small
\setlength{\tabcolsep}{3pt}
\renewcommand{\arraystretch}{1.08}

\begin{tabular*}{\linewidth}{@{\extracolsep{\fill}}lccll@{}}
\toprule
Method & \shortstack{Self-\\calibration}
& \shortstack{Prediction-cond.\\coverage}
& Main cost & Complexity \\
\midrule
SC-CP & Finite Sample & Yes & $K$ isotonic refits
& $\mathcal{O}(Kn\log n)$ \\
sICP
& \begin{tabular}{@{}c@{}}Asymptotic\end{tabular}
& Yes & One isotonic fit
& $\mathcal{O}(n\log n+m\log m)$ \\
Grid tICP & Finite Sample & Yes & Grid over $y$
& $\mathcal{O}(n\log n+Km\log m)$ \\
Closed-form tICP & Finite Sample & Yes & Order statistics
& $\mathcal{O}(n\log n+m\log m)$ \\
\bottomrule
\end{tabular*}

\par\vspace{3pt}
\begin{minipage}{\linewidth}
\footnotesize
\textit{Note.} Exact denotes finite-sample self-calibration.
$n$ is the sample size used for isotonic calibration,
$m$ the relevant calibration level-set size, and $K$ the response-grid size.
\end{minipage}
\end{table}

The ICP framework extends naturally from prediction interval construction to other conformal inference tasks. Using the same learned level sets, we introduce \emph{Isotonic Conformal Selection} (ICS) to target FDR control within strata defined by the calibrated predictions. ICS combines the ideas from ICP with conformal selection \citep{jinSelection2023}, where the goal is to select observations whose outcomes satisfy a desired criterion while controlling FDR. Conformal selection has been applied to a range of decision making problems \citep{guiConformalAlignmentKnowing2024, baiConformalSelectionEfficient2025, heSafeImputeReliableClinical2026}. The isotonic partition this supports both prediction interval construction and selection error control at the level of the calibrated prediction.

Beyond prediction interval construction, the conformal prediction framework  generalizes to other inference tasks, most notably candidate selection \citep{jinSelection2023}. Here the goal is to select observations whose outcomes satisfy a desired criterion while controlling the false discovery rate (FDR). Conformal selection has been applied to a range of decision making problems \citep{guiConformalAlignmentKnowing2024, baiConformalSelectionEfficient2025, heSafeImputeReliableClinical2026}. Building on our ICP methods, we introduce \emph{Isotonic Conformal Selection} (ICS), which targets FDR  control within strata defined by the calibrated predictions. This extends the role of the isotonic partition from constructing prediction intervals to controlling selection errors at the level of the calibrated prediction.

Together, these \emph{prediction} and \emph{selection} procedures form a unified framework that we call \emph{Isotonic Conformal Inference}. In both settings, the isotonic calibration map defines prediction strata within which the relevant inferential guarantee is enforced.
\begin{figure}[h]
    \centering
    \includegraphics[width=1.0\linewidth]{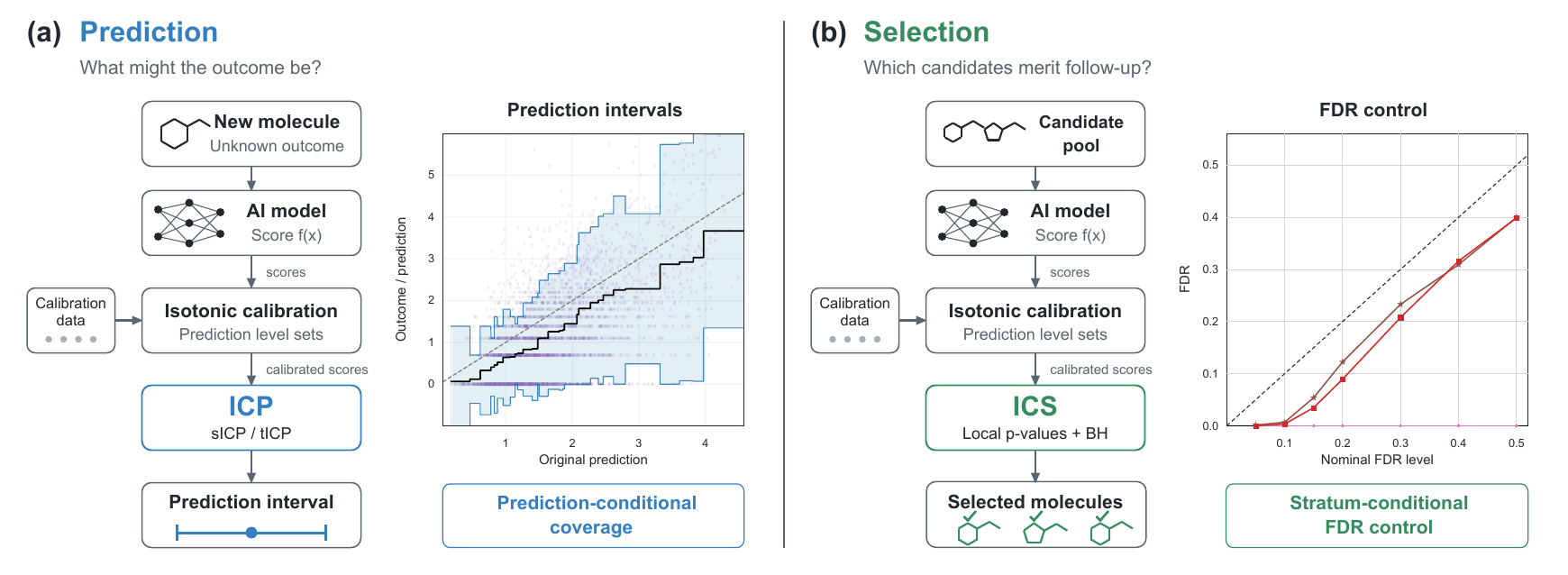}
    \caption{Overview of isotonic conformal inference for prediction and selection. A single isotonic calibration map partitions observations into strata sharing the same calibrated prediction. The prediction pipeline constructs prediction intervals using within-stratum conformal quantiles. The selection pipeline constructs localized conformal $p$-values and applies multiple testing within the resulting strata.}
    \label{fig:placeholder}
\end{figure}
\paragraph{Contributions.} A summary of the contributions from our \emph{Isotonic Conformal Inference} methods can be found below.
\begin{itemize}
    \item \textbf{Split Isotonic Conformal Prediction (sICP).}
    A two-stage procedure that calibrates the model's predictions by isotonic
    regression on one subset of the calibration data, then applies a conformal
    procedure on the remaining subset within strata of similar calibrated predictions.
    sICP attains objective~(\ref{eq:prediction-conditional}) in finite-samples
    (Theorem~\ref{thm:split_coverage}); because coverage holds conditional on the
    predicted value, the interval width adapts to the heteroscedasticity captured by
    the model, widening where higher predictions carry greater outcome variance. Its
    isotonic fit is exactly self-calibrated on the calibration split
    (Proposition~\ref{prop:isotonic_empirical_calibration}) and attains
    self-calibration asymptotically under standard regularity
    conditions (Theorem~\ref{thm:asymptotic_calibration_split}), recovering
    objective~(\ref{eq:self-calibration}). Calibration requires only a single isotonic fit, so sICP is
    essentially as cheap as split conformal prediction.

    \item \textbf{Transductive Isotonic Conformal Prediction (tICP).}
    We introduce a transductive isotonic conformal procedure that achieves exact finite-sample self-calibration and prediction-conditional coverage (Theorems~\ref{thm:transductive_exact_self_calibration} and~\ref{thm:transductive_coverage}). By exploiting the isotonic level-set structure directly, tICP replaces the repeated global isotonic refits of SC-CP with a simple local mean update. Furthermore, we derive a closed-form characterization of the tICP interval via local order statistics (Proposition~\ref{prop:closed_form_ticp}), bypassing candidate grid search entirely and matching the asymptotic runtime of sICP while retaining exact finite-sample guarantees. 

\item \textbf{Isotonic Conformal Selection (ICS).} We extend isotonic localization from prediction to multiple testing. ICS computes conformal $p$-values among calibration points in the same stratum, defined by the level sets of an isotonic calibrator fitted once on a held-out split, and applies Benjamini-Hochberg separately within each stratum. It controls the FDR within every stratum, conditional on the data used to fit the calibrator (Theorem~\ref{thm:ics-fdr}), and extends to exchangeable data via deterministic $p$-values (Theorem~\ref{thm:ics-fdr-exch}). This provides a conditional analogue of marginal conformal selection while requiring only a single fitted isotonic partition.
\end{itemize}

\section{Background and Related Works}

\subsection{Conformal Prediction}
We consider a standard regression setting with predictor $X \in \mathcal{X} \subseteq \mathbb{R}^d$ and response $Y \in \mathcal{Y} \subseteq \mathbb{R}$. 
Suppose we observe training data $\{(X_i, Y_i)\}_{i=1}^{n}$ and a new test point $(X_{n+1},Y_{n+1})$, which are drawn independently from a common distribution $P.$ Given this observed data, our goal is to construct a prediction interval $\widehat{C}(X) \subseteq \mathcal{Y}$, satisfying \[
\mathbb{P}\bigl(Y_{n+1}\in\widehat C(X_{n+1})\bigr)\ge 1-\alpha,
\] for a target miscoverage level \(\alpha \in (0,1).\) The probability averages over both the observed data and the test pair. The conformal methods described below also provide this guarantee under the weaker assumption of exchangeability.

The original conformal framework, known as \emph{full} or \emph{transductive conformal prediction} \citep{vovkAlgorithmicLearningRandom2005}, achieves this guarantee by treating the observed data and the test point symmetrically. For a test input $X_{n+1}$ and each candidate response $y\in\mathcal Y$, we form the augmented dataset
\(
\mathcal D^{(X_{n+1},y)}
=
\{(X_1,Y_1),\ldots,(X_n,Y_n),(X_{n+1},y)\}
\)
and fit a predictor $\hat f^{(X_{n+1},y)}$ using an algorithm that is invariant to the order of the observations. We then evaluate how well this predictor fits each observation in the augmented dataset. Using absolute residuals as nonconformity scores, we obtain
\[
\begin{aligned}
V_i^{(X_{n+1},y)}
&=
|Y_i-\hat f^{(X_{n+1},y)}(X_i)|,
\qquad i=1,\ldots,n,\\
V_{n+1}^{(X_{n+1},y)}
&=
|y-\hat f^{(X_{n+1},y)}(X_{n+1})|.
\end{aligned}
\]
A candidate response is included if its score does not exceed the $k$th smallest score in the augmented dataset, where $k=\lceil(1-\alpha)(n+1)\rceil$. The prediction set is therefore
\[
\widehat C(X_{n+1})
=
\left\{
y\in\mathcal Y:
V_{n+1}^{(X_{n+1},y)}
\le
V_{(k)}^{(X_{n+1},y)}
\right\},
\]
where $V_{(k)}^{(X_{n+1},y)}$ denotes the corresponding order statistic. When the candidate response equals the true response $Y_{n+1}$, the scores are exchangeable, which yields the desired coverage guarantee. While statistically rigorous, Full CP is computationally prohibitive as it requires refitting the model for every candidate outcome.

\emph{Split conformal prediction} (also known as inductive conformal prediction) \citep{papadopoulosInductiveConfidenceMachines2002, leiDistributionFreePredictiveInference2017a} avoids this computational burden by separating model training from interval construction. Concretely, we  divide the training set into two disjoint sets: a \emph{proper training set} $\mathcal{D}_1$ and a \emph{calibration set} $\mathcal{D}_2$, with $n_1 = |\mathcal{D}_1|$ and $n_2 = |\mathcal{D}_2|$. We fit a predictive model $\hat{f}$ on  $\mathcal{D}_1$. We then compute calibration residuals as $V_i = \lvert Y_i - \hat{f}(X_i) \rvert$ for $(X_i, Y_i) \in \mathcal{D}_2$,
and the corresponding test residual $V_{n+1}^{(X_{n+1},y)} = \lvert y - \hat{f}(X_{n+1}) \rvert$. The conformal quantile is computed as  $\hat{q}$, the $\lceil(1-\alpha)(n_2+1)\rceil$-th smallest value among the scores $\{V_i\}_{i \in \mathcal{D}_2}$. The resulting prediction interval for a new input $X_{n+1}$ is:
\[
\widehat{C}(X_{n+1}) = \left\{ y \in \mathcal{Y} : V_{n+1}^{(X_{n+1},y)} \le \hat{q}  \right\} = [\hat{f}(X_{n+1}) - \hat{q}, \hat{f}(X_{n+1}) + \hat{q} ].
\]
This interval satisfies the finite-sample validity guarantee with a computational cost dominated by a single training run.

\subsection{Conformal Selection}

In many applications, the goal is to select candidates whose unobserved outcomes satisfy a desirable scientific criterion. Suppose we observe a set of test covariates $\{X_{n+j}\}_{j = 1}^m$ with unobserved responses $\{Y_{n+j}\}_{j=1}^m$, and pre-specified thresholds $\{c_j\}_{j=1}^m.$ We wish to select candidates whose responses exceed their respective threshold. This can be formulated as the collection of random hypotheses
\(
H_{0,j}:Y_{n+j}\le c_j,
H_{1,j}:Y_{n+j}>c_j.
\)
Let
\(\mathcal H_0=\{j:Y_{n+j}\le c_j\}\) and
\(\mathcal H_1=\{j:Y_{n+j}>c_j\}\)
denote the realized null and alternative index sets.  A selection procedure
returns a  set $\mathcal R\subseteq\{1,\ldots,m\}$, with the objective of controlling the false discovery rate
\[
\mathrm{FDR}
=
\E[\frac{|\mathcal R\cap\mathcal H_0|}{|\mathcal R|\vee1}].
\]
The FDR is the expected proportion of selected candidates whose outcomes fail to exceed their thresholds.

 Prediction and selection require different guarantees. CP controls the  probability of failing to cover a test point. Whereas selection procedures control the expected proportion of false discoveries among selected candidates. Selecting candidates whose lower prediction interval bound exceeds $c_j$ does not by itself control FDR. 

To address this, \citep{jinSelection2023} introduce \emph{selection by prediction with conformal $p$-values}, or cfBH. Given a fitted predictor $\hat{f}$ consider the signed residual score  $V(x,y) = y - \hat f(x)$ which is increasing in $y.$
We compute  calibration scores $V_i = V(X_i,Y_i)$ and threshold scores
$\widetilde V_{n+j}=V(X_{n+j},c_j)$. For testing, $H_{0,j}:Y_{n+j}\leq c_j$, the conformal $p$-value is then \[
p_j
=
\frac{
\sum_{i=1}^n\mathbbm{1}\{V_i<\widetilde V_{n+j}\}
+
U_j\left(1+\sum_{i=1}^n
\mathbbm{1}\{V_i=\widetilde V_{n+j}\}\right)
}{n+1},
\]
where the $U_j$ are independent $\operatorname{Unif}(0,1)$ variables, independent of the data. Applying the Benjamini-Hochberg procedure \citep{benjaminiControlling1995} to these $p$-values yields finite-sample FDR control under the
conditions established by \citep{jinSelection2023}.
 
\subsection{Conditional Validity and Mondrian Conformal Prediction}

Marginal coverage can mask systematic undercoverage within particular subpopulations or regions of the feature space. 
Ideally, one would seek \emph{conditional coverage} guarantees that hold for every input $x$, i.e., $\mathbb{P}(Y \in \widehat{C}(X) \mid X=x) \geq 1-\alpha$. However, achieving non-trivial exact conditional coverage is generally impossible without strong distributional assumptions \citep{leiDistributionfreePredictionBands2014, foygel2021limits}. This impossibility result has motivated methods that target weaker but attainable notions of conditional coverage.

These methods relax the coverage requirement in different ways
Weighted conformal prediction \citep{tibshirani2019conformal} addresses covariate shift by weighting calibration observations according to the likelihood ratio between the test and training distributions. \citet{barber2023conformal} consider more general departures from exchangeability, using weighted conformal procedures and bounding the coverage loss. Localized methods instead give greater weight to calibration scores near the test input. \citet{horeConformalPredictionLocal2024} introduce a randomized localization procedure that preserves marginal validity while providing relaxed local coverage guarantees. 
\citet{gibbsConformalPredictionConditional2024},  formulate
conditional coverage as validity over a specified class of covariate shifts. Their framework provided exact finite-sample guarantees for finite dimensional classes. 

A particularly relevant approach is \emph{Mondrian Conformal Prediction}, which provides coverage conditional on membership in a prespecified group\citep{vovk2003mondrian,vovkAlgorithmicLearningRandom2005}. 
The feature space is partitioned into disjoint categories $K_1, \dots, K_M$ and the conformal procedure is run separately within each category. For a test point $X_{n+1} \in K_j$, the conformal quantile is calculated using only the calibration scores $V_i$ with points $X_i \in K_j$. This leads to coverage conditional on category membership, \(
\mathbb{P}\!\left(Y_{n+1} \in \widehat{C}(X_{n+1}) \,\middle|\,
X_{n+1} \in K_j\right) \geq 1 - \alpha.
\)

\subsection{Prediction Conditional Conformal Prediction}
Prediction conditional coverage relaxes feature conditional coverage by conditioning on the prediction rather than the full covariate vector. This guarantee is useful when predicted values guide the selection of promising candidates for further investigation. Standard CP may undercover among these candidates despite marginal coverage guarantees. \citet{jinConfidenceFocalConformal2025} address this problem through Joint Mondrian Conformal Inference (JOMI), which constructs prediction intervals with coverage conditional on a test point being selected. Prediction conditional coverage ensures that intervals remain valid among candidates selected by their predicted values.

Prediction-conditional coverage can also be defined with respect to a calibrated predictor. Self-Calibrating Conformal Prediction (SC-CP) \citep{laanSelfCalibratingConformalPrediction2024a} takes this approach, combining coverage conditional on the calibrated prediction with self-calibration of the predictor itself. SC-CP is closely related to the methods proposed in this paper and pursues the same dual objectives~(\ref{eq:self-calibration}) and (\ref{eq:prediction-conditional}). Calibration supports reliable interpretation of point predictions and can improve interval efficiency by reducing systematic prediction bias that an interval can accommodate.  SC-CP combines isotonic calibration with a Venn-Abers multiprediction \citep{vovkVennAbersPredictors2014a} to obtain finite-sample guarantees for the same objectives that we target in this work.

At the core of the method, the base model $f$ is calibrated via isotonic regression
\citep{barlowIsotonicRegressionProblem1972}: we fit an \emph{isotonic calibrator} $\hat\theta$ and form the calibrated predictor $\hat f = \hat\theta \circ f$ via
\[
\hat\theta \in
\arg\min_{\theta \in \Theta}
\sum_{i=1}^n \bigl(Y_i - \theta(f(X_i))\bigr)^2,
\qquad
\hat f = \hat\theta \circ f,
\]
where $\Theta$ is the class of monotone non-decreasing functions. Isotonic Regression produces a piecewise constant function, grouping observations that share the same calibrated prediction together. To obtain finite-sample guarantees SC-CP performs calibration transductively.
For each candidate outcome $y \in \mathcal{Y}$ it forms the augmented dataset
\(
\D^{(X_{n+1},y)} := \D \cup \{(X_{n+1},y)\}.
\)
On each augmented dataset the isotonic calibrator is refit,
\[
\hat\theta^{(y)} \in
\arg\min_{\theta \in \Theta}
\sum_{(X_i,Y_i) \in \D^{(X_{n+1},y)}}
\bigl(Y_i - \theta(f(X_i))\bigr)^2,
\qquad
\hat f^{(y)} := \hat\theta^{(y)} \circ f,
\]
yielding the conformity scores
\(
V_i^{(y)} := \bigl|Y_i - \hat f^{(y)}(X_i)\bigr|,
\text{  } i = 1, \dots, n,
\text{ and }
V_{n+1}^{(y)} := \bigl|y - \hat f^{(y)}(X_{n+1})\bigr|.
\)
Rather than taking a global quantile of these scores, SC-CP leverages the
step-function nature of isotonic regression and computes the quantile using only the
points whose calibrated prediction coincides with that of the test point, by
minimizing the pinball loss
\[
\hat q^{(y)}(X_{n+1})
= \arg\min_{q \in \mathbb{R}}
\sum_{i=1}^n
\mathbbm{1}\!\left\{\hat f^{(y)}(X_i) = \hat f^{(y)}(X_{n+1})\right\}
\ell_\alpha\bigl(q, V_i^{(y)}\bigr)
+ \ell_\alpha\bigl(q, V_{n+1}^{(y)}\bigr),
\]
where
$\ell_\alpha(q, V) = \mathbbm{1}(V < q)\,\alpha(q - V) + \mathbbm{1}(V \geq q)\,(1 - \alpha)(V - q)$
is the pinball loss at level $\alpha \in (0,1)$. It's minimizer over recovers the $(1-\alpha)$ empirical quantile of
$\{V_i\}_{i=1}^n$ \citep{koenkerRegressionQuantiles1978a}. The prediction interval then
consists of every candidate $y$ whose score falls below this localized quantile,
\[
\widehat{C}_{\text{SC-CP}}(X_{n+1})
:= \bigl\{ y \in \mathcal{Y} :
V_{n+1}^{(y)} \leq \hat q^{(y)}(X_{n+1}) \bigr\}.
\]
This localization is precisely what upgrades marginal validity to
prediction-conditional validity.
\citet{laanSelfCalibratingConformalPrediction2024a} show that, in finite-samples, the
calibrated predictor $\hat f^{(Y_{n+1})}$ achieves exact self-calibration (objective~(\ref{eq:self-calibration})) and the
resulting prediction intervals attain prediction-conditional coverage (objective~(\ref{eq:prediction-conditional})).

The price of these guarantees is computational. For a continuous outcome, the space,
$\mathcal{Y}$ must be discretized into a grid of size $K$, and SC-CP performs a full
isotonic calibration on the augmented dataset for each of the $K$ candidates.
One-dimensional isotonic regression costs $O(n)$ when the model scores are already
sorted and $O(n\log n)$ otherwise, so the direct implementation costs $O(Kn)$ or
$O(Kn\log n)$. This factor of $K$ is the main bottleneck that the methods of this
paper remove.

\section{Isotonic Conformal Prediction}
In this section we present two procedures within the \emph{Isotonic Conformal Prediction (ICP)} framework. Both methods are built on the same structural property of isotonic regression. Regressing the outcome on the base predictor $f(x)$ under a monotonicity constraint yields a calibrator, $\hat\theta$, that is a non-decreasing step function. The calibrated predictor $\hat f = \hat\theta \circ f$
therefore takes only finitely many distinct values and partitions the feature space
into \emph{level sets}. Conditioning on the calibrated prediction is therefore equivalent to conditioning on membership in one of these level sets.

This is what makes prediction-conditional coverage attainable from a single calibration fit. These level sets provide a natural partition for Mondrian CP. Rather than specifying groups in advance, we learn them through isotonic calibration and construct prediction intervals using calibration observations in the test point's level set. This reduces prediction conditional inference to computations within individual groups, avoiding repeated global isotonic fits. This allows for both methods to achieve finite-sample prediction-conditional coverage. 

The two ICP procedures we  differ in how their within-level-set computation is carried out, which determines the strength of their self-calibration guarantee. Split
Isotonic CP (sICP) (Section~\ref{sec:split_iso_cp}) holds the calibrated predictor and its
residuals fixed; it is the computationally cheaper method and  attains empirical self-calibration
on the calibration data, together with asymptotic population-level self-calibration. Specifically, under the stated assumptions, the largest discrepancy between a calibrated prediction and its conditional mean response  converges to zero as the sample used for isotonic calibration grows.
Transductive Isotonic CP (tICP)(Section~\ref{sec:trans_iso_cp}) instead applies a local
update within the level set and recovers exact finite-sample self-calibration,
at a modest additional cost. This cost is mitigated by the existence of a closed form solution for tICP. While the strength of their self-calibration guarantees are different, the following methods both achieve finite-sample prediction-conditional coverage. In the next section, we show how this same level-set structure extends to conformal selection, completing our unified \emph{Isotonic Conformal Inference} framework.

\begin{figure}[!htbp]
    \centering
    \includegraphics[width=\textwidth]{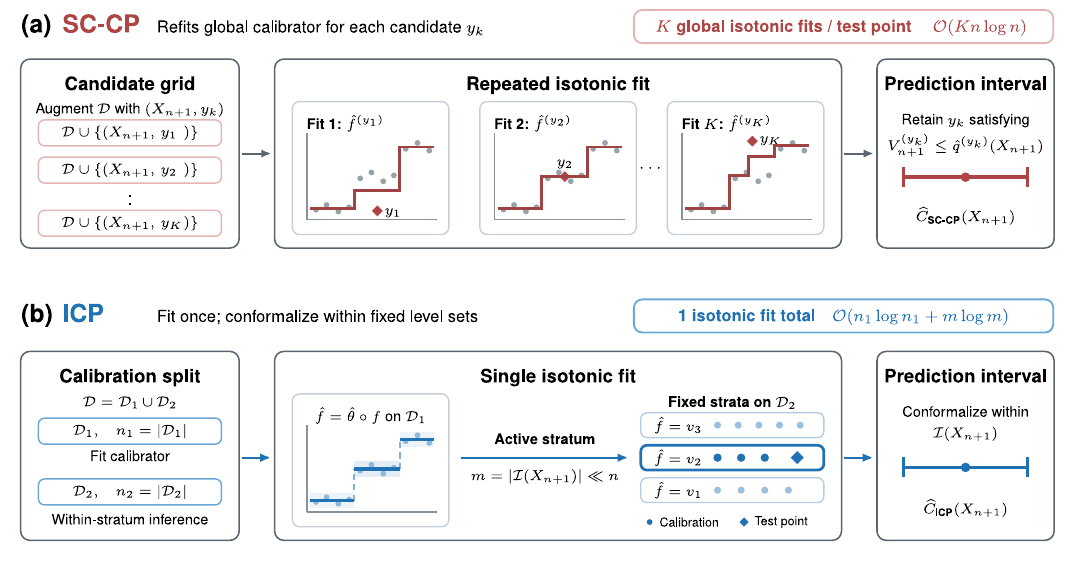}
    \caption{
        Comparison of SC-CP and ICP.
        \textbf{(a)} SC-CP refits the global isotonic calibrator
        for each candidate response.
        \textbf{(b)} ICP fits the isotonic calibrator once and
        performs conformal inference within the resulting
        fixed level sets.
    }
    \label{fig:icp_one_fit}
\end{figure}

\subsection{Split Isotonic Conformal Prediction}
\label{sec:split_iso_cp}
Motivated by the efficiency of split conformal prediction
\citep{leiDistributionfreePredictionBands2014}, we propose Split Isotonic Conformal
Prediction (sICP). sICP yields prediction-conditional coverage from a single isotonic
calibration fit. In exchange for these computational savings, the exact finite-sample self calibration goal (objective~(\ref{eq:self-calibration})) is relaxed to empirical and asymptotic versions.

Starting with a trained base model $f$ and a calibration dataset
$\mathcal{D} = \{(X_i, Y_i)\}_{i=1}^n$, we randomly partition the calibration indices
into two disjoint subsets $\mathcal{I}_1, \mathcal{I}_2$ with
$\mathcal{I}_1 \cup \mathcal{I}_2 = \{1, \dots, n\}$ and
$\mathcal{I}_1 \cap \mathcal{I}_2 = \emptyset$. Let
$\mathcal{D}_1 = \{(X_i, Y_i)\}_{i \in \mathcal{I}_1}$ denote the \emph{calibration
training set} and $\mathcal{D}_2 = \{(X_i, Y_i)\}_{i \in \mathcal{I}_2}$ the
\emph{score computation set}, with $n_1 = |\mathcal{D}_1|$ and $n_2 = |\mathcal{D}_2|$.

We first fit the isotonic calibrator $\hat\theta$ on $\mathcal{D}_1$, yielding the
calibrated predictor $\hat f := \hat\theta \circ f$,
\[
    \hat\theta \in
    \arg\min_{\theta \in \Theta}
    \sum_{i \in \mathcal{I}_1}
    \bigl(Y_i - \theta(f(X_i))\bigr)^2,
    \qquad
    \hat f = \hat\theta \circ f,
\]
where $\Theta$ is the class of monotone non-decreasing functions. Given a test point
$X_{n+1}$, we then restrict attention to the calibration points in $\mathcal{D}_2$ that
share its calibrated prediction,
\(
    \mathcal{I}(X_{n+1}) =
    \{ i \in \mathcal{I}_2 : \hat f(X_i) = \hat f(X_{n+1}) \}.
\)
Since $\hat f$ takes only finitely many distinct values, this membership check is
well-defined and can be implemented as a level-set lookup rather than a numerical
equality test. Finally, we conformalize within this level set by computing the absolute residuals
\(
    \mathcal{E} =
    \bigl\{ |Y_i - \hat f(X_i)| : i \in \mathcal{I}(X_{n+1}) \bigr\},
\)
and forming the prediction interval
\[
    \widehat{C}(X_{n+1}) =
    \bigl\{ y \in \mathcal{Y} :
    |y - \hat f(X_{n+1})| \leq \hat q(X_{n+1}) \bigr\},
\]
where $\hat q(X_{n+1}) = \mathrm{Quantile}\bigl(1 - \alpha,\, \mathcal{E} \cup \{\infty\}\bigr)$
is the $\lceil (1-\alpha)(|\mathcal{E}| + 1) \rceil$-th smallest value of
$\mathcal{E} \cup \{\infty\}$.

\begin{algorithm}[h]
\caption{Split Isotonic Conformal Prediction}
\label{alg:split_iso_cp}
\KwIn{Data $\mathcal{D}$, base model $f$, $\alpha$}
Split $\mathcal{D}$ into $\mathcal{D}_1, \mathcal{D}_2$\;
Fit isotonic $\hat{\theta}$ on $\mathcal{D}_1$; $\hat{f} \leftarrow \hat{\theta} \circ f$\;
\For{a test point $X_{n+1}$}{
    Identify group indices $\mathcal{I}(X_{n+1}) = \{i \in \mathcal{I}_2 : \hat{f}(X_i) = \hat{f}(X_{n+1})\}$\;
    Compute scores $\mathcal{E} = \{ |Y_i - \hat{f}(X_i)| : i \in \mathcal{I}(X_{n+1}) \}$\;
    $\hat{q}(X_{n+1}) \leftarrow \text{Quantile}(1-\alpha, \mathcal{E} \cup \{\infty\})$\;
    \Return $\widehat{C}(X_{n+1}) = \left\{ y \in \mathcal{Y} : |y - \hat{f}(X_{n+1})| \le \hat{q}(X_{n+1}) \right\}$;
}
\end{algorithm}

This construction attains prediction-conditional coverage by forming the quantile using
only the calibration points that share the test point's calibrated prediction, while
requiring a single isotonic fit; we formalize the guarantee in
Theorem~\ref{thm:split_coverage}. If $\mathcal{I}(X_{n+1})$ is empty, the added infinity score gives $\hat{q}(X_{n+1}) = \infty$ so the procedure returns the uninformative interval $\widehat{C}(X_{n+1}) = \mathcal{Y}$ This case is rare in practice: for large $n_2$, each level set
pools many calibration points wherever the test prediction is supported by the training data.

The cost of Split Isotonic CP separates into a one-time calibration step and a cheap
per-test-point step. Let $m = |\mathcal{I}(X_{n+1})|$ denote the number of calibration
points in the test point's level set. Fitting the one-dimensional isotonic calibrator on
$\mathcal{D}_1$ is done once and costs $O(n_1\log n_1)$, or $O(n_1)$ when the model
scores are already sorted. Each subsequent test point then requires only the $m$
within-level-set residuals and a single empirical quantile, costing $O(m\log m)$ with a
sorting-based quantile. The per-test-point cost therefore scales with the size of the
local level set rather than the full calibration set, and involves no refitting of the
calibrator and no search over candidate labels.

Since each prediction is cheap, sICP is well suited to latency-sensitive
applications, where decisions must be made quickly yet still demand reliable
uncertainty quantification, such as emergency medical triage situations. The
bulk of the computation, the isotonic fit, is paid for once offline; at deployment each
prediction requires only a within-level-set quantile, so prediction-conditional
intervals can be produced with negligible overhead.

\subsubsection{Theoretical Results for Split Isotonic Conformal Prediction}
\label{sec:theory_split}

Theorem~\ref{thm:split_coverage} establishes objective~(\ref{eq:prediction-conditional}): sICP achieves finite sample coverage conditional on the calibrated prediction under exchangeability. Its proof is given in Appendix~\ref{app:proof_split_coverage}.
\begin{theorem}[Prediction-conditional coverage for Split Isotonic CP]
\label{thm:split_coverage}
Assume that $\{(X_i,Y_i)\}_{i=1}^{n+1}$ are exchangeable. Let
$\widehat C(X_{n+1})$ be the prediction interval constructed by
Algorithm~\ref{alg:split_iso_cp}. Then
\[
\mathbb P\left(
Y_{n+1}\in \widehat C(X_{n+1})
\,\middle|\,
\hat f(X_{n+1})
\right)
\ge 1-\alpha.
\]
\end{theorem}

We next turn to objective~\ref{eq:self-calibration}. The isotonic fit satisfies an exact empirical self-calibration identity: on the data used to fit the calibrator, each fitted prediction equals the average response within its level set. The following proposition formalizes this property, with its proof given in Appendix~\ref{app:split_empirical_calibration}.

\begin{proposition}[Empirical self-calibration of the isotonic fit]
\label{prop:isotonic_empirical_calibration}
Let $\mathcal D_1=\{(X_i,Y_i)\}_{i\in I_1}$ be the data used to fit the
isotonic calibrator, and define
\[
    \hat f = \hat\theta \circ f,
    \qquad
    \hat\theta \in 
    \arg\min_{\theta\in \Theta}
    \sum_{i\in I_1} \bigl(Y_i-\theta(f(X_i))\bigr)^2,
\]
where $\Theta$ denotes the class of nondecreasing functions. Let
\(
    \mathcal V = \{v_1,\ldots,v_K\}
\)
be the distinct values taken by $\hat f(X_i)$ on $\mathcal D_1$, and
define the corresponding level sets
\(
    \mathcal B_k
    =
    \{i\in I_1:\hat f(X_i)=v_k\}.
\)
Then, for every $k=1,\ldots,K$, \(
    \frac{1}{|\mathcal B_k|}
    \sum_{i\in \mathcal B_k} Y_i
    =
    v_k.
\)
\end{proposition}

This empirical identity does not guarantee population self-calibration in finite samples. To obtain a population guarantee asymptotically, we impose the following assumptions, which ensure uniform convergence of the isotonic fit to the true calibration curve. These assumptions are used only for Theorem~\ref{thm:asymptotic_calibration_split}.

\noindent \textbf{Assumption 1.} \textit{The observations $(X_i, Y_i)$ are i.i.d. draws from a distribution $P$ with $\mathbb{E}[Y^2] < \infty$. The support $\mathcal{S}$ of the base predictions $S = f(X)$ is a compact interval in $\mathbb{R}$.}

\noindent \textbf{Assumption 2.} \textit{The true calibration function, defined as $\theta_0(s) = \mathbb{E}[Y \mid f(X)=s]$, is continuous and strictly monotonically increasing on $\mathcal{S}$.}

Under these assumptions, the discrepancy between calibrated predictions and their conditional mean responses vanishes uniformly almost surely as the sample used to fit the isotonic calibrator grows. The following theorem establishes this asymptotic population self-calibration guarantee; its proof is given in Appendix~\ref{app:asymptotic_calibration}.

\begin{theorem}[Asymptotic Population Self-Calibration]
\label{thm:asymptotic_calibration_split}
Suppose Assumptions 1 and 2 hold. Let $(X_{n+1}, Y_{n+1}) \sim P$ be a new test pair independent of $\mathcal{D}_1$. As $n_1 \to \infty$, the calibrated predictor $\hat{f}_{n_1}$ asymptotically achieves exact population-level self-calibration. Specifically, conditional on the training data $\mathcal{D}_1$, the maximal absolute calibration error converges to zero almost surely:
\begin{equation}
\sup_{x \in \mathcal{X}} \Bigl| \mathbb{E}\bigl[\,Y_{n+1} \mid \hat{f}_{n_1}(X_{n+1}) = \hat{f}_{n_1}(x), \mathcal{D}_1\bigr] - \hat{f}_{n_1}(x) \Bigr| \xrightarrow{a.s.} 0.
\end{equation}
\end{theorem}









 \subsection{Transductive Isotonic Conformal Prediction}
\label{sec:trans_iso_cp}

Split Isotonic CP (Section~\ref{sec:split_iso_cp}) attains prediction-conditional
coverage with a single isotonic fit, but only achieves \emph{empirical} self-calibration
of $\hat f$ on $\mathcal{D}_1$ (Proposition~\ref{prop:isotonic_empirical_calibration});
exact finite-sample population-level self-calibration is sacrificed. The original SC-CP
procedure recovers the latter, but at the cost of refitting the isotonic calibrator on
$\mathcal{D} \cup \{(X_{n+1}, y)\}$ for every candidate $y$, which is restrictive for
continuous outcomes.

To combat this we propose \emph{Transductive Isotonic Conformal Prediction}, which recovers both
finite-sample self-calibration and prediction-conditional coverage without repeatedly
refitting the calibrator. The idea is to fit the isotonic calibrator once on
$\mathcal{D}_1$, then perform the transductive update locally, within the test point's
level set on $\mathcal{D}_2$, replacing the global isotonic refit with a scalar mean update.
Theorems~\ref{thm:transductive_exact_self_calibration} and~\ref{thm:transductive_coverage}
establish that this procedure attains exact finite-sample self-calibration and
prediction-conditional coverage.

The setup matches Split Isotonic CP through the calibration step: we fit $\hat\theta$ on
$\mathcal{D}_1$ and set $\hat f = \hat\theta \circ f$ as in
Section~\ref{sec:split_iso_cp}. For a test point $X_{n+1}$ we again identify the
calibration points in its level set,
\(
    \mathcal{I}(X_{n+1}) =
    \bigl\{ i \in \mathcal{I}_2 : \hat f(X_i) = \hat f(X_{n+1}) \bigr\},
     m = |\mathcal{I}(X_{n+1})|,
\)
falling back to $\widehat{C}(X_{n+1}) = \mathcal{Y}$ whenever $m = 0$.

The two procedures diverge in the next step. For each candidate $y \in \mathcal{Y}$ we
define the transductive local mean 
$\hat\mu^{(y)}(X_{n+1}) =
\frac{1}{m+1}
\left( \sum_{i \in \mathcal{I}(X_{n+1})} Y_i + y \right).$
This is the within-bin mean that would result from refitting an \emph{unconstrained}
regression on $\mathcal{D}_2 \cup \{(X_{n+1}, y)\}$ restricted to the test point's level set. We then center the conformity scores on this local mean,
$V_i^{(y)} = \bigl| Y_i - \hat\mu^{(y)}(X_{n+1}) \bigr|,
    \text{ for } i \in \mathcal{I}(X_{n+1})$ and for the test score $V_{n+1}^{(y)} = \bigl| y - \hat\mu^{(y)}(X_{n+1}) \bigr|.$
The prediction interval is then
\[
\widehat{C}(X_{n+1}) =
\Bigl\{ y \in \mathcal{Y} :
V_{n+1}^{(y)} \leq
Q_{1-\alpha}\bigl(
\{V_i^{(y)}\}_{i \in \mathcal{I}(X_{n+1})}
\cup \{V_{n+1}^{(y)}\}
\bigr) \Bigr\}.
\]

\begin{algorithm}[!htbp]
\caption{Transductive Isotonic Conformal Prediction}
\label{alg:trans_iso_cp}
\KwIn{Dataset $\mathcal{D}$, base predictor $f$, miscoverage level $\alpha\in(0,1)$}

Split $\mathcal{D}$ into calibration folds $\mathcal{D}_1$ and $\mathcal{D}_2$\;
Fit isotonic calibrator $\hat{\theta}$ on $\mathcal{D}_1$ and set $\hat{f}\leftarrow\hat{\theta}\circ f$\;

\ForEach{test point $X_{n+1}$}{
    $\mathcal{I}(X_{n+1})\leftarrow\{i\in\mathcal{I}_2:\hat{f}(X_i)=\hat{f}(X_{n+1})\}$\;
    $m\leftarrow|\mathcal{I}(X_{n+1})|$\;

    \eIf{$m=0$}{
        \Return $\mathcal{Y}$\;
    }{
        \For{candidate label $y\in\mathcal{Y}$}{
            $\hat\mu^{(y)}(X_{n+1})\leftarrow\dfrac{1}{m+1}\left(\sum_{i\in\mathcal{I}(X_{n+1})}Y_i+y\right)$\;
            $V_i^{(y)}\leftarrow\left|Y_i-\hat\mu^{(y)}(X_{n+1})\right|$ for all $i\in\mathcal{I}(X_{n+1})$\;
            $V_{n+1}^{(y)}\leftarrow\left|y-\hat\mu^{(y)}(X_{n+1})\right|$\;
        }
        \Return $\widehat{C}(X_{n+1})=\left\{y\in\mathcal{Y}:V_{n+1}^{(y)}\le Q_{1-\alpha}\!\left(\{V_i^{(y)}\}_{i \in \mathcal{I}(X_{n+1})}\cup\{V_{n+1}^{(y)}\}\right)\right\}$\;
    }
}
\end{algorithm}


Like Split Isotonic CP, the isotonic calibrator is fit only once, at cost
$O(n_1\log n_1)$ , where $n_1 = |\mathcal{D}_1|$. The additional expense is an
inner loop over candidate labels. For each of the $K$ grid points $y$ used to discretize
$\mathcal{Y}$, the algorithm forms the local mean $\hat\mu^{(y)}(X_{n+1})$, computes the
$m = |\mathcal{I}(X_{n+1})|$ scores within the test point's level set, and sorts them to
obtain the quantile, at cost $O(m\log m)$. The per-test-point cost is therefore $O(K m\log m)$. When the calibration
map has many steps, $m$ is typically far smaller than $n_2$, so the cost is governed by
the size of the local level set rather than the full calibration set.

The two procedures share the same setup and calibration, and differ only in how the local
quantile is formed: Split Isotonic CP fixes the residuals at $|Y_i - \hat f(X_i)|$,
whereas Transductive Isotonic CP centers them on the transductive local mean
$\hat\mu^{(y)}(X_{n+1})$. This update is what upgrades empirical to exact finite-sample
self-calibration (Theorem~\ref{thm:transductive_exact_self_calibration}); the factor of
$K$ above is the price it pays.

\subsubsection{Closed Form Version of tICP}
Algorithm~\ref{alg:trans_iso_cp}, defines the tICP prediction set by inverting the conformity test over candidate labels $y\in \mathcal{Y}$. For continuous outcomes, a direct implementation of tICP requires evaluating the method on a finite response grid, making the computational cost and approximation depend on the chosen grid. In this section we present a closed form characterization of the tICP prediction interval, bypassing the need to search over candidate responses.

The key to this is that the isotonic partition remains fixed as the candidate $y$ varies. Let
\(
\overline Y_{\mathcal I}
=\frac{1}{|\mathcal I(X_{n+1})|}\sum_{i\in\mathcal I(X_{n+1})}Y_i
\) denote the response mean of the level set containing the test point. Writing $y =\overline Y_{\mathcal I} + d$, the augmented fitted value is updated through the local mean adjustment, \(
\widehat\mu^{(y)}(X_{n+1})
=\frac{m\overline Y_{\mathcal I}+y}{m+1}
=\overline Y_{\mathcal I}
+\frac{d}{m+1}.
\)
Constructing the prediction interval therefore amounts to determining how far $d$ can move away from the fixed $\overline Y_{\mathcal I}$ while satisfying the conformity test. Although updating $y$ changes both the calibration and test scores, the number of calibration scores below the test score can only increase as d moves away from zero.

Proposition~\ref{prop:closed_form_ticp} formalizes this, showing that the tICP prediction interval's endpoints can be determined by two order statistics and establishing the exact equivalence between the closed-form construction and the original transductive procedure. Algorithm~\ref{alg:closed_form_ticp} implements this directly. Once the relevant level set is identified, the closed form construction requires \(O(m\log m)\) operations to compute and sort the thresholds, thus removing the factor of $K$ required by a grid implementation.

\begin{algorithm}[!htbp]
\small
\caption{Closed-Form Transductive Isotonic Conformal Prediction}
\label{alg:closed_form_ticp}

\KwIn{Dataset $\mathcal D$, base predictor $f$,
miscoverage level $\alpha\in(0,1)$}

Split $\mathcal D$ into calibration folds
$\mathcal D_1$ and $\mathcal D_2$\;
Fit isotonic calibrator $\hat\theta$ on $\mathcal D_1$
and set $\hat f\leftarrow\hat\theta\circ f$\;

\For{each test point $X_{n+1}$}{
    $\mathcal I\leftarrow
    \{i\in \mathcal I_2:\hat f(X_i)=\hat f(X_{n+1})\}$\;
    $m\leftarrow|\mathcal I|$,
    $k\leftarrow\lceil(1-\alpha)(m+1)\rceil$\;

    \uIf{$m=0$}{
        $\widehat C(X_{n+1}) = \mathcal{Y}$\;
    }
    \uElseIf{$m=1$ \textbf{or} $k=m+1$}{
        $\widehat C(X_{n+1})\leftarrow\mathcal Y$\;
    }
    \Else{
        $\bar Y_{\mathcal I}\leftarrow
        \frac{1}{m}\sum_{i\in\mathcal I}Y_i$,
        $\rho\leftarrow\frac{m-1}{m+1}$\;

        \For{$i\in\mathcal I$}{
            $r_i^+\leftarrow
            \max\left\{
                Y_i-\bar Y_{\mathcal I},
                \frac{\bar Y_{\mathcal I}-Y_i}{\rho}
            \right\}$,
            \quad
            $r_i^-\leftarrow
            \max\left\{
                \bar Y_{\mathcal I}-Y_i,
                \frac{Y_i-\bar Y_{\mathcal I}}{\rho}
            \right\}$\;
        }

        Let $r_{(k)}^\pm$ be the $k$-th order statistics
        of $\{r_i^\pm:i\in\mathcal I\}$\;
        $\widehat C(X_{n+1})\leftarrow
        \mathcal Y\cap
        \left[
            \bar Y_{\mathcal I}-r_{(k)}^-,
            \bar Y_{\mathcal I}+r_{(k)}^+
        \right]$\;
    }
}
\end{algorithm}

\subsubsection{Theoretical Results for Transductive Isotonic Conformal Prediction }
\label{sec:theory_transductive}

For  Transductive Isotonic CP, the local transductive prediction for a candidate
label $y$ is
\(
\hat{\mu}^{(y)}(X_{n+1})
=
\frac{1}{m+1}
\left(
\sum_{i\in\mathcal I(X_{n+1})}Y_i+y
\right),
\)
where
\(
\mathcal I(X_{n+1})
=
\{i\in\mathcal I_2:\hat f(X_i)=\hat f(X_{n+1})\}, \text{ and }
m=|\mathcal I(X_{n+1})|.
\)

We first establish objective~\ref{eq:self-calibration}. When evaluated at the true test response, the local transductive update yields an oracle prediction that is satisfies finite sample self calibration. The following theorem formalizes this with its proof given in Appendix~\ref{app:proof_transductive_exact_self_calibration}.

\begin{theorem}[Exact finite-sample self-calibration for Transductive Isotonic CP]
\label{thm:transductive_exact_self_calibration}
Assume that $\{(X_i,Y_i)\}_{i=1}^{n+1}$ are exchangeable. Define \(
\hat\mu(X_{n+1})
=
\hat\mu^{(Y_{n+1})}(X_{n+1}),
\)
then
\[
\mathbb E\left[
Y_{n+1}
\,\middle|\,
\hat\mu(X_{n+1})
\right]
=
\hat\mu(X_{n+1})
\qquad \text{a.s.}
\]
\end{theorem}

We next establish objective~(\ref{eq:prediction-conditional}). The tICP prediction interval achieves finite sample coverage conditional on the same oracle prediction. Thus, both self-calibration and prediction-conditional coverage hold with respect to $\hat\mu^{(Y_{n+1})}(X_{n+1})$. The proof of the following theorem is given in Appendix~\ref{app:proof_transductive_coverage}.

\begin{theorem}[Prediction-conditional coverage for  Transductive Isotonic CP]
\label{thm:transductive_coverage}
Assume that $\{(X_i,Y_i)\}_{i=1}^{n+1}$ are exchangeable. Let
$\widehat C(X_{n+1})$ be the prediction interval constructed by
Algorithm~\ref{alg:trans_iso_cp}. Then
\[
\mathbb P\left(
Y_{n+1}\in \widehat C(X_{n+1})
\,\middle|\,
\hat\mu^{(Y_{n+1})}(X_{n+1})
\right)
\ge 1-\alpha.
\]
\end{theorem}

Finally, we characterize the tICP prediction set in closed form. The following proposition expresses its endpoints through order statistics of transformed within-level-set residuals, allowing the interval to be computed without searching over candidate labels. Its proof is given in Appendix~\ref{app:proof_prop_closed_form}.

\begin{proposition}[Closed-form tICP set]
\label{prop:closed_form_ticp}
Fix a test point $X_{n+1}$, let
$\mathcal I=\{i\in \mathcal{I}_2:\hat f(X_i)=\hat f(X_{n+1})\}$,
and suppose $m=|\mathcal I|\ge 2$. Define
$\bar Y_{\mathcal I}=m^{-1}\sum_{i\in\mathcal I}Y_i$,
$k=\lceil(1-\alpha)(m+1)\rceil$, and $\rho=(m-1)/(m+1)$.
For each $i\in\mathcal I$, set
\[
    r_i^+
    =
    \begin{cases}
        Y_i-\bar Y_{\mathcal I},
        & Y_i\ge\bar Y_{\mathcal I},\\[2pt]
        (\bar Y_{\mathcal I}-Y_i)/\rho,
        & Y_i<\bar Y_{\mathcal I},
    \end{cases}
    \qquad
    r_i^-
    =
    \begin{cases}
        \bar Y_{\mathcal I}-Y_i,
        & Y_i\le\bar Y_{\mathcal I},\\[2pt]
        (Y_i-\bar Y_{\mathcal I})/\rho,
        & Y_i>\bar Y_{\mathcal I}.
    \end{cases}
\]
Then Algorithm~\ref{alg:trans_iso_cp} produces
\[
    \widehat C_{\mathrm{tICP}}(X_{n+1})
    =
    \mathcal Y\cap
    \left[
        \bar Y_{\mathcal I}-r^-_{(k)},
        \bar Y_{\mathcal I}+r^+_{(k)}
    \right].
\]
If $k=m+1$, then
$\widehat C_{\mathrm{tICP}}(X_{n+1})=\mathcal Y$.
\end{proposition}

\section{From Conditional Prediction to Conditional Selection}
\label{sec:ics}
\emph{Isotonic Conformal Prediction} provides a general framework for inference conditional on calibrated predictions. We now extend its reach beyond prediction intervals to \emph{Isotonic Conformal Selection} using the learned level sets to control FDR within calibrated prediction strata.
Standard conformal selection \citep{jinSelection2023} pools calibration scores globally to control marginal FDR, which can fail to control false discoveries within specific prediction ranges under miscalibration or local heteroscedasticity. Recent works have extended conformalized multiple testing to conditional and selective settings: \citet{wuConditional2025} construct localized conformal $p$-values via kernel weighting in the feature space, while \citet{wangConformalized2024} control FDR over data-dependently selected subsets by applying a stable selection rule to both calibration and test data. However, kernel localization can suffer from weight collapse in high-dimensional spaces, and subset-selection rules typically require specifying a filtering criterion in advance.

In this section, we instead use the level sets of the isotonic calibrator directly as data-driven selection strata. Isotonic Conformal Selection (ICS) constructs localized conformal $p$-values and applies Benjamini--Hochberg separately within each stratum, targeting \textit{stratum-conditional FDR control} from a single one-dimensional isotonic fit.
Section~\ref{ICS} describes the split construction and selection procedure;
Section~\ref{sec:ics-theory} establishes localized rank validity,
per-selection error rate (PSER) control, and stratum-conditional FDR control.
As a full-conformal baseline, we also adapt SC-CP \citep{laanSelfCalibratingConformalPrediction2024a} to the selection task, yielding \emph{Full Self-Calibrating Conformal Selection} (Full SC-CS), which refits the calibrator for each candidate (detailed in Appendix~\ref{Full SC-CS}).

For a fixed calibrated predictor $\hat f$, let \(\mathcal V=\{\hat f(X_{n+j}):j=1,\ldots,m\}\)
be the set of calibrated prediction values and define \(
G_v
=
\{j:\hat f(X_{n+j})=v\}, v\in\mathcal V.
\)
within-stratum $v$, let
\[
\mathcal H_{0,v}=\mathcal H_0\cap G_v,
\qquad
\mathcal H_{1,v}=\mathcal H_1\cap G_v,
\qquad
\mathcal R_v=\mathcal R\cap G_v.
\]
The stratum-specific false discovery proportion, false discovery rate, and
power are defined by
\[
\mathrm{FDP}_v=\frac{|\mathcal R_v\cap\mathcal H_{0,v}|}
     {|\mathcal R_v|\vee1},
\qquad
\mathrm{FDR}_v=
\E[\mathrm{FDP}_v],
\qquad
\mathrm{Power}_v=
\E\!\left[
\frac{|\mathcal R_v\cap\mathcal H_{1,v}|}
     {|\mathcal H_{1,v}|\vee1}
\right].
\]
The expectation is over the calibration and test samples and any algorithmic
randomization.

\subsection{Isotonic Conformal Selection} \label{ICS}
\label{pvalue}
ICS fits the isotonic calibrator once on a calibrator-fitting subset
$\mathcal D_1$ and conformalizes within level sets on a separate subset $\mathcal{D}_2$. It reuses the split isotonic calibration and static
level-set construction introduced for sICP in Section~\ref{sec:split_iso_cp}.
Specifically,
we use the same split
$\mathcal D=\mathcal D_1\cup\mathcal D_2$ (with index sets $\mathcal{I}_1, \mathcal{I}_2$), fit $\hat\theta$ on $\mathcal D_1$, and hold the calibrated predictor
$\hat f=\hat\theta\circ f$ fixed.
For each test candidate $X_{n+j}$ with prespecified threshold $c_j$, we identify its local calibration set
\(
\mathcal{I}(X_{n+j}) = \bigl\{i \in \mathcal{I}_2 : \hat f(X_i) = \hat f(X_{n+j})\bigr\}
\), with \(n_j=|\mathcal{I}(X_{n+j})|\).
Thus the local calibration set contains only observations in
$\mathcal D_2$. We test $H_{0,j}:Y_{n+j}\le c_j$ by comparing the
candidate's threshold score with scores from this local set.

We first define an \emph{oracle} $p$-value, assuming access to the unobserved true response $Y_{n+j}$:
\begin{equation}
  p_j^{*} = \frac{\sum_{i \in \mathcal{I}(X_{n+j})} \mathbbm{1}\{Y_i < Y_{n+j}\} + U_j\!\left(1 + \sum_{i \in \mathcal{I}(X_{n+j})} \mathbbm{1}\{Y_i = Y_{n+j}\}\right)}{n_j + 1},
\end{equation}
where $U_j\sim\mathrm{Uniform}[0,1]$ is a tie-breaking variable.
Throughout, the $U_j$ are mutually independent and independent of all data.

In practice, the true response $Y_{n+j}$ is unobserved. To test the null hypothesis $H_{0,j}: Y_{n+j} \le c_j$ for a given test candidate $X_{n+j}$, we require a non-conformity score $V(x, y)$ that is monotonically increasing in $y$. Unlike the absolute residuals used for constructing symmetric prediction intervals in Isotonic Conformal Prediction, we utilize the signed calibrated residual score $V(x, y) = y - \hat{f}(x)$ to reflect the one-sided nature of the selection task. In standard conformal selection, the $p$-value is constructed by evaluating the empirical rank of the threshold score $\widetilde{V}_{n+j} = c_j - \hat{f}(X_{n+j})$ against a global pool of calibration scores $\{V_i\}_{i=1}^n$. To achieve conditional validity, we evaluate the empirical rank of this score exclusively within the candidate's corresponding local stratum. Writing $V_i=V(X_i,Y_i)$, the localized conformal $p$-value for testing $H_{0,j}$ is defined as:
\begin{equation} \label{eq:initial_pval}
    p_j = \frac{\sum_{i \in \mathcal{I}(X_{n+j})} \mathbbm{1}\{V_i < \widetilde{V}_{n+j}\} + U_j \cdot \left( 1 + \sum_{i \in \mathcal{I}(X_{n+j})} \mathbbm{1}\{V_i = \widetilde{V}_{n+j}\} \right)}{n_j + 1}.
\end{equation}
Within the local calibration set $\mathcal{I}(X_{n+j})$, all grouped samples share the identical calibrated prediction: $\hat{f}(X_i) \equiv \hat{f}(X_{n+j})$. Consequently, the calibrated prediction cancels from the comparison: $Y_i - \hat{f}(X_i) < c_j - \hat{f}(X_{n+j}) \iff Y_i < c_j$. The $p$-value thus reduces to a localized threshold test on raw responses:
\begin{equation} \label{prac_p}
p_{j}=\frac{\sum_{i\in\mathcal{I}(X_{n+j})}\mathbbm{1}\{Y_{i}<c_{j}\}+U_{j}\cdot\left(1+\sum_{i\in\mathcal{I}(X_{n+j})}\mathbbm{1}\{Y_{i}=c_{j}\}\right)}{n_j+1}.
\end{equation}
The Benjamini--Hochberg procedure is applied separately within each
stratum at the nominal FDR level $q$, yielding
$\mathcal R=\bigcup_{v\in\mathcal V}\mathcal R_v$.
The complete procedure is summarized in Algorithm~\ref{alg:ics}.

The isotonic fit is shared with sICP and is performed only once. For each
test candidate, ICS requires a rank computation within its local calibration
set, costing $O(n_j)$ for
$n_j=|\mathcal I(X_{n+j})|$. Applying BH additionally requires sorting the
$m_v$ localized $p$-values within each active stratum $v$, for a total cost
of $O(\sum_{v\in\mathcal V}m_v\log m_v)$.

\begin{algorithm}[!htbp]
\small
\caption{Isotonic Conformal Selection (ICS)}
\label{alg:ics}

\KwIn{Dataset $\mathcal D$, base predictor $f$,
thresholds $\{c_j\}_{j=1}^m$, target FDR $q$,
test pool $\{X_{n+j}\}_{j=1}^m$}
\KwOut{Selection set $\mathcal R$}

Split $\mathcal D$ into calibration folds
$\mathcal D_1$ and $\mathcal D_2$
with index sets $\mathcal{I}_1$ and $\mathcal{I}_2$\;
Fit isotonic calibrator $\hat\theta$ on $\mathcal D_1$
and set $\hat f\leftarrow\hat\theta\circ f$\;

$\mathcal V\leftarrow
\{\hat f(X_{n+j}):j=1,\dots,m\}$,
$\mathcal R\leftarrow\emptyset$\;

\ForEach{$v\in\mathcal V$}{
    $G_v\leftarrow\{j:\hat f(X_{n+j})=v\}$,
    $m_v\leftarrow|G_v|$\;
    $\mathcal I_v\leftarrow
    \{i\in \mathcal{I}_2:\hat f(X_i)=v\}$\;

    \ForEach{$j\in G_v$}{
        $U_j\sim\operatorname{Uniform}[0,1]$\;
        Compute $p_j$ using Eq.~\eqref{prac_p}
        over $\mathcal I_v$\;
    }

    Sort $p_{(1)}^v\le\dots\le p_{(m_v)}^v$\;
    $\hat k_v\leftarrow
    \max\!\left(
        \left\{
            k\in\{1,\dots,m_v\}:
            p_{(k)}^v\le\frac{kq}{m_v}
        \right\}\cup\{0\}
    \right)$\;

    \If{$\hat k_v>0$}{
        $\mathcal R\leftarrow\mathcal R\cup
        \{j\in G_v:p_j\le p_{(\hat k_v)}^v\}$\;
    }
}

\Return $\mathcal R$\;
\end{algorithm}

\paragraph{Remark (connection to tICP).}
ICS conceptually parallels Transductive Isotonic Conformal Prediction (tICP). While tICP achieves exact self-calibration by calculating a transductive local mean for each candidate, our one-sided selection framework sidesteps this computation. Given a test sample $X_{n+j}$ and a candidate threshold $c_j$, the transductive base prediction over the local calibration set $\mathcal{I}(X_{n+j})$ is defined as:
  $$\hat{\mu}^{(c_j)}(X_{n+j}) = \frac{1}{\vert{}\mathcal{I}(X_{n+j})\vert{} + 1} \left( \sum_{i \in \mathcal{I}(X_{n+j})} Y_i + c_j \right).$$
  
  To test the null hypothesis $H_{0,j}: Y_{n+j} \le c_j$ under this transductive framework, the corresponding non-conformity scores become $V_i^{(c_j)} = Y_i - \hat{\mu}^{(c_j)}(X_{n+j})$ for $i \in \mathcal{I}(X_{n+j})$, and the threshold score is $\widetilde{V}_{n+j}^{(c_j)} = c_j - \hat{\mu}^{(c_j)}(X_{n+j})$. Substituting these transductive scores into the indicator function yields a mathematical cancellation: $Y_i - \hat{\mu}^{(c_j)}(X_{n+j}) < c_j - \hat{\mu}^{(c_j)}(X_{n+j}) \iff Y_i < c_j.$ Consequently, the localized $p$-value constructed under the transductive setting remains algebraically identical to the simplified form in Equation (\ref{prac_p}). This equivalence demonstrates that the transductive update does not alter the localized $p$-values. Conditioning on the level sets of a single isotonic calibration map is sufficient to match the exact guarantees of the transductive procedure.
\subsection{Theoretical Results}
\label{sec:ics-theory}
In this section, we establish two finite-sample error guarantees for ICS. We first establish per-selection error rate (PSER) control for individual candidates under exchangeability; for a single candidate at level $\gamma$, PSER is the joint error probability $\Pbb(p_j\le\gamma,\;j\in\mathcal H_0)$, and its conditional versions use the corresponding conditional probabilities. We then establish stratum-conditional FDR control for the selections produced by BH under i.i.d.\ sampling for the randomized $p$-values, followed by an exchangeable-data extension using deterministic $p$-values.
These guarantees rely on the static partition of ICS, which uses a single isotonic fit and separates partition fitting from localized $p$-value computation. Throughout this section, the pretrained base predictor is fixed and the calibration split is chosen independently of all observations. Under joint exchangeability of the calibration and test data, Fact~3.7 of \citet{angelopoulosTheoreticalFoundationsConformal2026} ensures that $\mathcal D_2$ and the test observations remain exchangeable conditional on $\mathcal D_1$. Because the fitted map $\hat f$ is fixed under this conditioning, selecting observations in one of its level sets preserves exchangeability by Lemma~4.8 therein. Proposition~\ref{prop:pser} shows that this within-stratum exchangeability directly yields localized superuniformity of the oracle $p$-value $p_j^*$ and finite-sample PSER control for the practical $p$-value $p_j$.

\begin{proposition}[Localized validity and PSER control]
\label{prop:pser}
\label{prop:ics-pser}
Assume that the combined calibration and test observations
$\{(X_i,Y_i)\}_{i=1}^{n+m}$ are exchangeable. Consider ICS as defined
in Algorithm~\ref{alg:ics}.
For every $\gamma\in[0,1]$ and every $v$ in the range of $\hat{f}$, the oracle $p$-value satisfies
$\Pbb(p_j^*\le\gamma\mid\mathcal D_1,\hat f(X_{n+j})=v)\le\gamma$,
and the practical $p$-value in Equation~\eqref{prac_p} satisfies
\begin{align}
\Pbb\!\left(p_j\le\gamma,\;j\in\mathcal H_0
\;\middle|\;\mathcal D_1,\hat f(X_{n+j})=v\right)&\le\gamma,
\label{eq:pser-stratum-d1}\\
\Pbb(p_j\le\gamma,\;j\in\mathcal H_0)&\le\gamma.
\label{eq:pser_marg}
\end{align}
The conditional bound holds almost surely.
\end{proposition}
The proof is given in Appendix~\ref{app:proof_prop1}. Averaging the conditional bound~\eqref{eq:pser-stratum-d1} over either conditioning variable also yields PSER control conditional only on $\mathcal D_1$ or only on the fitted prediction value, while averaging over both yields the marginal bound~\eqref{eq:pser_marg}.

We now turn to false discovery rate control across multiple test candidates. To evaluate error control within each prediction stratum $v$, we target the conditional criterion
\begin{equation}
\label{eq:stratum-fdr-conditional}
\mathrm{FDR}_v(\mathcal D_1,\mathbf G)
=
\E\!\left[
\mathrm{FDP}_v
\;\middle|\;
\mathcal D_1,\mathbf G
\right],
\qquad
\mathbf G=\{G_u:u\in\mathcal V\},
\end{equation}
which conditions on the data used to fit the isotonic partition and on the
resulting stratum assignments. Because localized $p$-values within the same stratum $v$ share the local calibration sample $\mathcal I_v$, they are mutually dependent. Under i.i.d.\ sampling, we characterize this dependence through positive regression dependence on a subset (PRDS) \citep{benjaminicontrol2001}.

\begin{definition}[PRDS]
\label{def:prds}
A random vector $\mathbf{W} \in \mathbb{R}^m$ is positive regression dependent on a subset $I_0 \subseteq \{1,\ldots,m\}$ (PRDS on $I_0$) if, for any $j \in I_0$ and any non-decreasing set $A \subseteq \mathbb{R}^m$, the conditional probability $\Pbb(\mathbf{W} \in A \mid W_j = w)$ is non-decreasing in $w$.
\end{definition}

Theorem~\ref{thm:ics-fdr} establishes our main selection guarantee: under i.i.d.\ sampling, applying BH to the randomized localized $p$-values controls the FDR within every stratum conditional on $(\mathcal D_1,\mathbf G)$.

\begin{theorem}[Stratum-conditional FDR control under i.i.d.\ sampling]
\label{thm:ics-fdr}
Assume that the combined calibration and test observations $\{(X_i,Y_i)\}_{i=1}^{n+m}$ are i.i.d. For each test candidate $j \in G_v$, let $p_j$ be the randomized localized $p$-value defined in Equation~\eqref{prac_p}. Then, for every active stratum $v\in\mathcal V$, applying the Benjamini--Hochberg procedure at nominal level $q$ yields
\[
\mathrm{FDR}_v(\mathcal D_1,\mathbf G)\le q
\quad\text{almost surely in }(\mathcal D_1,\mathbf G).
\]
Consequently, averaging over $\mathcal D_1$ and $\mathbf G$ yields $\mathrm{FDR}_v\le q$ for each stratum $v$.
\end{theorem}
The proof of Theorem~\ref{thm:ics-fdr} is given in Appendix~\ref{A.4}. Under i.i.d.\ sampling, conditioning on $\mathcal D_1$ and the calibration and test stratum memberships yields i.i.d.\ observations within each level set. Lemma~\ref{lem:ics-local-hybrid-prds} in Appendix~\ref{A.4} shows that a hybrid within-stratum $p$-value vector, with its $j$-th coordinate replaced by the oracle $p_j^*$, satisfies PRDS on $p_j^*$ following \citet{jinSelection2023}; averaging the resulting BH bound over the calibration stratum memberships completes the proof.

In practical high-throughput screening, candidates are often sampled without replacement from a finite library, which may violate the i.i.d.\ assumption. However, as long as the calibration and test samples are drawn uniformly at random, joint exchangeability is preserved. Theorem~\ref{thm:ics-fdr-exch} establishes that ICS retains stratum-conditional FDR control in this broader exchangeable setting when deterministic localized $p$-values are used.

\begin{theorem}[Stratum-conditional FDR control under exchangeability]
\label{thm:ics-fdr-exch}
Assume that the combined calibration and test observations $\{(X_i,Y_i)\}_{i=1}^{n+m}$ are exchangeable, and that the responses have no ties among themselves or with the thresholds $\{c_j\}_{j=1}^m$ almost surely. For each test candidate $j \in G_v$, define the deterministic localized $p$-value:
\[
p_j = \frac{1 + \sum_{i \in \mathcal{I}_v} \mathbbm{1}\{Y_i < c_j\}}{|\mathcal{I}_v| + 1}.
\]
Then, for every active stratum $v\in\mathcal V$, applying the Benjamini--Hochberg procedure at nominal level $q$ yields
\[
\mathrm{FDR}_v(\mathcal D_1,\mathbf G)\le q
\quad\text{almost surely in }(\mathcal D_1,\mathbf G).
\]
Consequently, averaging over $\mathcal D_1$ and $\mathbf G$ yields $\mathrm{FDR}_v\le q$ for each stratum $v$.
\end{theorem}
The proof of Theorem~\ref{thm:ics-fdr-exch} is given in Appendix~\ref{A.5}. Rather than relying on PRDS, the proof exploits the permutation symmetry of the exchangeable responses via a leave-one-out BH argument with modified $p$-values within each static stratum.

\paragraph{Remark (pooled FDR).}
The guarantees in Theorems~\ref{thm:ics-fdr} and~\ref{thm:ics-fdr-exch}
apply within each stratum. They do not imply FDR control for
the pooled selection set
$\mathcal R=\bigcup_{v\in\mathcal V}\mathcal R_v$.
Example~\ref{app:ics-pooled-fdr-counterexample} in Appendix~\ref{app:selection-counterexamples}
constructs an i.i.d. instance with randomized $p$-values in which the pooled FDR exceeds $q$.

\section{Numerical Studies\label{sec:EXPERIMENTS}}
We evaluate isotonic conformal inference through both the prediction and selection tasks. The simulation studies evaluate prediction and selection in smooth and discontinuous regression settings, examining how the procedures adapt to predictor miscalibration. We then present real data experiments and assess the practical performance of ICS in healthcare utilization and molecular prediction and selection.

In the prediction setting, we compare our ICP methods to standard marginal CP \citep{vovkAlgorithmic2005}, SC-CP \citep{laanSelfCalibratingConformalPrediction2024a}, Mondrian CP \citep{vovk2003mondrian}, and CQR \citep{romano2019conformalized}. We report empirical coverage at the $1- \alpha = 0.9$ level, average interval width, and calibration error, examining performance both overall and within relevant groups or prediction bins. We additionally compare the runtime of sICP, and the grid and closed form tICP to SC-CP.

For the selection task, we introduce several existing methods and natural adaptations from the standard CS framework. Later, we employ these methods as baselines and compare their performance against our proposed method.

The most straightforward approach is to apply standard conformal selection
(CS), which uses marginal nonconformity scores followed by the BH procedure and
may overlook local heteroscedasticity. We also consider Bonferroni-adjusted
conformal selection (CS-Bonf), which replaces BH with the Bonferroni correction
\citep{dunnMultiple1961} to control the family-wise error rate. This correction
is typically conservative \citep{pernegerWhats1998, westfallResamplingbased1993}
and can substantially reduce selection power.

Another powerful baseline extends the Conformalized Quantile Regression (CQR) \citep{romano2019conformalized} to the selection task, which we denote as CQR-CS. To achieve this, CQR-CS first fits a quantile regression model to estimate a low quantile, denoted as $\hat{q}_{\text{low}}(X)$. It then defines the nonconformity score as the positive deviation from this lower bound, i.e., $V_i = Y_i - \hat{q}_{\text{low}}(X_i)$. Finally, CQR-CS evaluates a test candidate by computing its hypothetical score $\widetilde{V}_{n+j}^{c_j} = c_j - \hat{q}_{\text{low}}(X_{n+j})$ at the threshold $c_j$. By ranking this score against the empirical distribution of $\{V_i\}$ on the calibration set, it yields a valid $p$-value for the null hypothesis $H_{0,j}: Y_{n+j} \le c_j$. While this approach preserves finite-sample marginal FDR control, its power heavily relies on the quality of $\hat{q}_{\text{low}}(X)$. However, directly estimating conditional quantiles in high-dimensional feature spaces is challenging. Unlike conditional mean estimation, quantile estimation requires modeling the tail behavior of the distribution, which is highly sensitive to data sparsity. In high-dimensional regimes, localized neighborhoods become sparse, causing quantile estimates like $\hat{q}_{\text{low}}(X)$ to be highly unstable. Such localized misestimations inflate the variance of nonconformity scores, leading to a deficiency in selection power.

More recently, the localized conformal $p$-value \citep{wuConditional2025} was introduced to target conditional testing by weighting calibration samples via a kernel function. However, in high-dimensional feature spaces, it often suffers from the curse of dimensionality, where distance-based weights collapse and result in numerical instability, making it challenging to produce valid and powerful $p$-values. We refer to this approach as LCP.

\subsection{Continuous Simulation Study}
\label{sec:shared-synthetic-dgp}
We begin by considering a continuous regression design adapted from \citet{laanSelfCalibratingConformalPrediction2024a}. Fix parameters $d \in \mathbb{N}$, $\kappa > 0$, and $a,b \geq 0$. Each dataset consists of i.i.d.\
observations $\{(X_i,Y_i)\}_{i=1}^{n}$, where the covariate vector
$X=(X_1,\dots,X_d)\in[0,1]^d$ has independent coordinates with $X_j\sim\mathrm{Beta}(1,\kappa)$. Conditional on $X = x$, the response is Gaussian with $Y\mid X=x\sim\mathcal{N}\!\bigl(\mu(x),\sigma^2(x)\bigr)$. Where \[\mu(x):=d^{-1/2}\sum_{j=1}^{d}\bigl(x_j+\sin(4x_j)\bigr) \text{ and }  \sigma^2(x):=
\left\{
0.035-a\,\frac{\log(0.5+0.5x_1)}{8}
+b\,\frac{|\mu(x)|^6/20-0.02}{2}
\right\}^{2}.\]
The parameter $\kappa$ determines the covariate distribution, while $a$ and $b$ control the level of heteroskedasticity and the strength of the mean variance relationship.

\paragraph{Continuous Prediction Task}
For the conformal prediction task, we train a random forest regressor on $n = 1000$ data points generated using $\kappa_{\text{train}} \in \{1.5,2,2.5,3\}.$  Calibration and test data are generated using $\kappa =1$ throughout, preserving exchangeability between calibration and test observations. Thus the base predictor is trained on a covariate distribution that increasingly differs from the one used for calibration and testing.

Table~\ref{tab:synthetic_coverage} shows empirical marginal coverage near the nominal level $1- \alpha = 0.9$ for all methods across all values of $\kappa_{\text{train}}$ As the shift increases, the reported calibration metric for Marginal CP and CQR moves farther away from zero. By comparison, the metric remains close to zero for sICP and tICP, this is particularly pronounced for larger values of $\kappa_{\text{train}}$ These results illustrate the calibration behaviour of the isotonic methods while highlighting that coverage remains near the desired target level.

The clearest practical difference in this experiments is the computational cost. Figure~\ref{fig:speedup} reports total runtime for calibration and interval construction for 1000 test points on a logarithmic scale. Mean total runtime is approximately $0.20$ seconds for sICP, $0.37$ seconds for closed form tICP, $2.51$ seconds for grid search based tICP, and $82.87$ seconds for SC-CP. Thus empirically, the closed-form implementation is approximately $6.8$ times faster than grid-based tICP. Relative to SC-CP, sICP and closed-form
tICP are approximately $415$ and $224$ times faster, respectively.

\begin{table}[!htbp]
\centering
\caption{Prediction interval performance on simulated data for varying
$\kappa_{\mathrm{train}}$. The test distribution is fixed at $\kappa=1$,
with nominal coverage $1-\alpha=0.9$. sICP and tICP denote split and
transductive isotonic conformal prediction, respectively.}
\label{tab:synthetic_coverage}
\footnotesize
\setlength{\tabcolsep}{5pt}
\renewcommand{\arraystretch}{1.1}

\begin{tabular}{@{}lrrr@{\hspace{1.5em}}rrr@{}}
\toprule
& \multicolumn{3}{c}{$\kappa_{\mathrm{train}}=1.5$}
& \multicolumn{3}{c}{$\kappa_{\mathrm{train}}=2.0$} \\
\cmidrule(lr){2-4}
\cmidrule(l){5-7}
Method
& Cov & Width & Cal Err
& Cov & Width & Cal Err \\
\midrule
sICP
& 0.900 & 7.46 & -0.0194
& 0.905 & 8.46 &  0.0373 \\
tICP
& 0.896 & 7.54 &  0.0177
& 0.907 & 8.45 &  0.0092 \\
Closed-form tICP
& 0.897 & 7.65 &  0.0177
& 0.908 & 8.55 &  0.0092 \\
Marginal CP
& 0.897 & 7.57 &  0.0323
& 0.896 & 7.88 &  0.1860 \\
CQR
& 0.898 & 6.12 & -0.0664
& 0.897 & 7.19 &  0.0677 \\
Mondrian CP
& 0.897 & 7.60 &  0.0323
& 0.904 & 8.41 &  0.1860 \\
SC-CP
& 0.889 & 6.97 &  0.0223
& 0.899 & 7.93 &  0.0145 \\
\midrule
\addlinespace[0.6em]
& \multicolumn{3}{c}{$\kappa_{\mathrm{train}}=2.5$}
& \multicolumn{3}{c}{$\kappa_{\mathrm{train}}=3.0$} \\
\cmidrule(lr){2-4}
\cmidrule(l){5-7}
Method
& Cov & Width & Cal Err
& Cov & Width & Cal Err \\
\midrule
sICP
& 0.904 & 7.93 & -0.0934
& 0.904 & 7.88 & -0.0528 \\
tICP
& 0.897 & 7.84 & -0.0701
& 0.901 & 7.96 &  0.0737 \\
Closed-form tICP
& 0.899 & 7.93 & -0.0701
& 0.904 & 8.06 &  0.0737 \\
Marginal CP
& 0.897 & 8.23 &  0.4610
& 0.899 & 8.28 &  0.7226 \\
CQR
& 0.891 & 8.58 &  0.3550
& 0.897 & 9.60 &  0.5023 \\
Mondrian CP
& 0.901 & 8.51 &  0.4610
& 0.901 & 8.39 &  0.7226 \\
SC-CP
& 0.902 & 7.69 & -0.0427
& 0.901 & 7.66 &  0.0230 \\
\bottomrule
\end{tabular}
\end{table}

\begin{figure}[!htbp]
    \centering
    \includegraphics[width=\textwidth]{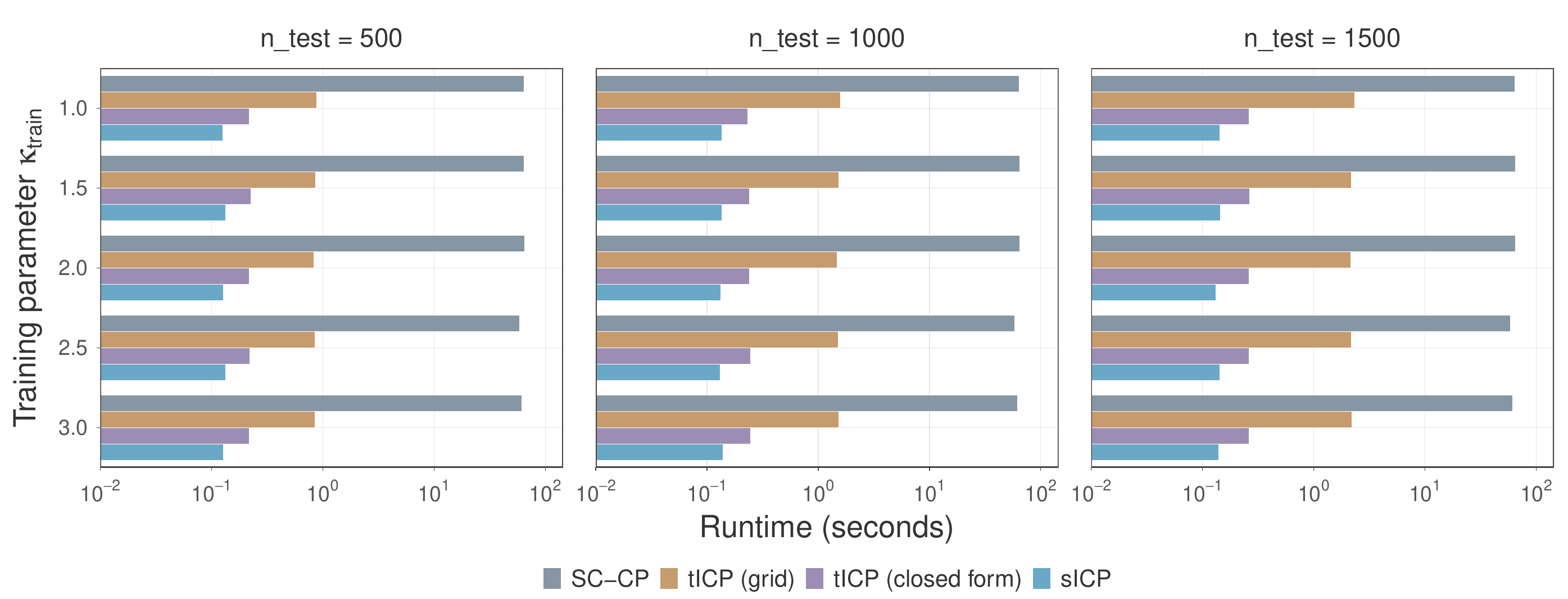}
    \caption{
        Runtime comparison on the simulated experiments as the training-distribution parameter
        $\kappa_{\mathrm{train}}$ varies. Panels correspond to $n_{\mathrm{test}} \in \{500, 1000, 1500\}$, with training and calibration sample sizes fixed at $1000$ each.  Bars report the total wall-clock time required for calibration and prediction interval construction. The x-axis is shown on a logarithmic scale. Closed-form tICP substantially reduces the computational cost of the transductive procedure relative to grid-based tICP, while SC-CP is considerably more computationally expensive than the isotonic conformal procedures.
    }
    \label{fig:speedup}
\end{figure}

\paragraph{Continuous Selection Task}
We use the data-generating process and Random Forest specification described above, with \(d=5\) and
\(\kappa=1\) for the training, calibration, and test distributions. Thus, the
selection study is i.i.d. and does not introduce a covariate shift. We consider
a homoscedastic setting with \(a=b=0\) and a heteroscedastic setting with
\(a=0\) and \(b=0.02\). The former has constant conditional standard deviation
\(0.035\), whereas the latter introduces the same mean-variance dependence
used in the prediction experiments.

We draw one training sample of 1,000 observations and fit the base predictor
once. Matching the prediction study, the Random Forest uses
\texttt{n\_estimators}=500, \texttt{max\_features}=\texttt{sqrt},
\texttt{min\_samples\_leaf}=5, and \texttt{max\_depth}=\texttt{None}. We then
perform 100 Monte Carlo
replications, drawing an independent calibration sample of size 1,000 and an
independent test sample of size 1,000 in each replication. The selection task
tests \(H_{0,j}:Y_{n+j}\leq2\) at
\(q\in\{0.05,0.10,0.15,0.20,0.30,0.40,0.50\}\). Full SC-CS and the baseline
methods use all calibration observations, while ICS splits them evenly into
\(\mathcal D_1\), used to fit the isotonic calibrator, and \(\mathcal D_2\),
used to construct the conformal $p$-values.

\paragraph{Base Predictors.} We consider two predictor regimes:
\begin{itemize}
    \item \textbf{Standard predictor:} Predictions are generated by the fixed
    Random Forest described above; CQR-CS uses the corresponding Quantile
    Random Forest.
    \item \textbf{Miscalibrated predictor:} To evaluate robustness to monotone distortion, we apply the transformation $g(u)=\log(1+u)$ to the fitted predictions $f(X)$ (and to $\hat{q}_{\mathrm{low}}(X)$ for CQR-CS).
\end{itemize}

\begin{figure*}[!h]
 \centering
\includegraphics[width=\textwidth]{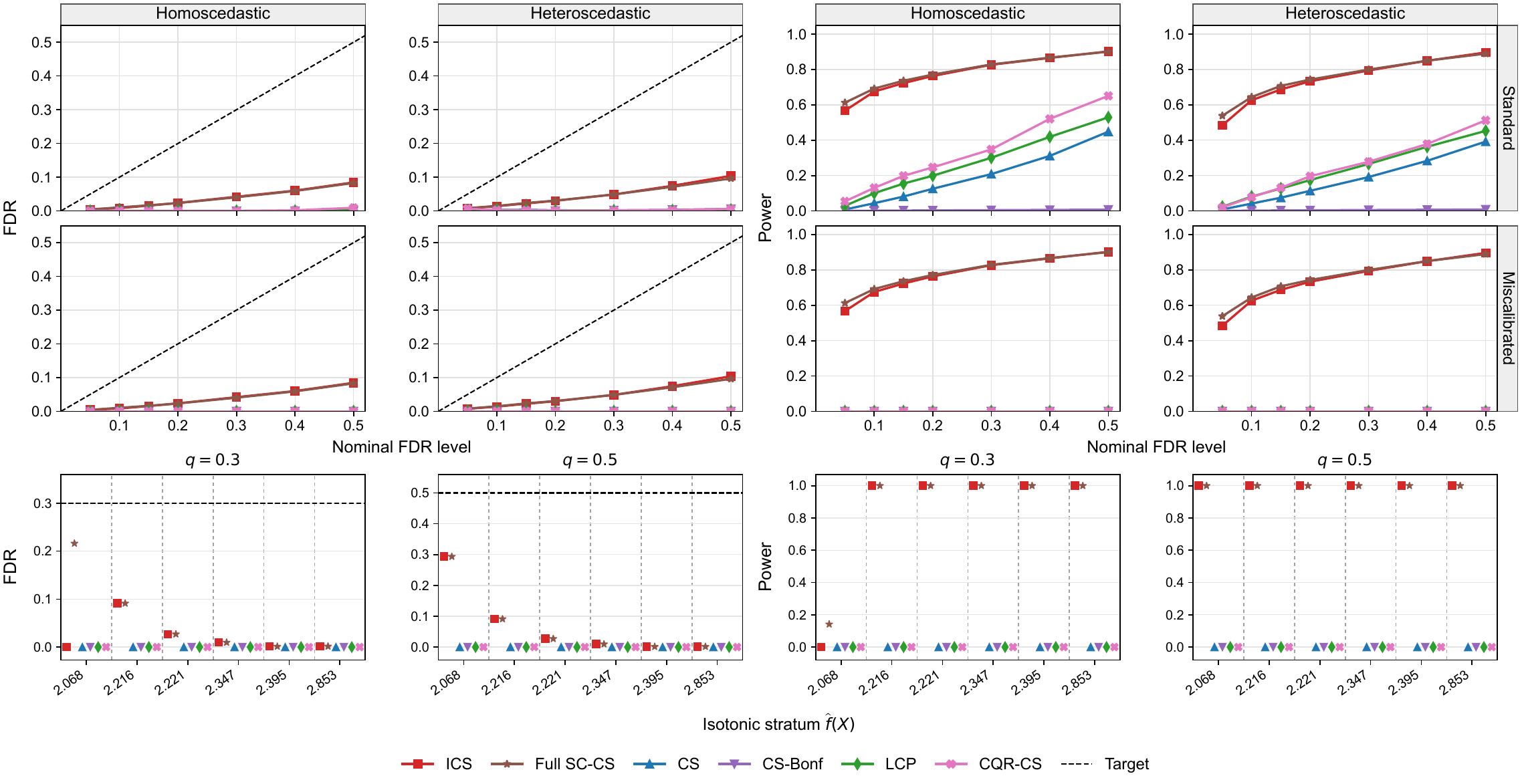}
 \caption{Marginal and stratum-conditional performance of the residual-score
 procedures. The upper two rows show pooled empirical FDR and power under
 homoscedastic and heteroscedastic noise with standard and monotonically
 transformed base predictions. The bottom row reports stratum-conditional FDR
 and power in the heteroscedastic, miscalibrated setting. Results are averaged over 100 replications; dashed lines denote the corresponding nominal FDR levels.}
 \label{Fig1}
\end{figure*}

Following \citet{laanSelfCalibratingConformalPrediction2024a}, we approximate
the isotonic calibrator with a single univariate monotone XGBoost tree. We use
the squared-error objective and one boosting round with \texttt{max\_depth}=20,
\texttt{min\_child\_weight}=20, \texttt{eta}=1, \texttt{gamma}=0,
\texttt{lambda}=0, and \texttt{monotone\_constraints}=\texttt{(1)}.
The minimum child weight prevents excessively small calibration strata. While Theorem~\ref{thm:ics-fdr} formally guarantees stratum-wise conditional FDR control rather than pooled marginal control, we report both marginal and conditional metrics in Figure~\ref{Fig1} to show that ICS empirically maintains pooled FDR well below 
$q$. Under homoscedasticity,
ICS and Full SC-CS achieve similar power and outperform the
unstratified residual-score baselines. For example, at \(q=0.3\), ICS attains power
\(0.827\) with pooled empirical FDR \(0.042\). The results are essentially
unchanged after the monotone transformation of the base predictions, as
expected from the rank-based isotonic construction.

The same pattern persists under heteroscedasticity. At nominal levels
\(q=0.10,0.20,0.30\), ICS attains powers \(0.626,0.734,0.795\), with empirical
FDR \(0.014,0.030,0.049\), respectively. Full SC-CS gives nearly identical
results. With the standard predictor, the corresponding powers of CS are
\(0.042,0.114,0.192\); after the monotone transformation, CS, LCP, and CQR-CS
make no discoveries, whereas ICS retains power \(0.795\) at \(q=0.30\).
Thus, the isotonic procedures preserve power under both localized noise and
monotone predictor miscalibration. A clipped-score sensitivity analysis is
reported in Appendix~\ref{app:clipped-score}.

The bottom row of Figure~\ref{Fig1} gives the stratum-conditional evaluation
under heteroscedastic noise with the miscalibrated predictor. The training
sample, fitted models,
calibration sample, and ICS split are held fixed, while an independent test
sample is generated in each of 100 replications. The figure retains strata
with an average test-set size greater than 50 and at least one nonzero reported
entry. Within each fixed stratum, the reported empirical FDR and power are the
Monte Carlo averages of the realized FDP and true-positive proportion,
respectively; a ratio is set to zero when its denominator is zero. ICS controls the empirical FDR below the corresponding nominal level in every displayed stratum. It attains unit power at \(q=0.3\) in all displayed strata
above the boundary stratum at \(\hat f(X)=2.068\), whereas the unstratified
baselines make no discoveries in these same strata. Full SC-CS is broadly
similar, with the largest difference confined to the boundary stratum.

\subsection{Discontinuous Simulation Study}
\label{sec:discon_data}
We next test a discontinuous regression setting, to evaluate the ability of calibration based procedures to adapt to abrupt changes in the relationship between the base prediction and the outcome. Adapting the design setup introduced by \citep{babiiIsotonicRegressionDiscontinuity2023}. We generate \(
Y
=
m(X)
+
\theta \mathbf{1}\{X \geq 0\}
+
\sigma \varepsilon,
\)
where
\(
m(x)=x^3+0.25x, X \sim 2\,\mathrm{Beta}(2,2)-1, \varepsilon \sim N(0,1),
\)
with $\theta=1$ and $\sigma=1$.

We assume the fixed base predictor $f(x) = m(x) = x^3 + 0.25x$, which does not model the jump at $x = 0$. The base predictor $f$ is continuous,  strictly increasing, and $f(0) = 0$, conditioning on $f(x)$ is the same as conditioning on $x$ itself. Consequently, \(
\mathbb{E}[Y\mid f(X)=t]
=
t+\mathbbm{1}\{t\geq 0\}.
\)
This provides a setting in which the target of isotonic calibration is known explicitly but the base predictor misses the discontinuity of the conditional mean. 

\paragraph{Discontinuous Prediction Task}
We evaluate each baseline method using independent training, calibration and testing data, where the training data is used only by methods such as CQR that require an additional regression fit. In addition to overall coverage and average interval width, we examine performance on either side of the discontinuity and within narrow neighborhoods of $x=0$.

Although all methods achieve overall coverage close to the $1- \alpha = 0.9$ level, as reported in Table~\ref{tab:babii_kumar_main}, their behaviors across the discontinuity at $X = 0$ differs. Marginal CP achieves $0.901$ overall coverage. This however masks a substantial imbalance in the behavior across the discontinuity. We observe coverage of $0.957$ for $X  < 0$ but this drops to $0.845$ for $X \geq 0.$ The single residual quantile used by the marginal method is incapable of adapting to the discontinuity at $X =0$. In contrast to this, the calibration based procedures correct this issue and achieve coverage near $0.9$ across both regions despite the abrupt change in the population calibration relationship. sICP obtains empirical coverage of $0.907$ and $0.917$ on the two sides of the discontinuity, respectively, while tICP obtains $0.909$ and $0.913$. SC-CP similarly produces nearly identical coverage on the two sides.

Figure \ref{fig:babii_kumar_bins} shows empirical coverage across bins of the base prediction $f(X).$ Marginals CP's coverage shift at $X = 0$ is well illustrated. In contrast, the other methods remain substantially closer to $0.9$ across the prediction range.

\begin{figure}[!htbp]
    \centering
    \includegraphics[width=0.85\linewidth]
    {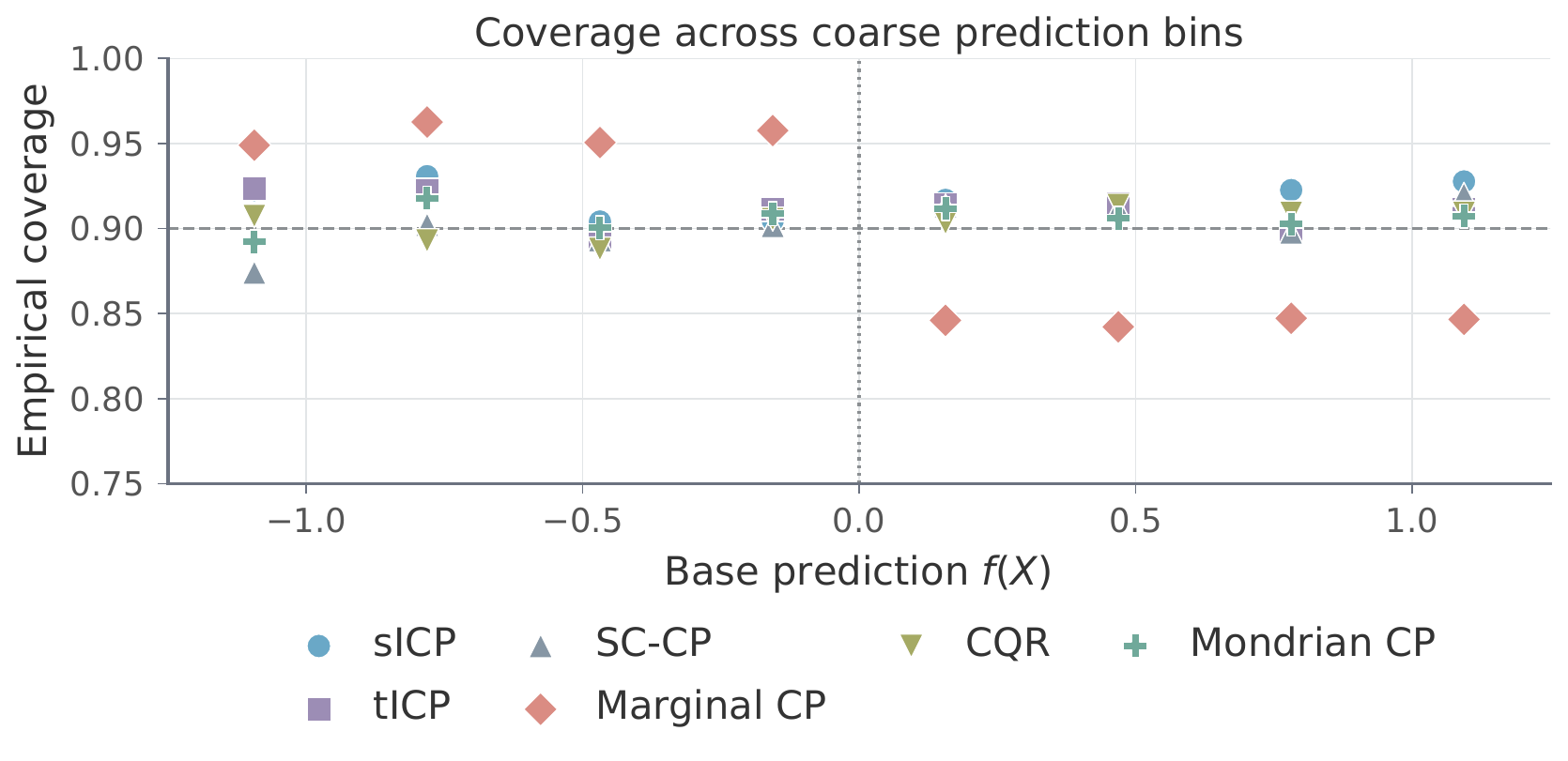}
    \caption{
    Empirical prediction interval coverage across coarse bins of the
    base prediction $f(X)$ in the discontinuous regression setting.
    The horizontal dashed line denotes the nominal coverage level
    $1-\alpha=0.9$, and the vertical dotted line marks the
    discontinuity at $f(X)=0$.
    }
    \label{fig:babii_kumar_bins}
\end{figure}

Notably the improvement in coverage does not come solely from larger prediction intervals. sICP and tICP have an average interval width of 3.456 and 3.485 respectively,
compared with $4.049$ for marginal conformal prediction, while maintaining
substantially more balanced coverage across the discontinuity. SC-CP produces the narrowest intervals across the tested methods with an average width of $3.376$. We note that through a single calibration fit sICP and tICP manage an interval width similar to what was achieved by SC-CP. CQR also adapts well to the discontinuity, attaining similar
coverage on either side of zero with an average width of $3.567$. This is
unsurprising in this setting, since CQR is trained directly as a function of
$X$ and can therefore learn the change in the conditional response near the
cutoff.

Interval widths reveal a further distinction between the methods, results can be found in Table~\ref{tab:babii_kumar_width}. sICP,tICP, and SC-CP produce very similar interval widths across both sides of the discontinuity. On the other hand, while Mondrian CP achieves the desired coverage in both regions, there is a lot of discrepancy in the interval width across both sides. Mondrian CP becomes substantially wider for \(X\geq0\), increasing from \(3.392\) to \(4.698\), indicating that calibration based procedures yield more stable local uncertainty than relying on fixed  partitions of the data.

\paragraph{Discontinuous Selection Task}
\label{sec:selection-discontinuous}
We next examine whether the selection procedures adapt to an abrupt change in
the relationship between the base prediction and the outcome. We use the
discontinuous design considered in the prediction study Section~\ref{sec:discon_data}.

We use 100 independent Monte Carlo replicates, each containing 1,000 training,
2,000 calibration, and 5,000 test observations. The fixed predictor requires
no training; the training sample is used only by CQR-CS to fit its Quantile
Random Forest. We test \(H_{0,j}:Y_j\leq c\) at
\(q\in\{0.05,0.10,0.15,0.20,0.30,0.40,0.50\}\), with the prespecified cutoff
\(c=\theta=1\). ICS splits the calibration sample equally into
\(\mathcal D_1\) and \(\mathcal D_2\), while Full SC-CS and the baselines use
all calibration observations. In addition to overall FDR and power, we examine
regional performance on either side of the discontinuity and within the same
\(0.20\)- and \(0.10\)-wide neighborhoods used in the prediction study.

\begin{figure*}[htbp]
    \centering
    \includegraphics[width=\textwidth]{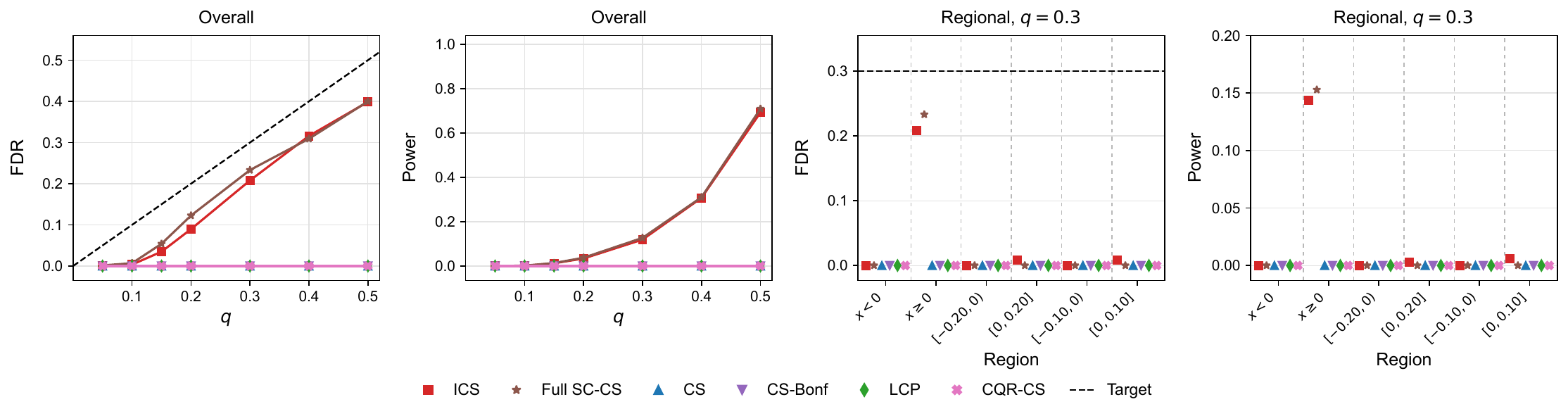}
    \caption{Residual-score selection performance in the discontinuous simulation study. The left two panels report overall empirical FDR and
    power across nominal levels. The right two panels report empirical
    regional FDR and power at \(q=0.3\), with method-specific symbols grouped
    by region. Interval
    labels denote subsets of \(x\). Results are averaged over 100 independent
    train--calibration--test replicates; dashed lines denote the nominal FDR
    level.}
    \label{fig:babii-kumar-selection}
\end{figure*}

Figure~\ref{fig:babii-kumar-selection} shows that the isotonic procedures adapt
to the omitted jump while their empirical FDR remains below every tested
nominal level. At \(q=0.3\), ICS attains
empirical FDR \(0.208\) and power \(0.119\), while Full SC-CS attains \(0.233\)
and \(0.127\), respectively. At \(q=0.5\), their powers increase to \(0.693\)
and \(0.709\), with empirical FDR \(0.399\) and \(0.400\). CS, CS-Bonf, LCP,
and CQR-CS make no discoveries over the reported nominal levels.

The regional results localize this difference. At \(q=0.3\), the isotonic
procedures make discoveries exclusively on the \(x\geq0\) side, where
the omitted jump raises the conditional mean. Within this region, ICS and Full
SC-CS attain empirical regional FDR \(0.208\) and \(0.233\), with regional
powers \(0.143\) and \(0.153\), respectively. Discoveries become sparse in the
narrow neighborhoods immediately adjacent to zero, where the conditional mean
is close to the selection threshold. Thus, isotonic calibration recovers
selection power from the discontinuity that is absent from the fixed smooth
predictor.

\subsection{Drug Discovery Application}\label{drug}

We apply Isotonic Conformal Inference  to drug discovery, where the
objective is to quantify uncertainty about molecular properties and identify chemical compounds with desirable biological
properties. Each compound is represented by a structural vector $X$, and the response $Y$ measures an assay specific ADMET. In the prediction setting we construct intervals for these responses. In the selection setting we identify compounds whose response exceeds a prespecified biological threshold while controlling selection errors within calibrated prediction
strata. In such situations reliable inference is important because promising compounds typically undergo costly downstream experimental validation.

We employ an aggregated ADMET dataset comprising \(n=22{,}805\) compounds with 15 biological assay responses, compiled and imputed by \citet{baiMultivariate2025a}. The raw data originates from multiple public sources \citep{wenzelPredictive2019, iwataPredicting2022, kimPubChem2023, watanabePredicting2018, falcon-canoReliable2022, espositoCombining2020, bragaPredhERG2015, aliagasComparison2022, perrymanPruned2020, mengBoosting2022, vermeirePredicting2022}. Missing values within this integrated dataset were imputed using Chemprop \citep{yangAnalyzing2019, heidChemprop2024}. We adopt the assay-specific biological thresholds established by \citet{baiMultivariate2025a}, under which approximately \(15\%\)--\(30\%\) of the compounds qualify as acceptable. Further details on the threshold choices, response distributions, and biological interpretation of the assays are provided in the supplementary materials of \citet{baiMultivariate2025a}.

\paragraph{Prediction for Drug Discovery}
Table \ref{tab:admet_results} reports overall prediction performance on the ADMET dataset where the target outcome is CL\_microsome\_human. All methods attain empirical coverage above the desired $0.9$ level. The isotonic procedures produce intervals that are slightly wider than the ones produced by baseline SC-CP but attain comparable coverage. Figure~\ref{fig:admet_prediction_bands} illustrate how their calibrated predictions and interval widths vary across the range of the base prediction model. 

The distinction between the ICP methods and marginal CP becomes cleared when coverage is examined across prediction quartiles. Table~\ref{tab:admet_quartile_coverage} reports results averaged over 10 random splits with quartile boundaries determined from the base prediction in each split. Although marginal CP achieves overall coverage of 0.903, its coverage falls to 0.871 and 0.880 in Q2 and Q3 respectively. ICP methods maintain coverage above 0.9 across the quartiles, with the exception of closed form tICP which exhibits coverage of 0.883 in Q2. These results illustrate how approximately nominal overall coverage can conceal differences in coverage across the prediction range. Corresponding interval widths are reported in Table~\ref{tab:admet_quartile_width}, where the ICP intervals are slightly wider compared to marginal CP, in order to maintain balanced coverage across all groups. The transductive nature of SC-CP allows for coverage to remain balanced across all groups while keeping interval width small.

The cost of this is computational, Figure~\ref{fig:admet_runtime} highlights this computational tradeoff. Both sICP and closed form tICP require substantially less time than SC-CP for calibration and interval construction. The closed form construction also reduces runtime relative to grid based inversion. Thus we see that ICP combines balanced empirical coverage across prediction quartiles with comparatively inexpensive calibration and interval building.
\begin{table}[!htbp]
\centering
\caption{Results for the ADMET dataset on CL\_microsome\_human. Coverage is reported at nominal level $1-\alpha=0.9$.}
\label{tab:admet_results}

\begin{tabular}{lccc}
\toprule
\textbf{Method}
& \textbf{Coverage}
& \textbf{Avg Width}
& \textbf{Cal Error} \\
\midrule

sICP        & 0.931 & 1.84 & -0.0204 \\
tICP & 0.916 & 1.77 & -0.0217 \\
Closed-form tICP         & 0.928 & 1.78 & -0.0217 \\
Marginal CP              & 0.912 & 1.65 & -0.0295 \\
CQR                      & 0.907 & 1.74 & -0.0020 \\
Mondrian CP              & 0.920 & 1.65 & -0.0295 \\
SC-CP                    & 0.920 & 1.62 & -0.0178 \\

\bottomrule
\end{tabular}
\end{table}

\begin{figure}[H]
    \centering
    \includegraphics[width=\textwidth]{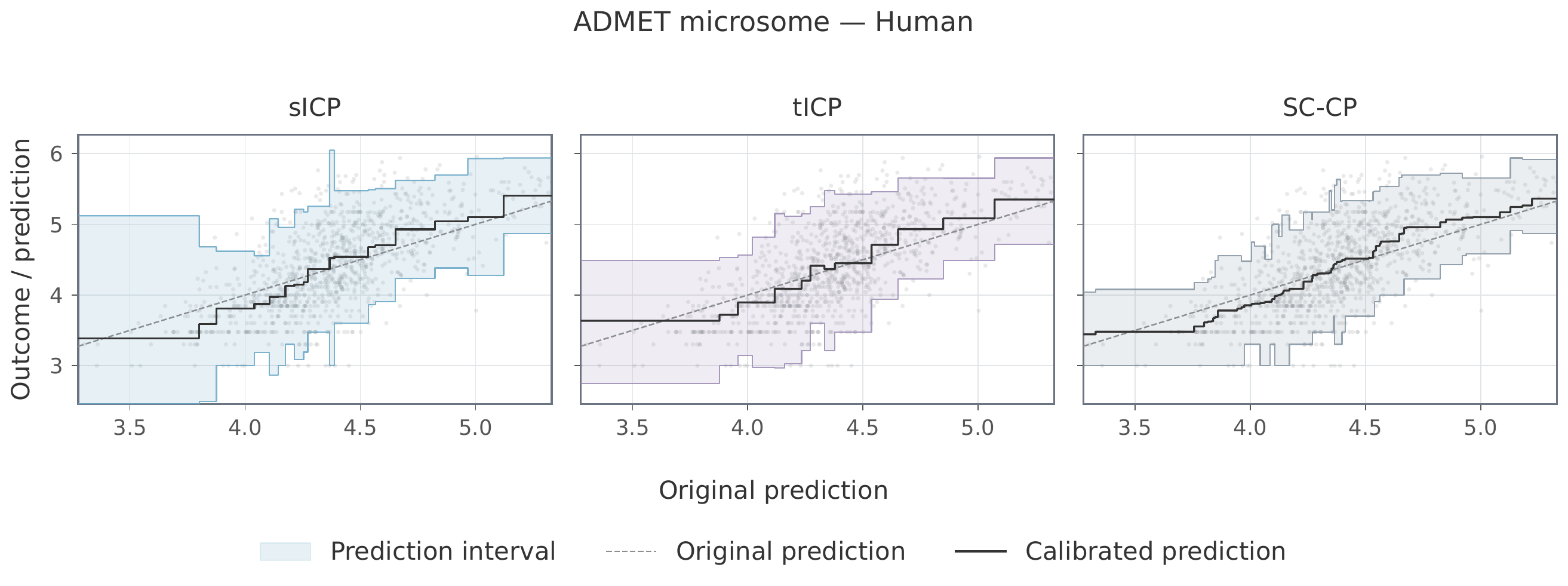}
    \caption{
        Prediction bands for sICP, closed-form tICP, and SC-CP on the ADMET dataset.
        Each panel shows the observed outcomes against the original model prediction,
        together with the calibrated prediction and the corresponding conformal
        prediction interval.
    }
    \label{fig:admet_prediction_bands}
\end{figure}

\begin{figure}[!htbp]
    \centering
    \includegraphics[width=0.95\textwidth]{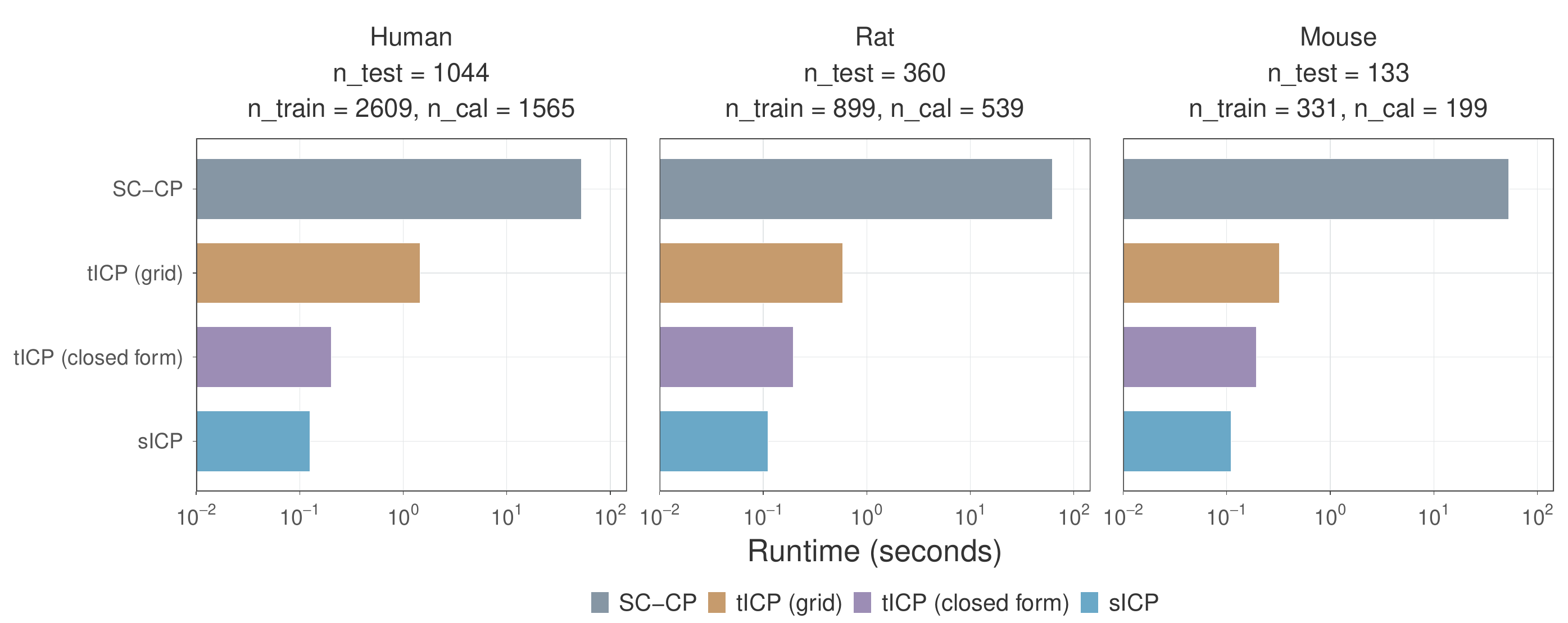}
    \caption{
        Runtime comparison for sICP, grid-based tICP, closed-form tICP, and SC-CP
        on the ADMET dataset. Bars report total method runtime, including calibration
        and prediction interval construction. The x-axis is shown on a
        logarithmic scale.
     }
    \label{fig:admet_runtime}
\end{figure}

\paragraph{Selection for Drug Discovery}
\paragraph{Marginal evaluation.}
Each compound is encoded by a 2,048-dimensional Morgan fingerprint. For every
assay, we randomly allocate 12,000 compounds to a fixed training pool and
reserve the remaining 10,805 for evaluation. Within the training pool, 9,600
observations fit candidate Random Forest regressors and 2,400 observations
select the configuration with the smallest validation mean squared error. The
same configuration is selected for all 15 assays: 500 trees, unrestricted
depth, \texttt{min\_samples\_leaf=10},
\texttt{min\_samples\_split=20}, and \texttt{max\_features=1.0}. The selected
model is refitted once on all 12,000 training observations and then held
fixed. The Quantile Random Forest used by CQR-CS has the same tree
configuration.

For the marginal evaluation, we perform 100 replications for each assay,
randomly repartitioning the evaluation pool into 8,000 calibration and 2,805
test observations. ICS divides the calibration sample equally into
\(\mathcal D_1\), used to fit the isotonic calibrator, and \(\mathcal D_2\),
used to construct the conformal $p$-values. Full SC-CS and the baseline
methods use all 8,000 calibration observations. We report results at
\(q\in\{0.3,0.5\}\). Both the calibration--test partition and the ICS
\(\mathcal D_1/\mathcal D_2\) split are random and independent of the
observed responses.

\paragraph{Stratum-conditional evaluation.}
We additionally examine stratum-conditional behavior on the
CL\_microsome\_human assay, whose threshold is \(c=4\).
Here the fitted model, calibration sample, ICS split, and resulting isotonic
strata are fixed. In each of 100 replications, a test sample of size 2,805 is drawn with replacement from a fixed held-out test pool. For each stratum \(v\)
and bootstrap replication \(b\), we compute the realized
\(\mathrm{FDP}^{(b)}_v\); the reported empirical conditional FDR is
\(\widehat{\mathrm{FDR}}_v=100^{-1}\sum_{b=1}^{100}
\mathrm{FDP}^{(b)}_v\). The main analysis uses the signed residual score
\(V(x,y)=y-\hat f(x)\), while Appendix~\ref{app:clipped-score} reports the
clipped-score sensitivity analysis.

Direct kernel localization in the 2,048-dimensional binary fingerprint space leads to severe weight degeneration. We therefore implement LCP \citep{wuConditional2025} using the standardized Random Forest prediction as a one-dimensional, training-derived localization feature. Following the bandwidth rate proposed for one-dimensional localization, we use a symmetric box kernel with
\(h=n_{\mathrm{cal}}^{-1/3}=8000^{-1/3}=0.05.
\)
This construction avoids numerical underflow while retaining localization with respect to a predictive summary of the molecular features.

\begin{figure*}[htbp]
    \centering
    \includegraphics[width=\textwidth]{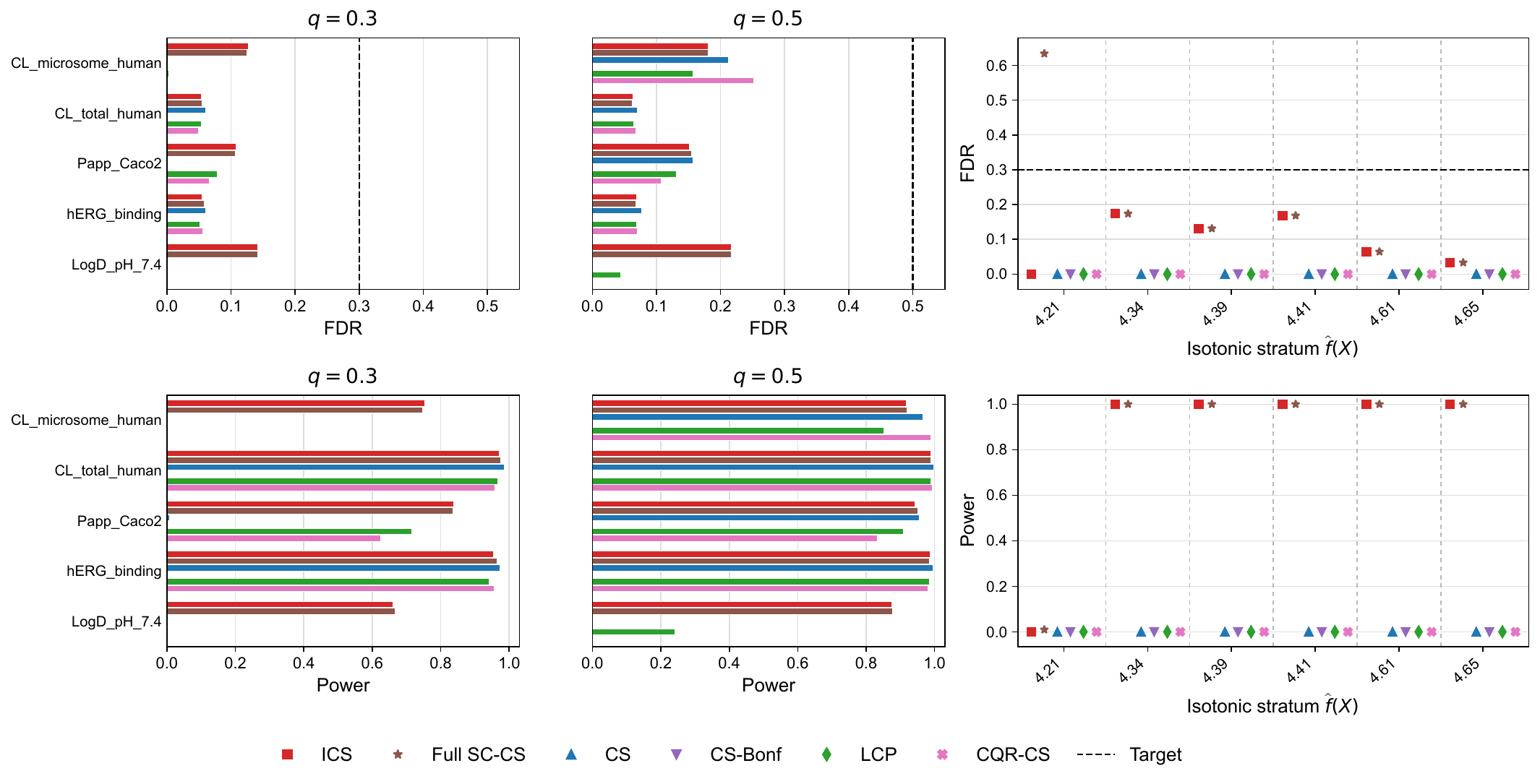}
    \caption{Marginal and stratum-conditional performance on the ADMET data.
    The left four panels show marginal empirical FDR and power across five representative assays at \(q=0.3\) and \(q=0.5\), averaged over 100 calibration--test
    partitions. The right panels show empirical stratum-conditional FDR and
    power for CL\_microsome\_human at \(q=0.3\); conditional FDR is
    the average of the realized within-stratum FDP over 100 bootstrap test
    samples. Only strata with average size greater than 125 and at least one
    nonzero entry are displayed. Method-specific symbols are grouped by
    stratum, with vertical gray dashed lines separating strata; black dashed
    lines denote the nominal FDR level. The remaining ten assays are reported in Appendix~\ref{app:additional-admet}, Table~\ref{tab:real-data-residual-additional}.
    All procedures use residual scores except CQR-CS, which retains its
    quantile score.}
    \label{fig:real-data}
\end{figure*}

\paragraph{Marginal results.}
The marginal results in Figure~\ref{fig:real-data} show that every method
remains below the nominal FDR level across all 15 assays. ICS closely tracks
Full SC-CS: their powers differ by at most \(0.011\) over all assay--level
combinations, although ICS uses only half of the calibration observations
for conformal ranking. At \(q=0.3\), their mean powers are \(0.898\) and
\(0.899\), compared with \(0.761\), \(0.772\), and \(0.785\) for CS, LCP,
and CQR-CS, respectively. The difference is concentrated in the more
difficult assays. For CL\_microsome\_human,
LogD\_pH\_7.4, and Papp\_Caco2, ICS attains powers
\(0.753\), \(0.661\), and \(0.837\) at \(q=0.3\), whereas CS attains
\(0.001\), \(0\), and \(0.007\). At \(q=0.5\), most procedures approach
unit power, and the mean power of ICS increases to \(0.968\). Thus, the
benefit is not uniform across assays: the methods are similar on easier
tasks, whereas isotonic localization avoids the severe power losses of
the global procedures on the most challenging assays.

\paragraph{Stratum-conditional results.}
The right panels of Figure~\ref{fig:real-data} summarize six displayed
strata under the fixed-calibration bootstrap diagnostic. In the five
higher-prediction strata, both ICS and Full SC-CS select every test
observation in each bootstrap replication, yielding power one and empirical
conditional FDR between \(0.033\) and \(0.174\). Their main difference in
these panels occurs at \(\hat f(X)\approx4.214\), with an average of
\(149.3\) test observations.
ICS makes no discoveries in this stratum,
whereas Full SC-CS selects \(2.98\) observations on average, with empirical
conditional FDR \(0.634\) and power
\(0.010\). CS, CS-Bonf, LCP, and CQR-CS make no discoveries in the six
displayed strata; their zero FDR therefore accompanies zero power.
These results illustrate the role of the ICS partition in conditional selection.

\section{Conclusion\label{sec:CONCLUSION}}

In this paper, we introduced Isotonic Conformal Inference, a unified
framework for prediction intervals and multiple testing with
finite-sample reliability guarantees conditional on a calibrated point
prediction. The framework exploits the level set structure induced by
isotonic calibration. For prediction, Split Isotonic CP and Transductive
Isotonic CP provide prediction-conditional coverage, together with their
respective calibration guarantees, while avoiding repeated isotonic
refitting across candidate responses. For selection, ICS uses this
structure to construct localized $p$-values that support
stratum-conditional false discovery rate control. Together, these
methods connect calibrated predictions with uncertainty quantification
and selection through a common isotonic structure.

The sample splitting steps in our framework reduce the amount of data
available at each stage, potentially limiting statistical efficiency,
particularly in small samples. An important direction for future work
is to use the partition induced by isotonic to generalize the framework to other inference tasks such as risk control, generalized risk control, and multivariate conformal selection.


\newpage{}

\addcontentsline{toc}{section}{REFERENCES}
\setlength{\bibsep}{0.0pt}\footnotesize 

\bibliographystyle{rss}
\bibliography{ref}

\clearpage
\normalsize
\appendix
\renewcommand{\thesection}{\Alph{section}.}
\renewcommand{\thesubsection}{\Alph{section}.\arabic{subsection}}

\makeatletter
\let\@thanks\@empty
\makeatother
\author{%
Daniel Bensimon \thanks{Equal contribution; Department of Mathematics and Statistics, McGill University; Mila}
\and
Zhiquan Liu \thanks{Equal contribution; School of Statistics and Data Science, Nankai University}
\and
Sean Xiang Yu \thanks{Eli Lilly and Company}
\and
Eric D. Kolaczyk \thanks{Department of Mathematics and Statistics, McGill University; Mila}
\and
Archer Y. Yang \thanks{Department of Mathematics and Statistics, McGill University; Mila; Corresponding author: archer.yang.yi@gmail.com}%
}
\title{Supplementary Materials for "Isotonic Conformal Inference"}
\date{}
\maketitle

\section{Proofs}
\label{app:proofs}

\subsection{Proof of Theorem~\ref{thm:split_coverage}}
\label{app:proof_split_coverage}

\begin{proof}
We first note that, conditional on $\D_1$, the calibration observations in $\D_2$ together with the test point $(X_{n+1},Y_{n+1})$ are exchangeable. This is an immediate consequence of Fact 3.7 in \citet{angelopoulosTheoreticalFoundationsConformal2026}. Now define
\(
\mathcal S := \{\, i \in \mathcal I_2 \cup \{n+1\} : \hat f(X_i) = \hat f(X_{n+1}) \,\},
\)
and let $Z_i := (X_i,Y_i), \text{ for } i \in \mathcal I_2 \cup \{n+1\}.$

Now fix a deterministic $s \subseteq \mathcal{I}_2 \cup \{n+1\}$ containing $\{n+1\}$, and define the event $E_s = \{\mathcal{S} = s\}.$
Let \(\pi\) be an arbitrary permutation of the index set $s$. Extend this to a permutation $\sigma$ of $\mathcal{I}_2 \cup \{n+1\}$ by \[
\sigma(i) =
\begin{cases}
\pi(i), & i \in s,\\
i, & i \notin s.
\end{cases}
\]

On the event $E_s$ all observations in $s$ have the same prediction so the event is unchanged under the permutation $\pi.$ Furthermore, the value of $\hat f(X_{n+1}) = \hat f(X_{\sigma({n+1})})$ is also unchanged.

Then for any measurable sets $A,B$ of appropriate dimensions, we get by exchangeability conditional on $\D_1$
\begin{align*}
\mathbb P\!\left( (Z_i)_{i\in s} \in A,  E_s,\hat f(X_{n+1})\in B\,\middle|\, \D_1 \right) = \mathbb P\!\left( (Z_{\sigma(i)})_{i\in s} \in A,  E_s,\hat f(X_{n+1})\in B\,\middle|\, \D_1 \right)  \\ = \mathbb P\!\left( (Z_{\pi(i)})_{i\in s} \in A,  E_s,\hat f(X_{n+1})\in B\,\middle|\, \D_1 \right)
\end{align*}

Where the first equality comes from exchangeability conditional on $\mathcal{D}_1$, and the second is because in the set $s$ the permutations $\pi$ and $\sigma$ behave the same. 
Thus we get,
\begin{align*}
&\mathbb P\!\left(
    (Z_i)_{i\in s}\in A
    \,\middle|\,
    \D_1,E_s,\hat f(X_{n+1})
\right)
=
\mathbb P\!\left(
    (Z_{\pi(i)})_{i\in s}\in A
    \,\middle|\,
    \D_1,E_s,\hat f(X_{n+1})
\right).
\end{align*}

Since this holds for every possible subset $s$ containing $n+1$ and every permutation of $s$, the observations $(Z_i)_{i\in\mathcal S}$ are exchangeable conditional on
$(\D_1,\mathcal S,\hat f(X_{n+1}))$.

 By construction $\{n+1\} \in \mathcal{S}$, and since the observations indexed by $\mathcal{S}$ are exchangeable conditional on
$(\D_1,\mathcal S,\hat f(X_{n+1}))$, the scores $\mathcal{E}$ are as well. 

Now define \[
p_i = \frac{1 + \sum_{j \in \mathcal{S}\setminus\{i\}}
\mathbf{1}\{R_j \ge R_i\}}{N},
\qquad i \in \mathcal{S}\text{ and } N = |\mathcal{S}|.
\] We now show that $p_{n+1}$ is superuniform conditional on $(\D_1,\mathcal{S},\hat{f}(X_{n+1}))$

For any realization of the scores, order their indices so that $R_{i_1} \geq R_{i_2} \geq ... \geq R_{i_N}.$ At least $j$ scores are greater than or equal to $R_{i_j},$ so $p_{i_j} \geq \frac{j}{N}$
Hence for $u \in [0,1],$ 
\begin{align*}
\frac{1}{N}\sum_{i\in \mathcal{S}}\mathbbm{1}\{p_i \leq u\} = \frac{1}{N}\sum_{i = 1}^N\mathbbm{1}\{p_{i_j} \leq u\} \leq \frac{1}{N}\sum_{i = 1}^N\mathbbm{1}\{j/N \leq u\} = \frac{\lfloor uN\rfloor}{N} \leq u
\end{align*}

Then using conditional exchangeability
\[\prob(p_{n+1} \leq u \mid  (\D_1,\mathcal S,\hat f(X_{n+1}))) = \mathbbm{E}\bigl[\frac{1}{N}\sum_{i\in \mathcal{S}}\mathbbm{1}\{p_i \leq u\}\mid (\D_1,\mathcal S,\hat f(X_{n+1}))\bigr] \leq u.\]

Making use of the following equivalences, \[
\begin{aligned}
p_{n+1}>\alpha
&\iff
\sum_{i\in\mathcal S\setminus\{n+1\}}
\mathbbm{1}\{R_i<R_{n+1}\}<N(1-\alpha)\\
&\iff R_{n+1}\le Q\\
&\iff Y_{n+1}\in\widehat C(X_{n+1}).
\end{aligned}
\]
 We thus have
\[
\begin{aligned}
&\mathbb P\!\left(
Y_{n+1}\in\widehat C(X_{n+1})
\,\middle|\,
\mathcal D_1,\mathcal S,\hat f(X_{n+1})
\right)\\
&\qquad=
1-\mathbb P\!\left(
p_{n+1}\le\alpha
\,\middle|\,
\mathcal D_1,\mathcal S,\hat f(X_{n+1})
\right)
\ge 1-\alpha
\end{aligned}
\]
almost surely. Taking conditional expectations and applying the tower property yields
\[
\mathbb P\!\left(
Y_{n+1}\in\widehat C(X_{n+1})
\,\middle|\,
\hat f(X_{n+1})
\right)\ge 1-\alpha.
\] 
\end{proof}

\subsection{Proof of Proposition~\ref{prop:isotonic_empirical_calibration}}
\label{app:split_empirical_calibration}
\begin{proof}[Proof]
Fix a block $\mathcal B_k$. By construction of the isotonic fit,
$\hat f(X_i)=v_k$ for every $i\in \mathcal B_k$. Once the final
constant block $\mathcal B_k$ has been formed, the value assigned to that
block must minimize the local squared-error loss
\[
    c \mapsto \sum_{i\in \mathcal B_k} (Y_i-c)^2.
\]
Therefore,
\[
    v_k
    =
    \arg\min_{c\in\mathbb R}
    \sum_{i\in \mathcal B_k} (Y_i-c)^2.
\]
Differentiating with respect to $c$ gives
\[
    -2\sum_{i\in \mathcal B_k}(Y_i-c)=0.
\]
Hence,
\[
    c
    =
    \frac{1}{|\mathcal B_k|}
    \sum_{i\in \mathcal B_k}Y_i.
\]
Since the fitted value on the block is $v_k$, we obtain
\[
    v_k
    =
    \frac{1}{|\mathcal B_k|}
    \sum_{i\in \mathcal B_k}Y_i.
\]
Because $k$ was arbitrary, the identity holds for every
$k=1,\ldots,K$.
\end{proof}

\subsection{Proof of Theorem~\ref{thm:asymptotic_calibration_split}}
\label{app:asymptotic_calibration}
\begin{proof}[Proof]
Under Assumptions 1 and 2, classical consistency results for isotonic regression \citep[e.g.,][]{robertson1988order} dictate that the estimator $\hat{\theta}_{n_1}$ converges uniformly to the true conditional expectation function $\theta_0$ over the compact support $\mathcal{S}$. That is,
\begin{equation}
\label{eq:uniform_conv}
\sup_{s \in \mathcal{S}} \bigl| \hat{\theta}_{n_1}(s) - \theta_0(s) \bigr| \xrightarrow{a.s.} 0 \quad \text{as } n_1 \to \infty.
\end{equation}

For a fixed query $x \in \mathcal{X}$, let $v = \hat{f}_{n_1}(x) = \hat{\theta}_{n_1}(f(x))$ denote the predicted value. To evaluate the population calibration error, we apply the tower property of conditional expectation. Conditional on the split $\mathcal{D}_1$, the mapping $\hat{\theta}_{n_1}$ is deterministic. Conditioning further on the true base prediction $f(X_{n+1})$, we have:
\begin{align}
\mathbb{E}\bigl[\,Y_{n+1} \mid \hat{f}_{n_1}(X_{n+1}) = v, \mathcal{D}_1\bigr] 
&= \mathbb{E}\Bigl[\, \mathbb{E}\bigl[Y_{n+1} \mid f(X_{n+1}), \mathcal{D}_1\bigr] \Bigm| \hat{f}_{n_1}(X_{n+1}) = v, \mathcal{D}_1 \,\Bigr] \nonumber \\
&= \mathbb{E}\Bigl[\, \theta_0(f(X_{n+1})) \Bigm| \hat{f}_{n_1}(X_{n+1}) = v, \mathcal{D}_1 \,\Bigr]. \label{eq:tower_prop}
\end{align}
The second equality follows directly from the independence of $(X_{n+1}, Y_{n+1})$ and $\mathcal{D}_1$, whereby $\mathbb{E}[Y_{n+1} \mid f(X_{n+1}), \mathcal{D}_1] = \mathbb{E}[Y_{n+1} \mid f(X_{n+1})] = \theta_0(f(X_{n+1}))$.

Subtracting the predicted value $v$ from both sides, and noting that conditional on the event $\hat{f}_{n_1}(X_{n+1}) = v$, we can trivially move the constant $v$ inside the expectation:
\begin{align}
\mathbb{E}\bigl[\,Y_{n+1} \mid \hat{f}_{n_1}(X_{n+1}) = v, \mathcal{D}_1\bigr] - v 
&= \mathbb{E}\Bigl[\, \theta_0(f(X_{n+1})) - v \Bigm| \hat{f}_{n_1}(X_{n+1}) = v, \mathcal{D}_1 \,\Bigr] \nonumber \\
&= \mathbb{E}\Bigl[\, \theta_0(f(X_{n+1})) - \hat{f}_{n_1}(X_{n+1}) \Bigm| \hat{f}_{n_1}(X_{n+1}) = v, \mathcal{D}_1 \,\Bigr]. \label{eq:diff_expectation}
\end{align}
Taking the absolute value and applying Jensen's inequality yields:
\begin{align}
\Bigl| \mathbb{E}\bigl[\,Y_{n+1} \mid \hat{f}_{n_1}(X_{n+1}) = v, \mathcal{D}_1\bigr] - v \Bigr| 
&\leq \mathbb{E}\Bigl[\, \bigl| \theta_0(f(X_{n+1})) - \hat{f}_{n_1}(X_{n+1}) \bigr| \Bigm| \hat{f}_{n_1}(X_{n+1}) = v, \mathcal{D}_1 \,\Bigr] \nonumber \\
&\leq \sup_{s \in \mathcal{S}} \bigl| \theta_0(s) - \hat{\theta}_{n_1}(s) \bigr|. \label{eq:sup_bound}
\end{align}
Crucially, the upper bound in \eqref{eq:sup_bound} is entirely independent of the specific query value $x$. Taking the supremum over all $x \in \mathcal{X}$ yields:
\begin{equation}
\sup_{x \in \mathcal{X}} \Bigl| \mathbb{E}\bigl[\,Y_{n+1} \mid \hat{f}_{n_1}(X_{n+1}) = \hat{f}_{n_1}(x), \mathcal{D}_1\bigr] - \hat{f}_{n_1}(x) \Bigr| \leq \sup_{s \in \mathcal{S}} \bigl| \theta_0(s) - \hat{\theta}_{n_1}(s) \bigr|.
\end{equation}
By the uniform convergence established in \eqref{eq:uniform_conv}, the right-hand side converges to $0$ almost surely as $n_1 \to \infty$. This concludes the proof.
\end{proof}

\subsection{Proof of Theorem~\ref{thm:transductive_exact_self_calibration}}
\label{app:proof_transductive_exact_self_calibration}
\begin{proof}[Proof]
We analyze the oracle predictions jointly over the set of all calibration and test points. Let $\mathcal{S} = \mathcal{I}_2 \cup \{n+1\}$ denote the index set of the hold-out scoring data and the test point, with size $N = n_2 + 1$. 

Conditional on the proper training set $\mathcal{D}_1$, the mapping $\hat{f}$ is fixed. For any $j \in \mathcal{S}$, we define its corresponding oracle transductive prediction symmetrically as the empirical mean of the responses within its designated bin:
\begin{equation}
\hat{\mu}_j = \frac{1}{|\mathcal{S}_j|} \sum_{k \in \mathcal{S}_j} Y_k, \quad \text{where } \mathcal{S}_j = \bigl\{ k \in \mathcal{S} : \hat{f}(X_k) = \hat{f}(X_j) \bigr\}.
\end{equation}
Let $g: \mathbb{R} \to \mathbb{R}$ be any bounded measurable function. Consider the sum of the product of $g(\hat{\mu}_j)$ and the local residuals $(Y_j - \hat{\mu}_j)$ across all points $j \in \mathcal{S}$:
\begin{equation}
\label{eq:sum_residuals}
\sum_{j \in \mathcal{S}} g(\hat{\mu}_j) (Y_j - \hat{\mu}_j).
\end{equation}
We group this sum by the distinct level-set values of $\hat{f}$. Let $\mathcal{V}$ be the set of unique values taken by $\hat{f}(X_j)$ for $j \in \mathcal{S}$. For each $v \in \mathcal{V}$, let $\mathcal{S}_v = \{ j \in \mathcal{S} : \hat{f}(X_j) = v \}$. By definition, for all $j \in \mathcal{S}_v$, the oracle prediction is identically: $\hat{\mu}_j = \frac{1}{|\mathcal{S}_v|} \sum_{k \in \mathcal{S}_v} Y_k =: \bar{Y}_v$.

The sum in \eqref{eq:sum_residuals} can thus be algebraically factored:
\begin{equation}
\sum_{j \in \mathcal{S}} g(\hat{\mu}_j) (Y_j - \hat{\mu}_j) = \sum_{v \in \mathcal{V}} \sum_{j \in \mathcal{S}_v} g(\bar{Y}_v) (Y_j - \bar{Y}_v) = \sum_{v \in \mathcal{V}} g(\bar{Y}_v) \underbrace{\left( \sum_{j \in \mathcal{S}_v} (Y_j - \bar{Y}_v) \right)}_{=\, 0}.
\end{equation}
Because the sum of deviations from the arithmetic mean within any fixed group is precisely zero, the entire sum over the dataset is strictly zero for any realization of the data:
\begin{equation}
\label{eq:zero_sum_algebra}
\sum_{j \in \mathcal{S}} g(\hat{\mu}_j) (Y_j - \hat{\mu}_j) = 0.
\end{equation}

Taking the expectation conditional on $\mathcal{D}_1$, we obtain:
\begin{equation}
\mathbb{E}\left[ \sum_{j \in \mathcal{S}} g(\hat{\mu}_j) (Y_j - \hat{\mu}_j) \Biggm| \mathcal{D}_1 \right] = 0.
\end{equation}
Crucially, since the pairs $\{(X_j, Y_j)\}_{j \in \mathcal{S}}$ are exchangeable conditional on $\mathcal{D}_1$, and the mapping to compute $\hat{\mu}_j$ is perfectly symmetric with respect to permutations of $\mathcal{S}$, the expectation of each individual term in the sum is identical. Therefore, focusing on the test point $n+1$:
\begin{equation}
N \cdot \mathbb{E}\left[\, g(\hat{\mu}_{n+1}) (Y_{n+1} - \hat{\mu}_{n+1}) \mid \mathcal{D}_1 \,\right] = 0.
\end{equation}
Because this equality holds for any bounded measurable function $g$, it immediately follows from the definition of conditional expectation that:
\begin{equation}
\mathbb{E}\left[\,Y_{n+1} \mid \hat{\mu}_{n+1}, \mathcal{D}_1\,\right] = \hat{\mu}_{n+1} \quad \text{a.s.}
\end{equation}
Finally, by the tower property of conditional expectation, marginalizing over $\mathcal{D}_1$ yields $\mathbb{E}[\,Y_{n+1} \mid \hat{\mu}_{n+1}\,] = \hat{\mu}_{n+1}$, establishing the exact finite-sample self-calibration property.
\end{proof}

\subsection{Proof of Theorem~\ref{thm:transductive_coverage}}
\label{app:proof_transductive_coverage}
\begin{proof}[Proof]
Condition on $\mathcal{D}_1$ so that $\hat{f} = \hat{\theta} \circ f$ is a fixed deterministic function. The remaining data $\{X_i,Y_i\}_{i \in \mathcal{I}_2 \cup \{n+1\}}$ remain exchangeable conditional on $\mathcal{D}_1$. Throughout the rest of the proof, we assume that we have conditioned on $\mathcal{D}_1$.

For a test point $X_{n+1}$, define the group
\[
\mathcal{I}(X_{n+1}) = \{i \in \mathcal{I}_2 : \hat{f}(X_i) = \hat{f}(X_{n+1})\}.
\]
Conditional on $\mathcal{I}(X_{n+1})$ and $\{(X_i,Y_i)\}_{i\in\mathcal{I}_2\setminus\mathcal{I}(X_{n+1})}$, the collection $\{(X_i,Y_i)\}_{i \in \mathcal{I}(X_{n+1}) \cup \{n+1\}}$ is exchangeable.

Relabel the data in $\mathcal{I}(X_{n+1})$ as $\{(X_j,Y_j)\}_{j=1}^{m}$, where $m=|\mathcal{I}(X_{n+1})|$. For each $y\in\mathcal{Y}$, define
\[
\begin{aligned}
\hat{\mu}^{(y)}(X_{n+1})
&=
\frac{1}{m+1}\left(\sum_{j=1}^{m}Y_j + y\right),\\
V_j^{(y)}
&=\left|Y_j-\hat{\mu}^{(y)}(X_{n+1})\right|\quad (j=1,\dots,m),\\
V_{n+1}^{(y)}
&=\left|y-\hat{\mu}^{(y)}(X_{n+1})\right|.
\end{aligned}
\]
Now take $y=Y_{n+1}$. The statistic $\hat{\mu}^{(Y_{n+1})}(X_{n+1})$ is a symmetric function of the exchangeable pairs
$\{(X_j,Y_j)\}_{j=1}^{m}\cup\{(X_{n+1},Y_{n+1})\}$. Therefore, conditional on $\mathcal{I}(X_{n+1})$,
$\{(X_i,Y_i)\}_{i\in\mathcal{I}_2\setminus\mathcal{I}(X_{n+1})}$, and $\hat{\mu}^{(Y_{n+1})}(X_{n+1})$, the collection
$\{(X_j,Y_j)\}_{j=1}^{m}\cup\{(X_{n+1},Y_{n+1})\}$ remains exchangeable; hence the corresponding augmented scores are exchangeable as well.

Hence, by the standard conformal quantile/rank argument,
\[
\begin{aligned}
\mathbb{P}\!\Bigl(
&V_{n+1}^{(Y_{n+1})}
\le
Q_{1-\alpha}\!\left(\{V_i^{(Y_{n+1})}:i\in\mathcal{I}(X_{n+1})\}\cup\{V_{n+1}^{(Y_{n+1})}\}\right)
\,\Big|\,
\mathcal{D}_1,\mathcal{I}(X_{n+1}), \\[-2pt]
&\{(X_i,Y_i)\}_{i\in\mathcal{I}_2\setminus\mathcal{I}(X_{n+1})},\hat{\mu}^{(Y_{n+1})}(X_{n+1})
\Bigr)
\ge 1-\alpha.
\end{aligned}
\]
Using the law of total expectation, we get 
\[
\mathbb{P}\!\left(
V_{n+1}^{(Y_{n+1})}
\le
Q_{1-\alpha}\!\left(\{V_i^{(Y_{n+1})}:i\in\mathcal{I}(X_{n+1})\}\cup\{V_{n+1}^{(Y_{n+1})}\}\right)
\middle|
\hat{\mu}^{(Y_{n+1})}(X_{n+1})
\right)
\ge 1-\alpha.
\]

\end{proof}

\subsection{Proof of Proposition~\ref{prop:closed_form_ticp}}
\label{app:proof_prop_closed_form}
\begin{proof}
Write $y=\bar Y_{\mathcal I}+d$. After augmenting the level set
by $y$, the fitted value is
\(
    \hat\mu^{(y)}(X_{n+1})
    =\frac{\sum_{i\in\mathcal I}Y_i+y}{m+1}
      =\bar Y_{\mathcal I}+\frac{d}{m+1}.\)
This leads to the conformity scores for the test point     
$V_{n+1}^{(y)}
    =\frac{m}{m+1}|d|,
    \text{ and }
    V_i^{(y)}
    =\left|Y_i-\bar Y_{\mathcal I}-\frac{d}{m+1}\right|,
 i\in\mathcal I.$

If $d=0$, the test score is zero, so the candidate is accepted.

Suppose $k\le m$. A candidate is accepted if and only if its
score does not exceed the $k$-th order statistic of the $m+1$
augmented scores, or equivalently,
\[
    \sum_{i\in\mathcal I}
    \mathbbm{1}\{V_i^{(y)}<V_{n+1}^{(y)}\}
    \le k-1.
\]
The strict inequality accounts for ties: acceptance requires
at most $k-1$ scores strictly below the test score.

For $d>0$, solving the score inequality gives
\[
\begin{aligned}
    V_i^{(y)}<V_{n+1}^{(y)}
    &\iff
    \left|Y_i-\bar Y_{\mathcal I}-\frac{d}{m+1}\right|
    <\frac{m}{m+1}d\\
    &\iff -\rho d<Y_i-\bar Y_{\mathcal I}<d
    \iff r_i^+<d.
\end{aligned}
\]
Thus the rank condition becomes
$\sum_{i\in\mathcal I}\mathbbm{1}\{r_i^+<d\}\le k-1$,
which holds if and only if $d\le r_{(k)}^+$.
Including $d=0$, the accepted values on the nonnegative
half-line are therefore $0\le d\le r_{(k)}^+$.

For $d<0$, write $\ell=-d>0$. Similarly,
\[
\begin{aligned}
    V_i^{(y)}<V_{n+1}^{(y)}
    &\iff
    \left|Y_i-\bar Y_{\mathcal I}+\frac{\ell}{m+1}\right|
    <\frac{m}{m+1}\ell\\
    &\iff -\ell<Y_i-\bar Y_{\mathcal I}<\rho\ell
    \iff r_i^-<\ell.
\end{aligned}
\]
The rank condition is now
$\sum_{i\in\mathcal I}\mathbbm{1}\{r_i^-<\ell\}\le k-1$,
equivalently $\ell\le r_{(k)}^-$.
Hence the accepted values on the nonpositive half-line are
$-r_{(k)}^-\le d\le0$.

Combining both cases and recalling $y=\bar Y_{\mathcal I}+d$
yields
\[
    \widehat C_{\mathrm{tICP}}(X_{n+1})
    =
    \mathcal Y\cap
    \left[
        \bar Y_{\mathcal I}-r_{(k)}^-,
        \bar Y_{\mathcal I}+r_{(k)}^+
    \right].
\]
Finally, if $k=m+1$, the relevant order statistic is the
maximum augmented score. Since the test score is among
these scores, every $y\in\mathcal Y$ is accepted, giving
$\widehat C_{\mathrm{tICP}}(X_{n+1})=\mathcal Y$.
\end{proof}

\subsection{Proof of Proposition~\ref{prop:pser}}
\label{app:proof_prop1}
\begin{proof}[Proof]
Fix $j$. By Fact~3.7 of
\citet{angelopoulosTheoreticalFoundationsConformal2026}, the observations
in $\mathcal D_2$ together with the test observations remain exchangeable
conditional on $\mathcal D_1$. Under this conditioning, $\hat f$ is fixed,
so for each fitted level-set value $v$, \(\mathcal Z_v=\{(x,y):\hat f(x)=v\}\)
is a fixed subset of the observation space. Apply Lemma~4.8 of the same
reference to the observations indexed by
$\mathcal I_2\cup\{n+j\}$ and the bin $\mathcal Z_v$.
Conditional on the indices of the observations belonging to this bin,
those observations are exchangeable. On $\hat f(X_{n+j})=v$, these
indices are precisely $\mathcal I(X_{n+j})\cup\{n+j\}$.

Conditional further on the unordered multiset of responses in this local
set, the test response is equally likely to occupy any position, and the
independent $U_j$ randomizes its rank within ties. Averaging over the
local membership indices and the split index sets therefore gives
\begin{equation}
\Pbb\!\left(p_j^*\le\gamma
\;\middle|\;\mathcal D_1,\hat f(X_{n+j})=v\right)\le\gamma,
\qquad\gamma\in[0,1].
\label{eq:oracle-superuniform-prop-pser}
\end{equation}
On $j\in\mathcal H_0$, monotonicity of the randomized empirical rank
in its argument, using the same $U_j$, gives $p_j^*\le p_j$. Hence
\[
\{p_j\le\gamma,\;j\in\mathcal H_0\}
\subseteq\{p_j^*\le\gamma\},
\]
which proves \eqref{eq:pser-stratum-d1}.
Averaging over the fitted prediction conditional on $\mathcal D_1$
yields PSER control conditional only on $\mathcal D_1$; averaging over
$\mathcal D_1$ conditional on $\hat f(X_{n+j})=v$ yields PSER control
conditional only on the fitted prediction value. Averaging over both
conditioning variables proves the marginal bound
\eqref{eq:pser_marg}.
\end{proof}
\subsection{Proof of Theorem \ref{thm:ics-fdr}}
\label{A.4}

\begin{lemma}[Local hybrid PRDS]
\label{lem:ics-local-hybrid-prds}
Write $\mathbf C=(\hat f(X_i))_{i\in\mathcal I_2}$ for the
score-computation stratum memberships. Fix a stratum $v$ and an index
$j\in G_v$. Under the i.i.d. assumption, conditional on
$\mathcal D_1$, $\mathbf C$, and $\mathbf G$, the hybrid vector
\[
\mathbf p_v^{(j)}
=\bigl(p_j^*,\{p_\ell:\ell\in G_v\setminus\{j\}\}\bigr)
\]
is PRDS on its oracle coordinate $p_j^*$.
\end{lemma}

\begin{proof}
Conditional on $\mathcal D_1$, the isotonic map and its level sets are
fixed. Conditioning further on $\mathbf C$ and $\mathbf G$ fixes
$\mathcal I_v$ and $G_v$. Under i.i.d. sampling, the observations indexed
by $\mathcal I_v$ and $n+j$ are independent draws from the same
stratum-conditional distribution. Construct a pseudo-sample whose calibration
observations are the points indexed by $\mathcal I_v$, whose distinguished test
observation is $(X_{n+j},Y_{n+j})$, and whose remaining test observations have
responses replaced by their fixed thresholds $c_\ell$.  The conformal
$p$-values of this pseudo-sample are exactly $p_j^*$ in the distinguished
coordinate and $p_\ell$ in every other coordinate.  Lemma~10 of \citet{jinSelection2023} therefore
applies within this stratum and proves the claim.
\end{proof}

\begin{proof}[Proof of Theorem \ref{thm:ics-fdr}]
Fix a stratum $v$. Condition on $\mathcal D_1$, the calibration
stratum memberships $\mathbf C$, and the test stratum assignments
$\mathbf G$. This fixes $\mathcal I_v$, $G_v$, and $m_v=|G_v|$ while
preserving the i.i.d. law within the stratum. All probabilities and
expectations below are taken under this conditioning. Let $\mathcal R_v$
be the rejection set produced by BH in stratum $v$.

For each $j\in G_v$, let $\mathcal R_v^{(j)}$ be the BH rejection set after
replacing $p_j$ by the oracle value $p_j^*$ and leaving the other local
$p$-values unchanged.  If $j\in\mathcal H_0$, then $p_j^*\le p_j$.  Moreover,
on the event $\{|\mathcal R_v|=k,\ j\in\mathcal R_v\}$, the BH step-up rule
gives $p_j\le kq/m_v$; lowering this already rejected coordinate to $p_j^*$
does not change the number of rejections.  Therefore,
\begin{align}
\E[\mathrm{FDP}_v\mid\mathcal D_1,\mathbf C,\mathbf G]
&=\sum_{j\in G_v}\sum_{k=1}^{m_v}\frac{1}{k}
  \mathbb E\!\left[
    \mathbbm{1}\{j\in\mathcal H_0\}
    \mathbbm{1}\{|\mathcal R_v|=k\}
    \mathbbm{1}\{j\in\mathcal R_v\}
  \right] \nonumber\\
&\le \sum_{j\in G_v}\sum_{k=1}^{m_v}\frac{1}{k}
  \mathbb E\!\left[
    \mathbbm{1}\{|\mathcal R_v^{(j)}|=k\}
    \mathbbm{1}\!\left\{p_j^*\le\frac{kq}{m_v}\right\}
  \right].
\label{eq:ics-local-fdr-hybrid}
\end{align}

By Lemma~\ref{lem:ics-local-hybrid-prds}, the hybrid vector is PRDS on
$p_j^*$.  The usual BH inequality under PRDS
(\citet{benjaminicontrol2001}) gives, for every $j\in G_v$,
\[
\sum_{k=1}^{m_v}\frac{1}{k}
\mathbb E\!\left[
  \mathbbm{1}\{|\mathcal R_v^{(j)}|=k\}
  \mathbbm{1}\!\left\{p_j^*\le\frac{kq}{m_v}\right\}
\right]
\le \frac{q}{m_v}.
\]
Substitution into \eqref{eq:ics-local-fdr-hybrid} gives
$\E[\mathrm{FDP}_v\mid\mathcal D_1,\mathbf C,\mathbf G]\le q$.
Averaging over $\mathbf C$ conditional on $\mathcal D_1$ and
$\mathbf G$ yields, by the tower property,
\[
\mathrm{FDR}_v(\mathcal D_1,\mathbf G)
=\E\!\left[
\E[\mathrm{FDP}_v\mid\mathcal D_1,\mathbf C,\mathbf G]
\;\middle|\;\mathcal D_1,\mathbf G
\right]\le q.
\]
Averaging over $\mathcal D_1$ and $\mathbf G$ gives
$\mathrm{FDR}_v\le q$.
\end{proof}

\subsection{Proof of Theorem \ref{thm:ics-fdr-exch}} \label{A.5}
\begin{proof}
Fix a stratum $v$ and condition on the data-independent split index
sets, $\mathcal D_1$, and all calibration and test stratum memberships
$\mathbf C$ and $\mathbf G$. The map $\hat f$, the local calibration set
$\mathcal I_v$, and the test set $G_v$ are then fixed. Write
$N_v=|\mathcal I_v|$ and $m_v=|G_v|$.
Fact~3.7 of \citet{angelopoulosTheoreticalFoundationsConformal2026}
gives exchangeability of the remaining observations conditional on
$\mathcal D_1$. Permutations within the selected level set preserve
all the conditioning variables, so the symmetric-selection argument
of Lemma~4.8 preserves exchangeability
within that set. In particular, for every $j\in G_v$,
\(\mathcal U_j=\{Y_i:i\in\mathcal I_v\}\cup\{Y_{n+j}\}\)
is exchangeable under this conditioning. The thresholds are fixed independently of these unobserved
responses.

For $j\in G_v$, define the deterministic oracle $p$-value
\[
p_j^*=\frac{1+\sum_{i\in\mathcal I_v}
  \mathbbm{1}\{Y_i<Y_{n+j}\}}{N_v+1}.
\]
For every $\ell\in G_v\setminus\{j\}$, define
\[
p_\ell^{(j)}=
\frac{\sum_{i\in\mathcal I_v}\mathbbm{1}\{Y_i<c_\ell\}
  +\mathbbm{1}\{Y_{n+j}<c_\ell\}}{N_v+1}.
\]
Let $\mathcal R_v$ be the BH rejection set for the practical deterministic
$p$-values in Theorem~\ref{thm:ics-fdr-exch}, and let $\mathcal R_v^{(j)}$ be
the BH rejection set for the vector having coordinate $p_j^*$ at $j$ and
coordinates $p_\ell^{(j)}$ at $\ell\ne j$.

Consider the event
$E_j=\{Y_{n+j}\le c_j,\ j\in\mathcal R_v\}$.  On $E_j$,
$p_j^*\le p_j$.  If $c_\ell<c_j$, then $p_\ell\le p_j$, so $\ell$ is also
rejected, and $p_\ell^{(j)}\le p_\ell$.  If $c_\ell\ge c_j$, continuity of
the response distribution gives $Y_{n+j}<c_\ell$ almost surely on the
relevant event, and hence $p_\ell^{(j)}=p_\ell$.  Thus the modification only
decreases coordinates that were already rejected and leaves all other
coordinates unchanged.  The BH step-up rejection count is consequently the
same for $\mathcal R_v$ and $\mathcal R_v^{(j)}$ on $E_j$.

Now set the $j$-th coordinate of the modified vector to zero and denote the
resulting BH rejection set by $\mathcal R_{v,j\to0}^{(j)}$.  The standard BH
leave-one-out identity gives
\begin{equation}
\E[\mathrm{FDP}_v\mid\mathcal D_1,\mathbf C,\mathbf G,\mathcal I_1,\mathcal I_2]
\le \sum_{j\in G_v}
\mathbb E\!\left[
\frac{
\mathbbm{1}\!\left\{p_j^*\le
  q|\mathcal R_{v,j\to0}^{(j)}|/m_v\right\}}
{|\mathcal R_{v,j\to0}^{(j)}|\vee1}
\;\middle|\;\mathcal D_1,\mathbf C,\mathbf G,\mathcal I_1,\mathcal I_2
\right].
\label{eq:ics-exch-loo-bound}
\end{equation}
Indeed, this follows by decomposing the FDP according to the rejection count,
using the equality of the two rejection counts on $E_j$, and then lowering
the rejected oracle coordinate from $p_j^*$ to zero.

For $\ell\ne j$, the value $p_\ell^{(j)}$ is a symmetric function of the
unordered multiset $\mathcal U_j$.  Hence
$\mathcal R_{v,j\to0}^{(j)}$ is fixed conditional on $\mathcal U_j$.  By
continuity there are no ties, and exchangeability makes $p_j^*$ uniform on
the grid
\[
\left\{\frac{1}{N_v+1},\ldots,\frac{N_v+1}{N_v+1}\right\}
\]
conditional on $\mathcal U_j$.  Therefore, for any
$\mathcal U_j$-measurable $t\in[0,1]$,
\(\mathbb P(p_j^*\le t\mid\mathcal U_j)\le t.\)
Taking $t=q|\mathcal R_{v,j\to0}^{(j)}|/m_v$ shows that each expectation in
\eqref{eq:ics-exch-loo-bound} is at most $q/m_v$.  Summing over
$j\in G_v$ gives
$\E[\mathrm{FDP}_v\mid\mathcal D_1,\mathbf C,\mathbf G,\mathcal I_1,\mathcal I_2]\le q$.
Averaging over the calibration memberships and split index sets yields
\[
\mathrm{FDR}_v(\mathcal D_1,\mathbf G)
=\E\!\left[
\E[\mathrm{FDP}_v\mid\mathcal D_1,\mathbf C,\mathbf G,\mathcal I_1,\mathcal I_2]
\;\middle|\;\mathcal D_1,\mathbf G
\right]\le q.
\]
Finally, averaging over $\mathcal D_1$ and $\mathbf G$ gives
$\mathrm{FDR}_v\le q$.
\end{proof}
\section{Baseline: Full Self-Calibrating Conformal Selection} \label{Full SC-CS}

In this appendix, we describe Full Self-Calibrating Conformal Selection (Full SC-CS), a natural extension of Self-Calibrating Conformal Prediction (SC-CP) \citep{laanSelfCalibratingConformalPrediction2024a} to the selection problem. Building on the full conformal prediction framework \citep{vovkAlgorithmicLearningRandom2005}, Full SC-CS uses the entire calibration set $\mathcal{D} = \{(X_i, Y_i)\}_{i=1}^n$ and reconstructs the isotonic partition for each test candidate: for $X_{n+j}$ with threshold $c_j$, we refit the isotonic calibrator $\hat{\theta}^{(c_j)}$ on the calibration data augmented with $(f(X_{n+j}), c_j)$, yielding the candidate-specific calibrated predictor $\hat{f}^{(c_j)} = \hat{\theta}^{(c_j)} \circ f$. The localized $p$-value $p_j$ is then computed within the resulting dynamic level set $\mathcal{I}^{(c_j)}(X_{n+j})$, following Equation~(\ref{prac_p}); the complete procedure is summarized in Algorithm~\ref{alg:local_fsccs}. Note that in practice, candidates whose $v_j$ values agree up to floating-point tolerance are assigned to the same stratum.

\begin{algorithm}[ht]
\caption{Full Self-Calibrating Conformal Selection}
\label{alg:local_fsccs}
\DontPrintSemicolon
\KwIn{Calibration data $\mathcal{D} = \{(X_i, Y_i)\}_{i=1}^n$, base model $f$, thresholds $\{c_j\}_{j=1}^m$, target FDR $q$, test pool $\{X_{n+j}\}_{j=1}^m$}
\KwOut{Selection set $\mathcal{R}$}
\BlankLine
Initialize global selection set: $\mathcal{R} \leftarrow \emptyset$\;
\BlankLine
\For{each test candidate $X_{n+j}$ for $j=1, \dots, m$}{
    Form augmented data: $F_{\text{aug}} \leftarrow \{f(X_i)\}_{i=1}^n \cup \{f(X_{n+j})\}$, $Y_{\text{aug}} \leftarrow \{Y_i\}_{i=1}^n \cup \{c_j\}$\;
    Fit isotonic calibrator $\hat{\theta}^{(c_j)}$ on $(F_{\text{aug}}, Y_{\text{aug}})$; set $\hat{f}^{(c_j)} \leftarrow \hat{\theta}^{(c_j)} \circ f$\;
    Identify local calibration set: $\mathcal{I}^{(c_j)}(X_{n+j}) \leftarrow \left\{ i \in \{1, \dots, n\} : \hat{f}^{(c_j)}(X_i) = \hat{f}^{(c_j)}(X_{n+j}) \right\}$\;
    Compute localized $p$-value $p_j$ using Eq.~(\ref{prac_p}) over $\mathcal{I}^{(c_j)}(X_{n+j})$\;
    Record level-set value for stratification: $v_j \leftarrow \hat{f}^{(c_j)}(X_{n+j})$\;
}
\BlankLine
Let $\mathcal{V}$ be the set of unique values in $\{v_1, \dots, v_m\}$\;
\ForEach{$v \in \mathcal{V}$}{
    Define the stratum of candidates: $G_v \leftarrow \{j : v_j = v\}$\;
    Let $m_v = |G_v|$, and sort $\{p_j : j \in G_v\}$ as $p_{(1)}^v \le \dots \le p_{(m_v)}^v$\;
    Find local BH cutoff index:
    $ \hat{k}_v \leftarrow \max\!\left(\left\{k \in \{1,\dots,m_v\} : p_{(k)}^v \le \frac{k}{m_v}\, q\right\} \cup \{0\}\right) $
    
    \If{$\hat{k}_v > 0$}{
        Update global selection set: $\mathcal{R} \leftarrow \mathcal{R} \cup \big\{j \in G_v : p_j \le p_{(\hat{k}_v)}^v\big\}$\;
    }
}
\Return $\mathcal{R}$\;
\end{algorithm}

Examples~\ref{app:full-pser-counterexample} and
\ref{app:full-fdr-counterexample} in Appendix~\ref{app:selection-counterexamples} show that the stratum-conditional PSER and FDR guarantees established for ICS fail for this construction.
The oracle-rank argument in the proof of
Proposition~\ref{prop:pser} requires the stratum to be determined by data
that treat the calibration responses and the test response symmetrically.
This holds if the partition is fitted independently of the rank data, as
in ICS, or if it is fitted on the augmented data with the true test
response, as in oracle full conformal prediction. Full SC-CS fits the
partition on the calibration data augmented with $(f(X_{n+j}),c_j)$.
The calibration responses enter at their observed values, whereas the test
response enters only through the threshold $c_j$. The fitted level set thus
depends on the calibration responses but not on $Y_{n+j}$. Conditioning on
it alters the law of the calibration responses in $\mathcal I(X_{n+j})$ while
leaving that of $Y_{n+j}$ unchanged, so
$\{Y_i:i\in\mathcal I(X_{n+j})\}\cup\{Y_{n+j}\}$ is no longer exchangeable
given the level set. Example~\ref{app:full-pser-counterexample} makes this
explicit. The usual monotonicity step $p_j\ge p_j^*$ does not repair it,
because replacing $Y_{n+j}$ by $c_j$ also changes the local calibration set. For FDR, the proofs of Theorems~\ref{thm:ics-fdr} and
\ref{thm:ics-fdr-exch} (Appendices~\ref{A.4} and~\ref{A.5}) rely on a single
partition shared by all candidates. In Full SC-CS each candidate has its own
partition, so swapping or coupling the test response $Y_{n+j}$ with calibration
responses alters the refitted partitions and stratum memberships, breaking both
the local hybrid PRDS structure and the within-stratum permutation symmetry. Example~\ref{app:full-fdr-counterexample}
shows that the conditional FDR bound fails even when all candidates share
the same fitted stratum.
Full SC-CS therefore serves as an empirical benchmark; its empirical error rates in the numerical studies should not be interpreted as theoretical guarantees.


\section{The Limits of Conditional and Pooled Error Control}
\label{app:selection-counterexamples}
The following examples clarify the scope of the selection guarantees.
We use least-squares isotonic calibration and the randomized localized
$p$-values in Equation~\eqref{prac_p}, with mutually independent
$U_j\sim\mathrm{Uniform}[0,1]$ independent of all observations.
The first two examples show that, for Full SC-CS, the conditional
PSER and conditional FDR given its fitted prediction value can exceed
the nominal level. The third shows that the stratum-conditional FDR
control established for ICS in Theorem~\ref{thm:ics-fdr} does not imply
FDR control for the pooled selection set.

\begin{example}[Failure of conditional PSER control for Full SC-CS]
\label{app:full-pser-counterexample}
Take one calibration observation and one test candidate ($n=m=1$), a
constant base predictor $f\equiv0$, and independent responses
$Y_1,Y_2\sim\mathrm{Uniform}[0,1]$, with $X_1=X_2=0$. Set the candidate threshold to
$c=3/4$. Since the calibration point and the augmented candidate share the same base-predictor value \(f(X_1)=f(X_2)=0\), the least-squares isotonic fit assigns them the common value
\[
L:=\hat f^{(c)}(X_2)=\hat f^{(c)}(X_1)=\frac{Y_1+c}{2}.
\]
Thus the local calibration set always contains observation~1. The event
$L>c$ is equivalent to $Y_1>c$ and has probability $1/4$.
For almost every fitted value $v\in(c,(1+c)/2)$, conditioning on $L=v$
implies $Y_1>c$, so Equation~\eqref{prac_p} gives $p_1=U_1/2$.
The true test response remains independent of $L$ and $U_1$. Hence, at
$\gamma=1/4$,
\[
\Pbb(p_1\le\gamma,\;Y_2\le c\mid L=v)
=\Pbb(U_1\le2\gamma)\Pbb(Y_2\le c)
=2\gamma c=\frac{3}{8}>\gamma.
\]
The same value $3/8$ is obtained when conditioning on the positive-probability
event $\{L>c\}$. The stratum-conditional PSER bound of Proposition~\ref{prop:pser} therefore does not extend to candidate-refitted partitions such as Full SC-CS. The failure is already present in this weaker, joint formulation; conditioning additionally on the candidate being a true null would only increase the left-hand side.
\end{example}

\begin{example}[Failure of conditional FDR control for Full SC-CS]
\label{app:full-fdr-counterexample}
Use the same construction with two test candidates ($n=1$, $m=2$), \(X_1=X_2=X_3=0,\)
independent responses $Y_1,Y_2,Y_3\sim\mathrm{Uniform}[0,1]$, and
$c_1=c_2=c=3/4$. Both candidate-specific augmented fits coincide, so
both candidates belong to the single stratum with fitted value
$L=(Y_1+c)/2$. Conditional on any $L=v>c$ in its support, their
$p$-values are $p_j=U_j/2$, $j=1,2$.

Apply BH within this stratum at $q=1/4$. Writing
$U_{(1)}=\min(U_1,U_2)$ and $U_{(2)}=\max(U_1,U_2)$, the event of at least
one rejection is
\(\{\mathcal R\ne\emptyset\}
=\{U_{(1)}\le q\}\cup\{U_{(2)}\le2q\}.\)
Its conditional probability is
\(\Pbb(\mathcal R\ne\emptyset\mid L=v)
=\bigl[1-(1-q)^2\bigr]+q^2=2q.\)
The second term corresponds to both uniforms lying in $(q,2q]$, which is
disjoint from the first event. The null indicators
$\mathbbm1\{Y_{n+j}\le c_j\}$ $(j=1,2)$ are independent Bernoulli$(c)$ variables,
independent of the rejection set under this conditioning. For every nonempty
rejection set, the conditional expected FDP is therefore $c$.
Consequently,
\[
\E[\mathrm{FDP}\mid L=v]
=c\,\Pbb(\mathcal R\ne\emptyset\mid L=v)
=2cq=\frac{3}{8}>q.
\]
This violation holds for almost every $v\in(c,(1+c)/2)$ and also after
conditioning on $\{L>c\}$. Thus the conditional FDR failure occurs even
when the candidates share exactly the same fitted stratum and their
randomized $p$-values are independent conditional on that fitted value. The mechanism is a selection effect: conditioning on $\{L = v\}$ reveals $Y_1 > c$, so the local calibration response no longer resembles a null test response, and the $p$-values $\Pbb(p_j \le t \mid L=v) = \min(2t, 1)$ are anti-conservative by a factor of two.

\end{example}

\begin{example}[Stratum FDR control does not imply pooled FDR control for ICS]
\label{app:ics-pooled-fdr-counterexample}

Although
\[
\mathrm{FDP}
=\sum_{v\in\mathcal V}
\frac{|\mathcal R_v|}{|\mathcal R|\vee1}\,\mathrm{FDP}_v,
\]
the weights depend on the random rejection sets. Bounds on the expected
FDP in each stratum therefore do not directly give the same bound for
the pooled FDR.

Let all calibration and test observations be i.i.d.\ with
\[
X\sim\mathrm{Uniform}\{0,1,2\},\qquad Y=X/4,
\qquad f(x)=x.
\]
Use a response-independent split with $n_1=12$ fitting observations,
$n_2=1$ score-computation observation, and $m=2$ test candidates.
Set $c_1=c_2=1$ and $q=1/4$. All candidates are true nulls.
Let $E$ be the event that all three feature values occur in
$\mathcal D_1$. On $E$, the isotonic least-squares fit satisfies
$\hat f(x)=x/4$ at each support point and creates three distinct strata.
The independent score-computation observation occupies exactly one of
these strata. Since its response is strictly below both thresholds,
a test candidate in the occupied stratum has
$p_j=(1+U_j)/2\ge1/2$ and cannot be rejected at this $q$.
A candidate in either empty stratum has $p_j=U_j$, exactly as specified
by Equation~\eqref{prac_p} when its local calibration set is empty.

Conditional on $E$, let $K$ denote the number of distinct empty strata
represented in the test pool. The score-computation and test features
are still independent and uniform on the three support points, giving
\[
\begin{array}{c|ccc}
k & 0 & 1 & 2\\ \hline
\Pbb(K=k\mid E) & 1/9 & 2/3 & 2/9\\
\Pbb(\mathcal R\ne\emptyset\mid E,K=k) & 0 & q & 2q-q^2
\end{array}
\]
For $K=1$, BH applied to one or two independent uniform $p$-values
rejects with probability $q$; in the two-candidate case,
$\Pbb(U_{(1)}\le q/2\text{ or }U_{(2)}\le q)=q$.
Indeed, the complementary event
$\{U_{(1)}>q/2,\ U_{(2)}>q\}$ has probability
$(1-q)^2+2(q-q/2)(1-q)=1-q$: either both uniforms exceed
$q$, or exactly one exceeds $q$ and the other lies in $(q/2,q]$.
For $K=2$, each empty stratum contains one candidate and the two
level-$q$ rejection events are independent. Since all candidates are null,
\[
\E[\mathrm{FDP}\mid E]
=\Pbb(\mathcal R\ne\emptyset\mid E)
=\frac{2}{3}q+\frac{2}{9}(2q-q^2)
=\frac{19}{72}>\frac14.
\]
This also yields an unconditional violation, rather than only a violation
on the event $E$. Inclusion--exclusion gives
\[
\Pbb(E)=1-3(2/3)^{12}+3(1/3)^{12}
=\frac{19228}{19683}.
\]
Since the FDP is nonnegative on $E^c$,
\[
\mathrm{FDR}
\ge\Pbb(E)\,\frac{19}{72}
=\frac{91333}{354294}
\approx0.2578>q.
\]
\end{example}

\section{Additional Data Experiments}
\subsection{Additional Results Discontinuous Data Setting}

\begin{table}[!htbp]
\centering
\caption{
Prediction interval performance in the discontinuous calibration setting.
Results are averaged across 10 independent replicates at nominal coverage
$1-\alpha=0.9$. The coverage gap is the absolute difference between coverage
for $X<0$ and $X\geq0$.
}
\label{tab:babii_kumar_main}
\small
\setlength{\tabcolsep}{5pt}
\renewcommand{\arraystretch}{1.08}

\begin{tabular}{lccccc}
\toprule
Method
& Overall Cov.
& Cov. $X<0$
& Cov. $X\geq0$
& Cov. Gap
& Avg. Width \\
\midrule

sICP
& 0.912
& 0.907
& 0.917
& 0.010
& 3.456 \\

tICP
& 0.911
& 0.909
& 0.913
& 0.004
& 3.485 \\

SC-CP
& 0.906
& 0.900
& 0.912
& 0.012
& 3.376 \\

Marginal CP
& 0.901
& 0.957
& 0.845
& 0.111
& 4.049 \\

CQR
& 0.904
& 0.902
& 0.905
& 0.004
& 3.567 \\

Mondrian CP
& 0.909
& 0.908
& 0.910
& 0.002
& 4.045 \\

\bottomrule
\end{tabular}
\end{table}

\begin{table}[!htbp]
\centering
\caption{
Local prediction interval performance near the discontinuity at $X=0$.
Results are averaged across 10 independent replicates at nominal coverage
$1-\alpha=0.9$.
}
\label{tab:babii_kumar_local}
\small
\setlength{\tabcolsep}{5pt}
\renewcommand{\arraystretch}{1.08}

\begin{tabular}{lcccc}
\toprule
Method
& \multicolumn{2}{c}{Coverage}
& \multicolumn{2}{c}{Average Width} \\
\cmidrule(lr){2-3}
\cmidrule(lr){4-5}
& $-0.1\leq X<0$
& $0\leq X\leq0.1$
& $-0.1\leq X<0$
& $0\leq X\leq0.1$ \\
\midrule

sICP
& 0.912
& 0.920
& 3.449
& 3.613 \\

tICP
& 0.920
& 0.909
& 3.610
& 3.531 \\

SC-CP
& 0.908
& 0.907
& 3.383
& 3.426 \\

Marginal CP
& 0.961
& 0.848
& 4.049
& 4.049 \\

CQR
& 0.905
& 0.901
& 3.544
& 3.627 \\

Mondrian CP
& 0.930
& 0.904
& 3.592
& 4.600 \\

\bottomrule
\end{tabular}
\end{table}

\begin{table}[!htbp]
\centering
\caption{
Average prediction interval width on either side of the discontinuity.
Results are averaged across 10 independent replicates.
}
\label{tab:babii_kumar_width}
\small
\setlength{\tabcolsep}{6pt}
\renewcommand{\arraystretch}{1.08}

\begin{tabular}{lccc}
\toprule
Method
& Width $X<0$
& Width $X\geq0$
& Absolute Difference \\
\midrule

sICP
& 3.405
& 3.507
& 0.102 \\

tICP
& 3.462
& 3.509
& 0.048 \\

SC-CP
& 3.328
& 3.425
& 0.096 \\

Marginal CP
& 4.049
& 4.049
& 0.000 \\

CQR
& 3.545
& 3.589
& 0.044 \\

Mondrian CP
& 3.392
& 4.698
& 1.306 \\

\bottomrule
\end{tabular}
\end{table}
\subsection{Medical Expenditure Panel Survey}
We evaluate the ICP methods on the Medical Expenditure Panel Survey (MEPS) dataset
\citep{cohenMedicalExpenditurePanel2009}, following the experimental setup of
\citet{laanSelfCalibratingConformalPrediction2024a}. The response $Y$ is health
care utilization, and prediction performance is evaluated on the transformed
outcome $\log(1+Y)$. The covariates consist of standard sociodemographic and
clinical variables. We define a binary group variable $A$ from race, with
$A=1$ denoting individuals coded as White and $A=0$ all others, and report
coverage, interval width, and calibration error separately within each group.

The data are randomly partitioned into training, calibration, and test sets.
We consider two
base-model training regimes. In Setting~1 (poorly calibrated), the base model
is trained to predict the raw outcome $Y$, after which its predictions are
transformed via $\log(1+\cdot)$. Since $\log(1+y)$ is concave, Jensen's
inequality gives
    $\mathbb{E}\!\left[\log(1+Y)\mid X\right]
    \leq
    \log\!\left(1+\mathbb{E}[Y\mid X]\right),$
so transforming an accurate predictor of $\mathbb{E}[Y\mid X]$ generally
produces an upward-biased predictor of the conditional mean of
$\log(1+Y)$. Setting~1 therefore deliberately induces miscalibration of the
base predictor. In Setting~2 (well calibrated), the model is instead trained
directly on $\log(1+Y)$, removing this transformation-induced source of
miscalibration.

\begin{table}[!htbp]
\centering
\caption{Prediction interval performance on MEPS-21 under
Setting~1 (poorly calibrated base predictor) and Setting~2
(well-calibrated base predictor), by group $A$.
Nominal coverage is $1-\alpha=0.9$.
sICP and tICP denote split and transductive isotonic
conformal prediction, respectively.}
\label{tab:meps21_results}

\small
\setlength{\tabcolsep}{5pt}
\renewcommand{\arraystretch}{1.0}

\begin{tabular}{@{}lrrrrrr@{}}
\toprule
& \multicolumn{2}{c}{Coverage}
& \multicolumn{2}{c}{Avg. width}
& \multicolumn{2}{c}{Cal. metric} \\
\cmidrule(lr){2-3}
\cmidrule(lr){4-5}
\cmidrule(lr){6-7}
Method
& $A=0$ & $A=1$
& $A=0$ & $A=1$
& $A=0$ & $A=1$ \\
\midrule

\multicolumn{7}{c}{\textbf{Setting 1: Poorly calibrated base predictor}} \\
\midrule
sICP
& 0.927 & 0.926 & 2.62 & 3.22 & 0.0061 & -0.0169 \\
tICP
& 0.895 & 0.912 & 2.33 & 3.03 & 0.0165 & -0.0249 \\
Closed-form tICP
& 0.921 & 0.920 & 2.34 & 3.04 & 0.0165 & -0.0249 \\
Marginal CP
& 0.920 & 0.873 & 3.34 & 3.34 & 0.4720 & 0.4824 \\
CQR
& 0.924 & 0.897 & 2.36 & 2.90 & 0.1050 & -0.1407 \\
Mondrian CP
& 0.904 & 0.898 & 2.93 & 3.59 & 0.4720 & 0.4824 \\
SC-CP
& 0.914 & 0.918 & 2.20 & 2.86 & 0.0117 & -0.0265 \\
\midrule

\multicolumn{7}{c}{\textbf{Setting 2: Well-calibrated base predictor}} \\
\midrule
sICP
& 0.918 & 0.927 & 2.62 & 3.32 & -0.0174 & -0.0017 \\
tICP
& 0.905 & 0.912 & 2.28 & 3.00 & -0.0045 & -0.0094 \\
Closed-form tICP
& 0.919 & 0.920 & 2.29 & 3.01 & -0.0045 & -0.0094 \\
Marginal CP
& 0.919 & 0.871 & 2.82 & 2.82 & -0.0088 & -0.0626 \\
CQR
& 0.924 & 0.897 & 2.36 & 2.90 & 0.1050 & -0.1407 \\
Mondrian CP
& 0.890 & 0.891 & 2.33 & 2.96 & -0.0088 & -0.0626 \\
SC-CP
& 0.906 & 0.906 & 2.12 & 2.75 & -0.0131 & -0.0157 \\
\bottomrule
\end{tabular}
\end{table}

Table~\ref{tab:meps21_results} shows that Marginal CP attains approximately
nominal coverage when averaged over the population but exhibits a pronounced
coverage imbalance across the two groups.
The prediction-adaptive procedures substantially reduce this groupwise
imbalance. In Setting~1, sICP attains coverage $0.927$ and $0.926$ for
$A=0$ and $A=1$, respectively, while closed-form tICP attains $0.921$ and
$0.920$. SC-CP similarly yields $0.914$ and $0.918$. These procedures thus produce
nearly equal coverage across the two groups. A similar pattern is observed in
Setting~2, where closed-form tICP attains coverage $0.919$ and $0.920$ and
SC-CP attains $0.906$ in both groups.

The distinction between the two training regimes is particularly apparent in
the calibration error. In Setting~1, Marginal CP inherits the substantial
miscalibration induced by transforming a model trained on the raw response,
with reported calibration errors of approximately $0.47$--$0.48$ in the two
groups. sICP, tICP, and SC-CP reduce these values substantially, producing
calibration errors much closer to zero. In Setting~2, where the base model is
trained directly on the transformed outcome, the calibration errors associated
with the uncalibrated predictor are already considerably smaller. These results
also emphasize the distinction between marginal conformal coverage and
calibration of the underlying point prediction: approximately correct marginal
coverage does not by itself imply that the point predictor is well calibrated.

Prediction calibration is also associated with substantial gains in interval
efficiency. SC-CP produces the narrowest intervals among the methods considered
in both groups and both settings.  Closed-form tICP also achieves substantial reductions, with
average widths $2.34$ and $3.04$, while maintaining nearly identical coverage
across the two groups. sICP is somewhat wider, particularly for $A=1$, but
likewise yields highly balanced groupwise coverage.

In Setting~2 the efficiency gains remain visible, although the base predictor
requires much less calibration. SC-CP again gives the narrowest intervals,
with widths $2.12$ and $2.75$ for $A=0$ and $A=1$, respectively. Closed-form
tICP produces widths $2.29$ and $3.01$, also improving upon Marginal CP while
maintaining balanced coverage. The similarity of the results across the two
settings suggests that the isotonic procedures do not require severe initial
miscalibration in order to remain competitive in interval width.

The direct grid-based and closed-form implementations of tICP use the same
underlying transductive calibration construction and consequently have
essentially identical calibration error. Their small differences in empirical
coverage and interval width arise from the numerical inversion used by the
grid-based implementation, whereas the closed-form procedure computes the
prediction set without discretizing the candidate-label space.

\begin{figure}[t]
    \centering
    \includegraphics[width=0.95\textwidth]{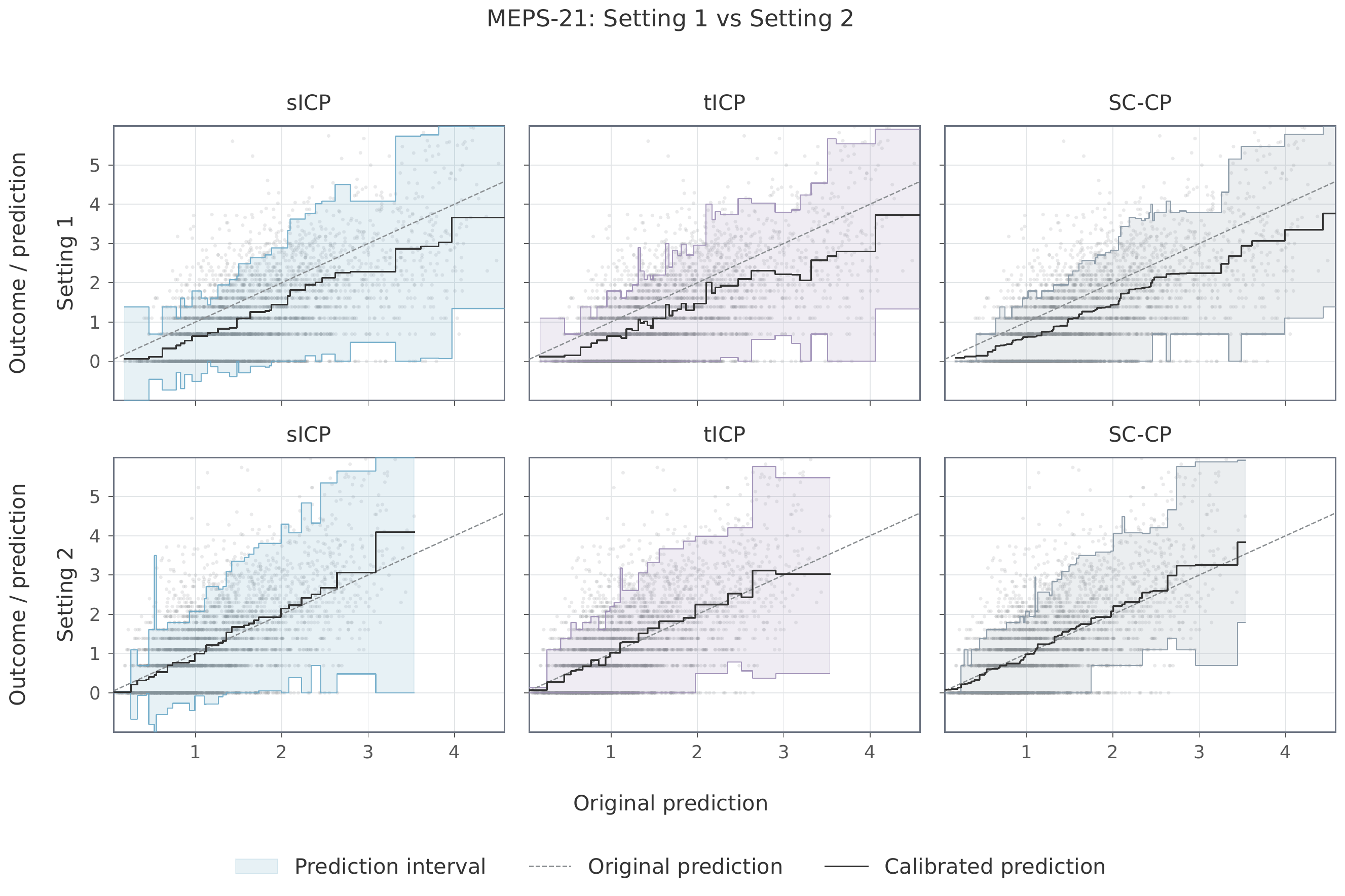}
    \caption{
    Prediction intervals and calibrated predictions for MEPS-21 under
    Setting~1 and Setting~2. The top row corresponds to Setting~1
    (poorly calibrated base predictor), while the bottom row corresponds to
    Setting~2 (well-calibrated base predictor). Columns show sICP,
    closed-form tICP, and SC-CP. The dashed diagonal denotes the original
    prediction, the solid black step function denotes the calibrated point
    prediction, and the shaded region denotes the corresponding conformal
    prediction interval.
    }
    \label{fig:meps_prediction_intervals}
\end{figure}

Figure~\ref{fig:meps_prediction_intervals} provides a direct visualization of
the calibration corrections and the resulting prediction intervals. The
horizontal axis is the original uncalibrated prediction $f(x)$, and the
vertical axis is the transformed outcome $\log(1+Y)$. The dashed diagonal is
the identity map, so departures of the solid black calibrated prediction from
this line indicate the correction applied to the original model.

In Setting~1 (top row), the calibrated predictions lie predominantly below the
identity line, consistent with the substantial calibration error induced by
the transformed base predictor. The correction becomes more pronounced over
portions of the prediction range where the original model overpredicts the
transformed outcome. In Setting~2 (bottom row), the calibrated predictions
remain substantially closer to the identity line, consistent with the much
smaller transformation-induced miscalibration in this setting.

The shaded regions show the corresponding $90\%$ conformal prediction
intervals. Their widths vary across the range of the fitted prediction rather
than remaining globally constant, illustrating the prediction-dependent
adaptation of sICP, tICP, and SC-CP. The three methods produce broadly similar
calibration corrections but differ in the widths of the resulting prediction
sets, with SC-CP generally producing the narrowest intervals and closed-form
tICP closely matching the direct transductive construction.

\subsection{Bike Dataset}
In the following  sections, we present experimental results for the Concrete, STAR, Bike, and Bio datasets used in \citet{romano2019conformalized}. In the STAR dataset, the sensitive attribute 
A was \texttt{gender}; in the Bike dataset, it was \texttt{workingday}; and in the remaining datasets, it indicated whether the last feature was above or below its median value. We also present experimental results on an ADMET and Molecular activity dataset.
\begin{table}[H]
\centering
\caption{Prediction interval performance on the Bike dataset under Setting~1 (poorly calibrated base predictor). Coverage is evaluated at nominal level $1-\alpha=0.9$.}
\label{tab:bike_s1}

\begin{tabular}{lcccccc}
\toprule
\textbf{Method}
& \multicolumn{2}{c}{\textbf{Coverage}}
& \multicolumn{2}{c}{\textbf{Avg Width}}
& \multicolumn{2}{c}{\textbf{Cal Error}} \\
\cmidrule(lr){2-3}
\cmidrule(lr){4-5}
\cmidrule(lr){6-7}
& $A=0$ & $A=1$
& $A=0$ & $A=1$
& $A=0$ & $A=1$ \\
\midrule

sICP
& 0.878 & 0.935
& 1.23 & 1.16
& 0.0392 & -0.0129 \\

tICP
& 0.858 & 0.909
& 1.18 & 1.12
& 0.0524 & -0.0133 \\

Closed-form tICP
& 0.883 & 0.928
& 1.20 & 1.14
& 0.0524 & -0.0133 \\

Marginal CP
& 0.865 & 0.925
& 1.17 & 1.17
& 0.1476 & 0.0697 \\

CQR
& 0.915 & 0.912
& 1.37 & 1.21
& 0.0467 & 0.0216 \\

Mondrian CP
& 0.869 & 0.925
& 1.15 & 1.08
& 0.1476 & 0.0697 \\

SC-CP
& 0.884 & 0.931
& 1.21 & 1.08
& 0.0343 & -0.0214 \\

\bottomrule
\end{tabular}
\end{table}

\begin{table}[H]
\centering
\caption{Prediction interval performance on the Bike dataset under Setting~2 (well-calibrated base predictor). Coverage is evaluated at nominal level $1-\alpha=0.9$.}
\label{tab:bike_s2}

\begin{tabular}{lcccccc}
\toprule
\textbf{Method}
& \multicolumn{2}{c}{\textbf{Coverage}}
& \multicolumn{2}{c}{\textbf{Avg Width}}
& \multicolumn{2}{c}{\textbf{Cal Error}} \\
\cmidrule(lr){2-3}
\cmidrule(lr){4-5}
\cmidrule(lr){6-7}
& $A=0$ & $A=1$
& $A=0$ & $A=1$
& $A=0$ & $A=1$ \\
\midrule

sICP
& 0.906 & 0.934
& 1.37 & 1.25
& 0.0335 & -0.0033 \\

tICP
& 0.891 & 0.915
& 1.19 & 1.13
& 0.0297 & -0.0059 \\

Closed-form tICP
& 0.906 & 0.926
& 1.20 & 1.14
& 0.0297 & -0.0059 \\

Marginal CP
& 0.891 & 0.920
& 1.12 & 1.12
& 0.0475 & -0.0007 \\

CQR
& 0.915 & 0.912
& 1.37 & 1.21
& 0.0467 & 0.0216 \\

Mondrian CP
& 0.885 & 0.920
& 1.11 & 1.07
& 0.0475 & -0.0007 \\

SC-CP
& 0.900 & 0.916
& 1.19 & 1.10
& 0.0288 & -0.0111 \\

\bottomrule
\end{tabular}
\end{table}

 \begin{figure}[H]
    \centering
   \includegraphics[width=0.95\textwidth]{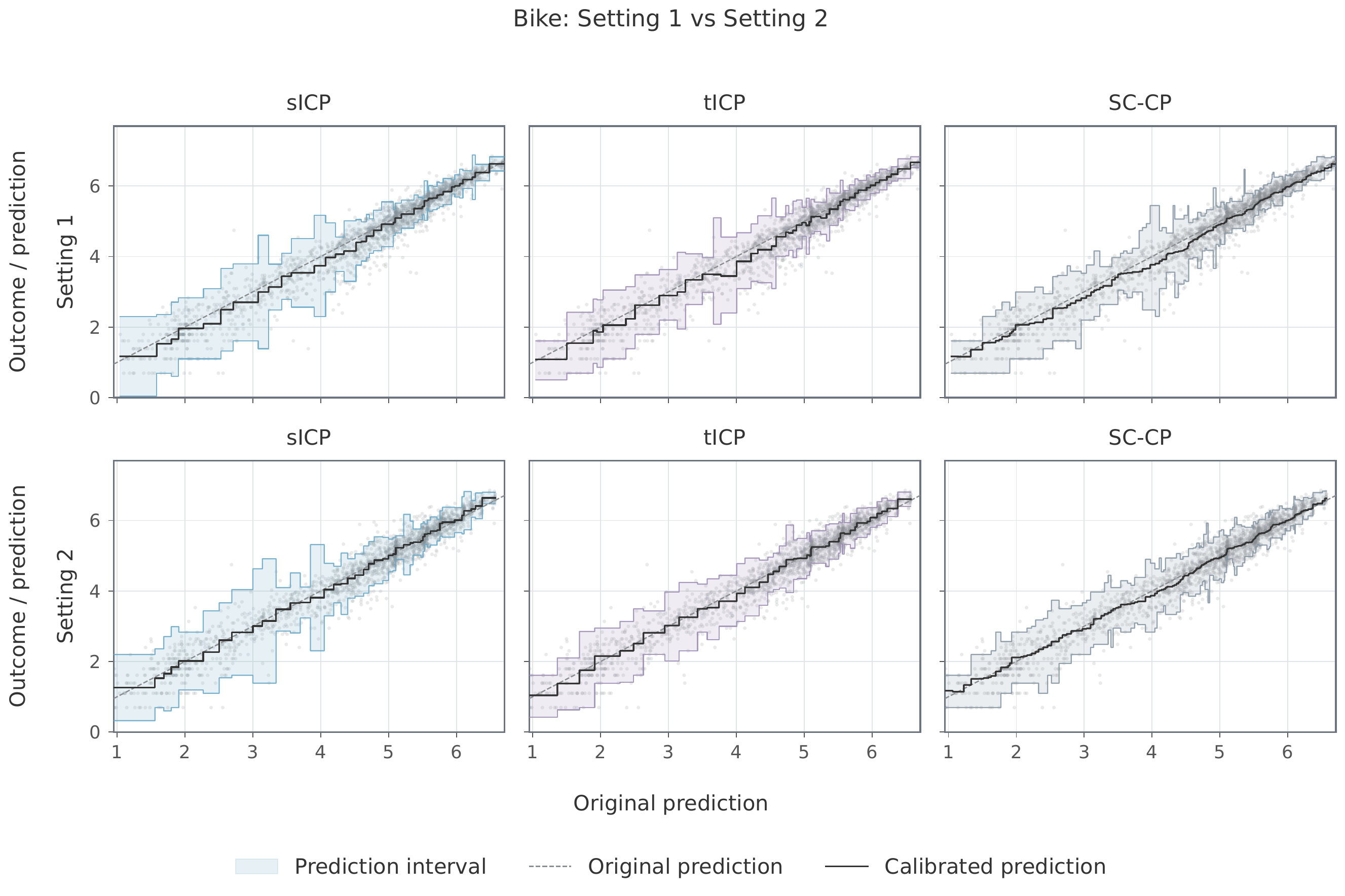}

    \caption{
    Prediction intervals and calibrated predictions for the Bike dataset under two calibration settings.
    The dashed gray line denotes the identity map, corresponding to no calibration.
    The black step function denotes the calibrated prediction, the red points denote pointwise calibrated predictions, and the shaded region denotes the conformal prediction interval.
    }
    \label{fig:bike_calibration_plots}
\end{figure}
 
\subsection{Concrete Dataset}
 \begin{table}[H]
\centering
\caption{Prediction interval performance on the Concrete dataset under Setting~1 (poorly calibrated base predictor). Coverage is evaluated at nominal level $1-\alpha=0.9$.}
\label{tab:concrete_s1}

\begin{tabular}{lcccccc}
\toprule
\textbf{Method}
& \multicolumn{2}{c}{\textbf{Coverage}}
& \multicolumn{2}{c}{\textbf{Avg Width}}
& \multicolumn{2}{c}{\textbf{Cal Error}} \\
\cmidrule(lr){2-3}
\cmidrule(lr){4-5}
\cmidrule(lr){6-7}
& $A=0$ & $A=1$
& $A=0$ & $A=1$
& $A=0$ & $A=1$ \\
\midrule

sICP
& 0.865 & 0.983
& 0.83 & 0.77
& -0.0802 & -0.0005 \\

tICP
& 0.845 & 0.966
& 0.75 & 0.58
& -0.0657 & 0.0107 \\

Closed-form tICP
& 0.885 & 0.966
& 0.76 & 0.58
& -0.0657 & 0.0107 \\

Marginal CP
& 0.845 & 1.000
& 0.56 & 0.56
& 0.0053 & -0.0003 \\

CQR
& 0.858 & 0.931
& 0.57 & 0.46
& -0.0047 & -0.0508 \\

Mondrian CP
& 0.851 & 0.931
& 0.71 & 0.39
& 0.0053 & -0.0003 \\

SC-CP
& 0.818 & 0.897
& 0.60 & 0.44
& -0.0644 & -0.0326 \\

\bottomrule
\end{tabular}
\end{table}

\begin{table}[H]
\centering
\caption{Prediction interval performance on the Concrete dataset under Setting~2 (well-calibrated base predictor). Coverage is evaluated at nominal level $1-\alpha=0.9$.}
\label{tab:concrete_s2}

\begin{tabular}{lcccccc}
\toprule
\textbf{Method}
& \multicolumn{2}{c}{\textbf{Coverage}}
& \multicolumn{2}{c}{\textbf{Avg Width}}
& \multicolumn{2}{c}{\textbf{Cal Error}} \\
\cmidrule(lr){2-3}
\cmidrule(lr){4-5}
\cmidrule(lr){6-7}
& $A=0$ & $A=1$
& $A=0$ & $A=1$
& $A=0$ & $A=1$ \\
\midrule

sICP
& 0.926 & 0.948
& 0.94 & 0.72
& -0.0487 & -0.0031 \\

tICP
& 0.892 & 0.931
& 0.91 & 0.63
& -0.0465 & 0.0006 \\

Closed-form tICP
& 0.905 & 0.948
& 0.92 & 0.63
& -0.0465 & 0.0006 \\

Marginal CP
& 0.797 & 1.000
& 0.53 & 0.53
& -0.0267 & -0.0204 \\

CQR
& 0.858 & 0.931
& 0.57 & 0.46
& -0.0047 & -0.0508 \\

Mondrian CP
& 0.858 & 0.914
& 0.60 & 0.41
& -0.0267 & -0.0204 \\

SC-CP
& 0.872 & 0.948
& 0.72 & 0.54
& -0.0447 & -0.0349 \\

\bottomrule
\end{tabular}
\end{table}

\begin{figure}[H]
    \centering
    \includegraphics[width=0.95\textwidth]{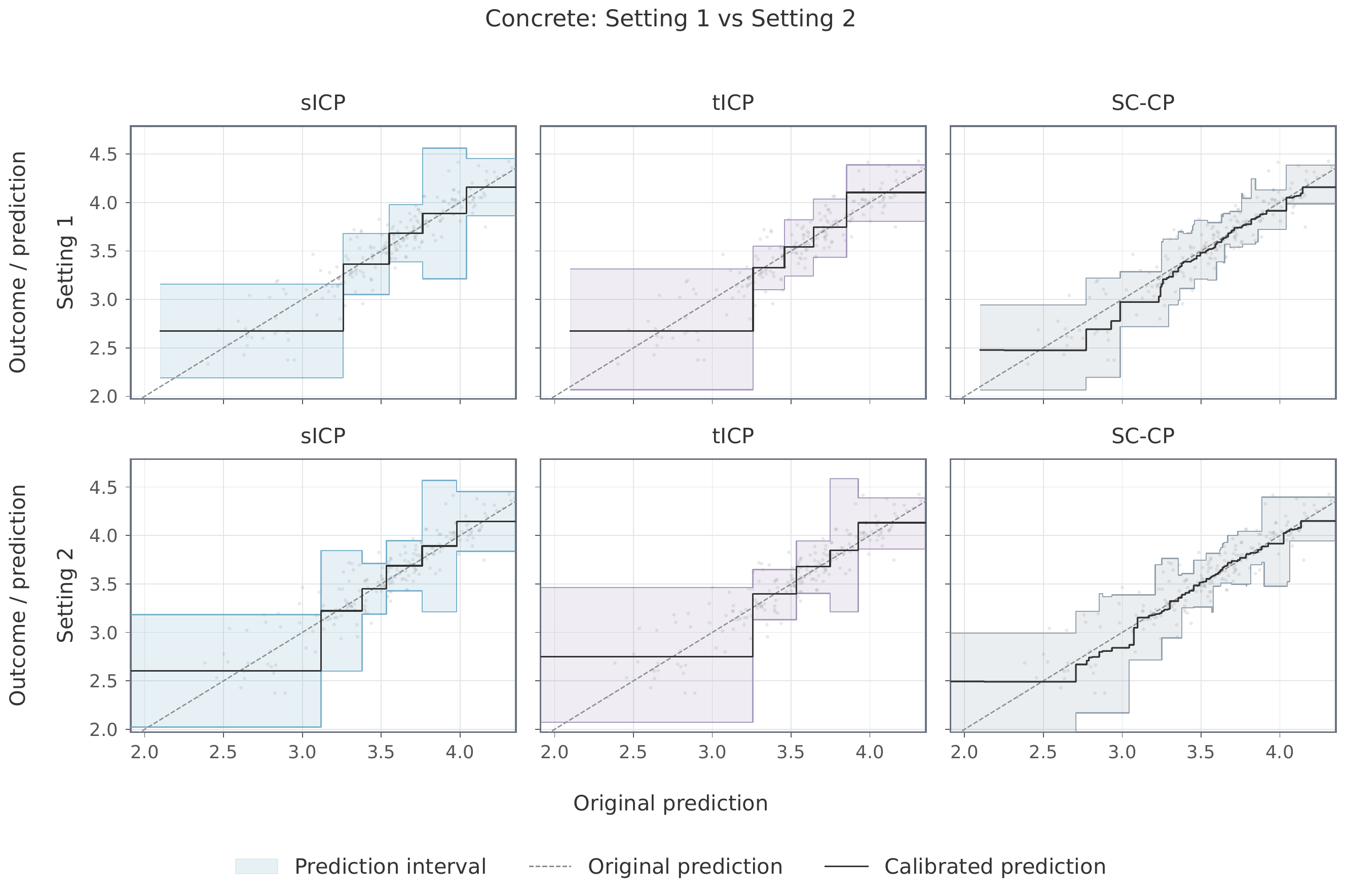}

    \caption{
    Prediction intervals and calibrated predictions for the Concrete dataset under two calibration settings.
    The dashed gray line denotes the identity map, corresponding to no calibration.
    The black step function denotes the calibrated prediction, the red points denote pointwise calibrated predictions, and the shaded region denotes the conformal prediction interval.
    }
    \label{fig:concrete_calibration_plots}
\end{figure}
 
\subsection{Bio Dataset}
 
\begin{table}[H]
\centering
\caption{Prediction interval performance on the Bio dataset under Setting~1 (poorly calibrated base predictor). Coverage is evaluated at nominal level $1-\alpha=0.9$.}
\label{tab:bio_s1}

\begin{tabular}{lcccccc}
\toprule
\textbf{Method}
& \multicolumn{2}{c}{\textbf{Coverage}}
& \multicolumn{2}{c}{\textbf{Avg Width}}
& \multicolumn{2}{c}{\textbf{Cal Error}} \\
\cmidrule(lr){2-3}
\cmidrule(lr){4-5}
\cmidrule(lr){6-7}
& $A=0$ & $A=1$
& $A=0$ & $A=1$
& $A=0$ & $A=1$ \\
\midrule

sICP
& 0.899 & 0.921
& 1.21 & 1.42
& -0.0132 & -0.0030 \\

tICP
& 0.893 & 0.917
& 1.20 & 1.40
& -0.0090 & 0.0016 \\

Closed-form tICP
& 0.896 & 0.919
& 1.21 & 1.41
& -0.0090 & 0.0016 \\

Marginal CP
& 0.900 & 0.902
& 1.55 & 1.55
& 0.1213 & 0.1102 \\

CQR
& 0.896 & 0.903
& 1.69 & 1.76
& 0.1108 & -0.0682 \\

Mondrian CP
& 0.894 & 0.909
& 1.24 & 1.40
& 0.1213 & 0.1102 \\

SC-CP
& 0.888 & 0.913
& 1.13 & 1.32
& -0.0098 & -0.0011 \\

\bottomrule
\end{tabular}
\end{table}

\begin{table}[H]
\centering
\caption{Prediction interval performance on the Bio dataset under Setting~2 (well-calibrated base predictor). Coverage is evaluated at nominal level $1-\alpha=0.9$.}
\label{tab:bio_s2}

\begin{tabular}{lcccccc}
\toprule
\textbf{Method}
& \multicolumn{2}{c}{\textbf{Coverage}}
& \multicolumn{2}{c}{\textbf{Avg Width}}
& \multicolumn{2}{c}{\textbf{Cal Error}} \\
\cmidrule(lr){2-3}
\cmidrule(lr){4-5}
\cmidrule(lr){6-7}
& $A=0$ & $A=1$
& $A=0$ & $A=1$
& $A=0$ & $A=1$ \\
\midrule

sICP
& 0.896 & 0.921
& 1.21 & 1.33
& -0.0199 & 0.0059 \\

tICP
& 0.891 & 0.917
& 1.15 & 1.29
& -0.0142 & 0.0097 \\

Closed-form tICP
& 0.894 & 0.918
& 1.16 & 1.30
& -0.0142 & 0.0097 \\

Marginal CP
& 0.895 & 0.900
& 1.45 & 1.45
& -0.0077 & -0.0220 \\

CQR
& 0.896 & 0.903
& 1.69 & 1.76
& 0.1108 & -0.0682 \\

Mondrian CP
& 0.886 & 0.916
& 1.17 & 1.29
& -0.0077 & -0.0220 \\

SC-CP
& 0.891 & 0.915
& 1.13 & 1.25
& -0.0144 & 0.0076 \\

\bottomrule
\end{tabular}
\end{table}

 \begin{figure}[H]
    \centering
    \includegraphics[width=0.95\textwidth]{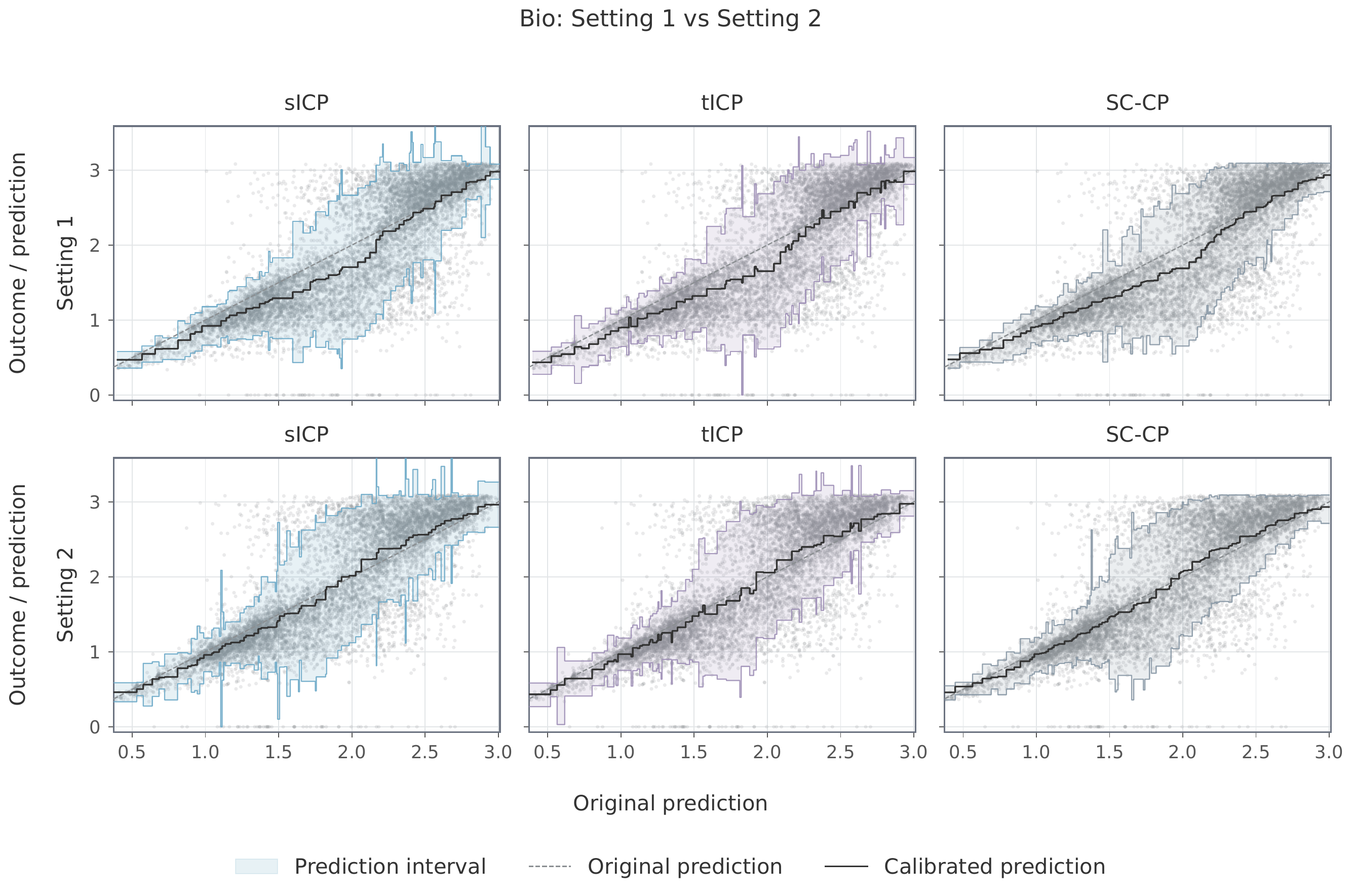}

    \caption{
    Prediction intervals and calibrated predictions for the Bio dataset under two calibration settings.
    The dashed gray line denotes the identity map, corresponding to no calibration.
    The black step function denotes the calibrated prediction, the red points denote pointwise calibrated predictions, and the shaded region denotes the conformal prediction interval.
    }
    \label{fig:bio_calibration_plots}
\end{figure}
\subsection{STAR Dataset}
\begin{table}[H]
\centering
\caption{Prediction interval performance on the STAR dataset under Setting~1 (poorly calibrated base predictor). Coverage is evaluated at nominal level $1-\alpha=0.9$.}
\label{star_s1}
\begin{tabular}{lcccccc}
\toprule
\textbf{Method}
& \multicolumn{2}{c}{\textbf{Coverage}}
& \multicolumn{2}{c}{\textbf{Avg Width}}
& \multicolumn{2}{c}{\textbf{Cal Error}} \\
\cmidrule(lr){2-3}
\cmidrule(lr){4-5}
\cmidrule(lr){6-7}
& $A=0$ & $A=1$
& $A=0$ & $A=1$
& $A=0$ & $A=1$ \\
\midrule

sICP
& 0.910 & 0.919
& 0.19 & 0.19
& 0.0038 & 0.0080 \\

tICP
& 0.878 & 0.877
& 0.17 & 0.16
& 0.0019 & 0.0058 \\

Closed-form tICP
& 0.896 & 0.910
& 0.19 & 0.18
& 0.0019 & 0.0058 \\

Marginal CP
& 0.887 & 0.896
& 0.17 & 0.17
& 0.0052 & 0.0103 \\

CQR
& 0.919 & 0.886
& 0.18 & 0.18
& 0.0030 & 0.0103 \\

Mondrian CP
& 0.892 & 0.919
& 0.18 & 0.18
& 0.0052 & 0.0103 \\

SC-CP
& 0.901 & 0.915
& 0.18 & 0.18
& 0.0064 & 0.0108 \\

\bottomrule
\end{tabular}
\end{table}

\begin{table}[H]
\centering
\caption{Prediction interval performance on the STAR dataset under Setting~2 (well-calibrated base predictor). Coverage is evaluated at nominal level $1-\alpha=0.9$.}
\label{star_s2}
\begin{tabular}{lcccccc}
\toprule
\textbf{Method}
& \multicolumn{2}{c}{\textbf{Coverage}}
& \multicolumn{2}{c}{\textbf{Avg Width}}
& \multicolumn{2}{c}{\textbf{Cal Error}} \\
\cmidrule(lr){2-3}
\cmidrule(lr){4-5}
\cmidrule(lr){6-7}
& $A=0$ & $A=1$
& $A=0$ & $A=1$
& $A=0$ & $A=1$ \\
\midrule

sICP
& 0.923 & 0.929
& 0.19 & 0.19
& 0.0030 & 0.0075 \\

tICP
& 0.887 & 0.886
& 0.17 & 0.17
& 0.0023 & 0.0044 \\

Closed-form tICP
& 0.919 & 0.905
& 0.19 & 0.19
& 0.0023 & 0.0044 \\

Marginal CP
& 0.896 & 0.896
& 0.17 & 0.17
& 0.0040 & 0.0088 \\

CQR
& 0.919 & 0.886
& 0.18 & 0.18
& 0.0030 & 0.0103 \\

Mondrian CP
& 0.896 & 0.929
& 0.18 & 0.18
& 0.0040 & 0.0088 \\

SC-CP
& 0.896 & 0.910
& 0.18 & 0.17
& 0.0065 & 0.0097 \\

\bottomrule
\end{tabular}
\end{table}

\begin{figure}[H]
    \centering
    \includegraphics[width=0.95\textwidth]{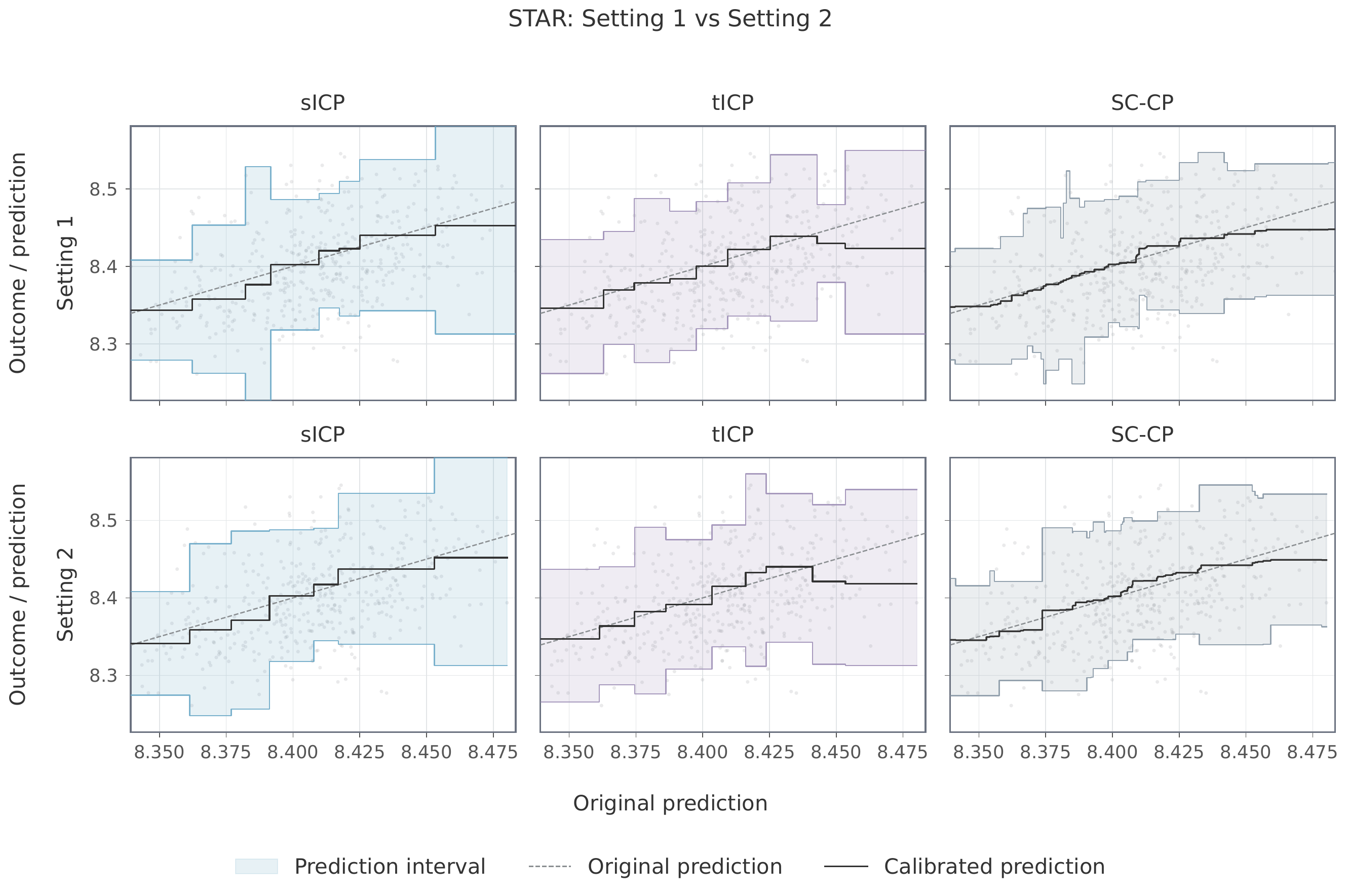}

    \caption{
    Prediction intervals and calibrated predictions for the STAR dataset under two calibration settings.
    The dashed gray line denotes the identity map, corresponding to no calibration.
    The black step function denotes the calibrated prediction, the red points denote pointwise calibrated predictions, and the shaded region denotes the conformal prediction interval.
    }
    \label{fig:star_calibration_plots}
\end{figure}

\subsection{ADMET Dataset}

\begin{table}[H]
\centering
\caption{Mean empirical coverage on ADMET across 10 random splits,
stratified by quartile of the base RF prediction. Quartile cutoffs are
computed from calibration predictions in each split. Nominal coverage is $0.9$.}
\label{tab:admet_quartile_coverage}
\small
\setlength{\tabcolsep}{5pt}
\begin{tabular}{@{}lccccc@{}}
\toprule
Method & Q1 & Q2 & Q3 & Q4 & Overall \\
\midrule
sICP             & 0.942 & 0.917 & 0.922 & 0.917 & 0.925 \\
tICP (grid)      & 0.903 & 0.883 & 0.912 & 0.926 & 0.906 \\
Closed-form tICP & 0.930 & 0.907 & 0.928 & 0.929 & 0.924 \\
\midrule
Marginal CP      & 0.934 & 0.871 & 0.880 & 0.928 & 0.903 \\
CQR              & 0.937 & 0.899 & 0.891 & 0.892 & 0.905 \\
Mondrian CP      & 0.915 & 0.895 & 0.913 & 0.904 & 0.907 \\
SC-CP            & 0.931 & 0.911 & 0.915 & 0.910 & 0.917 \\
\bottomrule
\end{tabular}
\end{table}
\begin{table}[H]
\centering
\caption{Average prediction interval width on ADMET for the same splits
and base-prediction quartiles as Table~\ref{tab:admet_quartile_coverage}.
Entries are means across 10 random splits.}
\label{tab:admet_quartile_width}
\small
\begin{tabular}{@{}lccccc@{}}
\toprule
Method & Q1 & Q2 & Q3 & Q4 & Overall \\
\midrule
sICP & 1.650 & 1.730 & 1.837 & 1.487 & 1.681 \\
tICP (grid) & 1.571 & 1.684 & 1.858 & 1.564 & 1.674 \\
Closed-form tICP & 1.583 & 1.695 & 1.868 & 1.574 & 1.686 \\
\midrule
Marginal CP & 1.584 & 1.584 & 1.584 & 1.584 & 1.584 \\
CQR & 1.767 & 1.678 & 1.671 & 1.781 & 1.724 \\
Mondrian CP & 1.458 & 1.662 & 1.766 & 1.482 & 1.595 \\
SC-CP & 1.443 & 1.638 & 1.722 & 1.430 & 1.561 \\
\bottomrule
\end{tabular}
\end{table}

\subsection{Molecular Activity Data}

\begin{figure}[H]
    \centering
    \includegraphics[width=\textwidth]{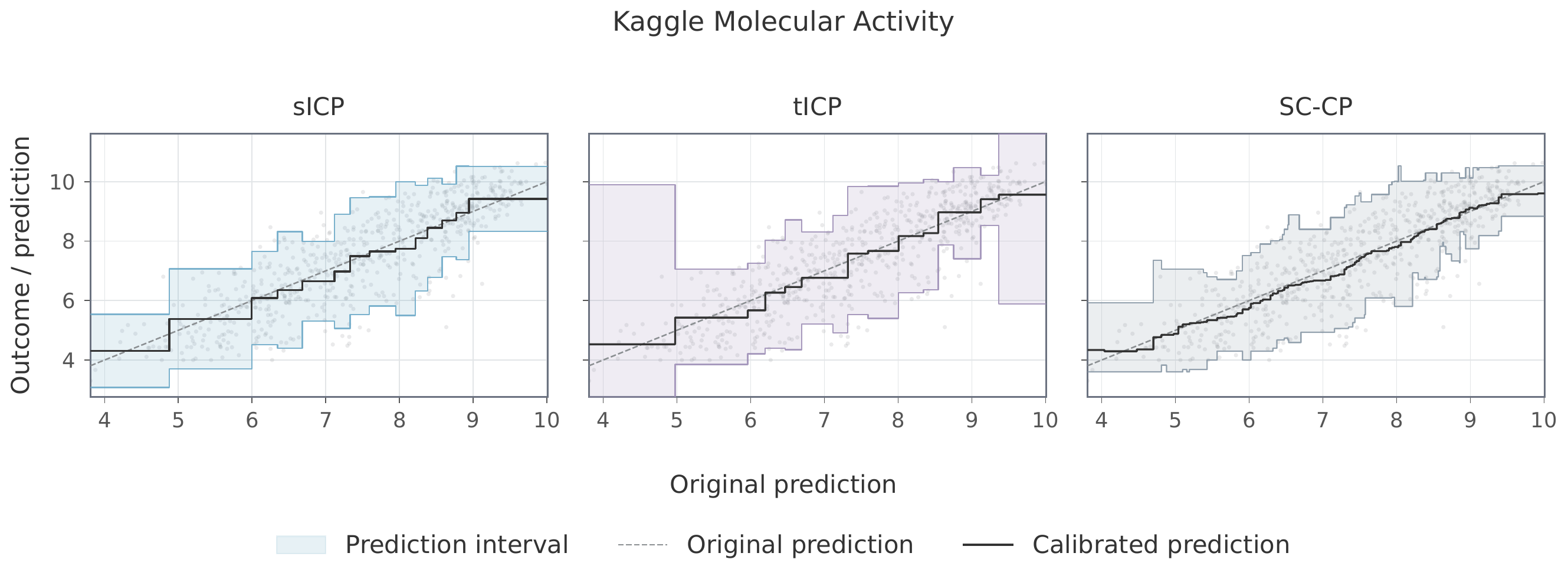}
    \caption{
        Prediction bands for sICP, closed-form tICP, and SC-CP on the Kaggle molecular
        activity dataset. Each panel shows the observed outcomes against the original
        model prediction, together with the calibrated prediction and the corresponding
        conformal prediction interval.
    }
    \label{fig:kaggle_prediction_bands}
\end{figure}

\begin{figure}[H]
    \centering
    \includegraphics[width=0.95\textwidth]{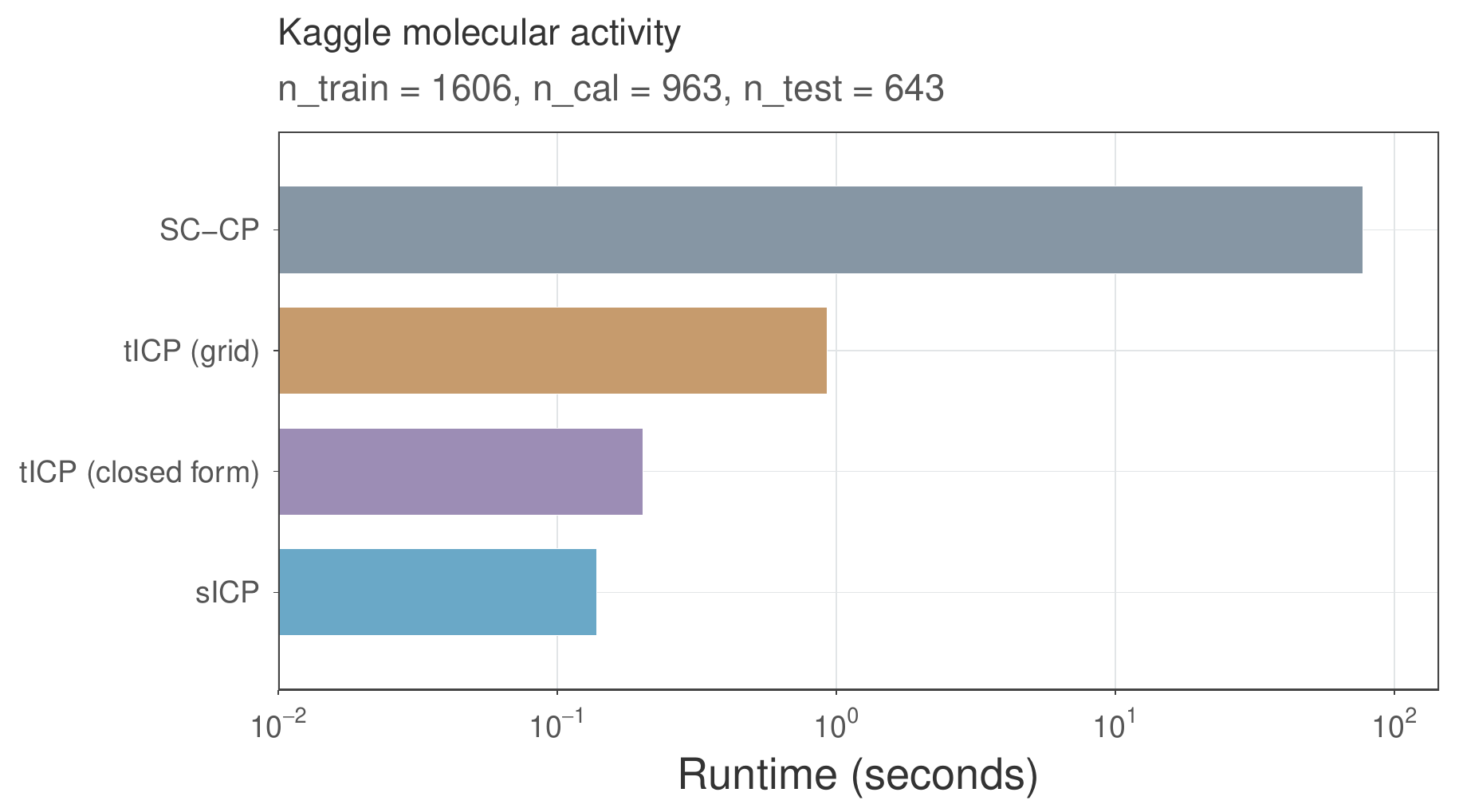}
    \caption{
        Runtime comparison for sICP, grid-based tICP, closed-form tICP, and SC-CP
        on the Kaggle molecular activity dataset. Bars report total method runtime,
        including calibration and prediction interval construction. The x-axis is shown on a logarithmic scale.
    }
    \label{fig:kaggle_runtime}
\end{figure}

\subsection{Additional ADMET Marginal Results}
\label{app:additional-admet}
Table~\ref{tab:real-data-residual-additional} reports both nominal levels
for the ten assays omitted from the main figure. Together
with Figure~\ref{fig:real-data}, it covers all 15 assays and six methods.
\begin{table*}[htbp]
\centering
\caption{Residual-score marginal results for the ten ADMET assays omitted from Figure~\ref{fig:real-data}. Each entry reports empirical FDR / power, averaged over 100 calibration--test partitions. CQR-CS retains its quantile score.}
\label{tab:real-data-residual-additional}
\begingroup\scriptsize
\setlength{\tabcolsep}{3pt}
\renewcommand{\arraystretch}{1.12}
\begin{adjustbox}{max width=\textwidth}
\begin{tabular}{llrrrrrr}
\toprule
Assay & $q$ & ICS & Full SC-CS & CS & CS-Bonf & LCP & CQR-CS \\
\midrule
CL microsome mouse & 0.3 & 0.064 / 0.916 & 0.066 / 0.920 & 0.083 / 0.966 & 0.000 / 0.003 & 0.057 / 0.892 & 0.089 / 0.976 \\
 & 0.5 & 0.092 / 0.981 & 0.090 / 0.979 & 0.106 / 0.996 & 0.002 / 0.005 & 0.090 / 0.979 & 0.114 / 1.000 \\
\addlinespace[2pt]
CL microsome rat & 0.3 & 0.078 / 0.889 & 0.077 / 0.886 & 0.096 / 0.937 & 0.000 / 0.000 & 0.059 / 0.824 & 0.093 / 0.918 \\
 & 0.5 & 0.113 / 0.963 & 0.115 / 0.966 & 0.133 / 0.985 & 0.000 / 0.000 & 0.100 / 0.940 & 0.152 / 0.992 \\
\addlinespace[2pt]
CL total dog & 0.3 & 0.043 / 0.947 & 0.043 / 0.950 & 0.054 / 0.973 & 0.000 / 0.000 & 0.041 / 0.941 & 0.041 / 0.938 \\
 & 0.5 & 0.062 / 0.981 & 0.062 / 0.981 & 0.062 / 0.983 & 0.000 / 0.000 & 0.061 / 0.979 & 0.048 / 0.957 \\
\addlinespace[2pt]
CL total monkey & 0.3 & 0.023 / 0.983 & 0.024 / 0.984 & 0.026 / 0.987 & 0.000 / 0.000 & 0.023 / 0.980 & 0.024 / 0.981 \\
 & 0.5 & 0.030 / 0.992 & 0.032 / 0.994 & 0.033 / 0.996 & 0.000 / 0.000 & 0.029 / 0.992 & 0.030 / 0.990 \\
\addlinespace[2pt]
CL total rat & 0.3 & 0.060 / 0.940 & 0.060 / 0.942 & 0.068 / 0.960 & 0.000 / 0.000 & 0.055 / 0.919 & 0.047 / 0.895 \\
 & 0.5 & 0.079 / 0.975 & 0.078 / 0.975 & 0.086 / 0.984 & 0.000 / 0.000 & 0.079 / 0.976 & 0.072 / 0.961 \\
\addlinespace[2pt]
CYP2C8 inhibition & 0.3 & 0.110 / 0.902 & 0.110 / 0.901 & 0.123 / 0.943 & 0.000 / 0.000 & 0.089 / 0.799 & 0.114 / 0.917 \\
 & 0.5 & 0.139 / 0.976 & 0.138 / 0.976 & 0.155 / 0.995 & 0.000 / 0.000 & 0.138 / 0.974 & 0.153 / 0.992 \\
\addlinespace[2pt]
CYP2C9 inhibition & 0.3 & 0.081 / 0.948 & 0.081 / 0.949 & 0.086 / 0.964 & 0.000 / 0.000 & 0.078 / 0.934 & 0.081 / 0.948 \\
 & 0.5 & 0.098 / 0.988 & 0.097 / 0.988 & 0.102 / 0.993 & 0.000 / 0.000 & 0.099 / 0.989 & 0.096 / 0.982 \\
\addlinespace[2pt]
CYP2D6 inhibition & 0.3 & 0.049 / 0.953 & 0.051 / 0.960 & 0.059 / 0.982 & 0.000 / 0.000 & 0.046 / 0.945 & 0.043 / 0.925 \\
 & 0.5 & 0.067 / 0.993 & 0.067 / 0.994 & 0.068 / 0.996 & 0.000 / 0.000 & 0.066 / 0.992 & 0.055 / 0.967 \\
\addlinespace[2pt]
CYP3A4 inhibition & 0.3 & 0.067 / 0.936 & 0.066 / 0.934 & 0.073 / 0.955 & 0.000 / 0.000 & 0.063 / 0.924 & 0.062 / 0.909 \\
 & 0.5 & 0.088 / 0.981 & 0.087 / 0.982 & 0.096 / 0.990 & 0.000 / 0.000 & 0.088 / 0.981 & 0.081 / 0.966 \\
\addlinespace[2pt]
Pgp human & 0.3 & 0.087 / 0.882 & 0.086 / 0.882 & 0.059 / 0.772 & 0.000 / 0.002 & 0.066 / 0.797 & 0.075 / 0.822 \\
 & 0.5 & 0.123 / 0.976 & 0.121 / 0.974 & 0.119 / 0.971 & 0.000 / 0.004 & 0.121 / 0.975 & 0.120 / 0.966 \\
\bottomrule
\end{tabular}
\end{adjustbox}
\endgroup
\end{table*}

\section{Clipped-Score Sensitivity Analysis}
\label{app:clipped-score}
The main experiments use the signed residual score because the base predictor
targets the continuous conditional mean. We additionally consider the clipped
score proposed for threshold selection,
\[
V_{\mathrm{clip}}(x,y)
=M\mathbbm{1}\{y>c\}-s(x),
\]
where \(M=100\) separates the two binary outcomes. The prediction
\(s(x)\) is chosen according to the method:
\[
s(x)=
\begin{cases}
f(x), & \text{CS, CS-Bonf, and LCP},\\
\hat f(x), & \text{ICS},\\
\hat f^{(c_j)}(x), & \text{Full SC-CS, for candidate }j.
\end{cases}
\]
Here $f$ is the base Random Forest prediction, including the prescribed
monotone transformation in the miscalibrated setting; $\hat f$ is the
static isotonic predictor fitted on $\mathcal D_1$; and
$\hat f^{(c_j)}$ is refitted separately for candidate $j$ as described
in Appendix~\ref{Full SC-CS}. Thus, this
experiment changes only the nonconformity score; it does not give the clipped
procedure a separately tuned classifier. Its efficiency depends on how well
\(s(x)\) ranks the exceedance probability
\(\eta(x)=\mathbb P(Y>c\mid X=x)\).

\begin{figure*}[htbp]
    \centering
    \includegraphics[width=\textwidth]{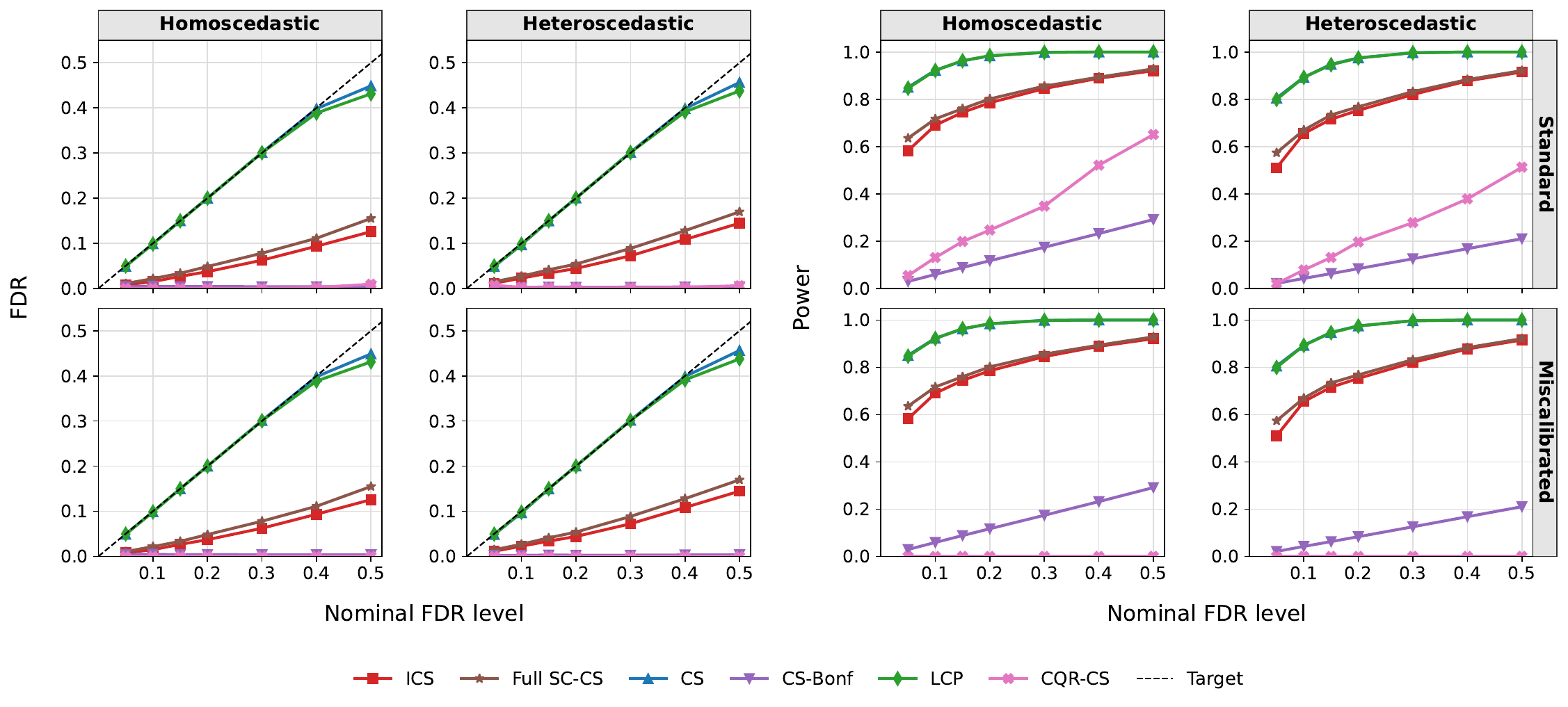}
    \caption{Pooled empirical FDR and power of the clipped-score procedures
    under homoscedastic and heteroscedastic noise, using standard and
    monotonically transformed base predictions. CQR-CS retains its quantile
    score.}
    \label{fig:selection-clipped-simulation}
\end{figure*}

Figure~\ref{fig:selection-clipped-simulation} reports the corresponding
clipped-score results under the same simulation settings. At \(q=0.30\), ICS
attains powers \(0.845\) and \(0.820\) under homoscedastic and heteroscedastic
noise, with pooled empirical FDR \(0.062\) and \(0.072\), respectively.
Clipping is particularly favorable to the pooled power of unstratified CS,
which reaches nearly one with pooled empirical FDR around \(0.30\).
The clipped score assigns non-actives the score \(-s(x)\) and actives
the score \(M-s(x)\), making the resulting selection sensitive to how
well \(s(x)\) ranks candidates relative to the activity threshold.
Unstratified CS applies a single BH procedure across all candidates,
whereas ICS applies BH separately within each prediction stratum.
Thus, the pooled FDR criterion for CS allows errors to be distributed
unevenly across prediction regions, while ICS targets error control
within each stratum. In these experiments, ICS exhibits modestly lower pooled power and substantially lower pooled empirical FDR.

\begin{figure*}[htbp]
    \centering
    \includegraphics[width=\textwidth]{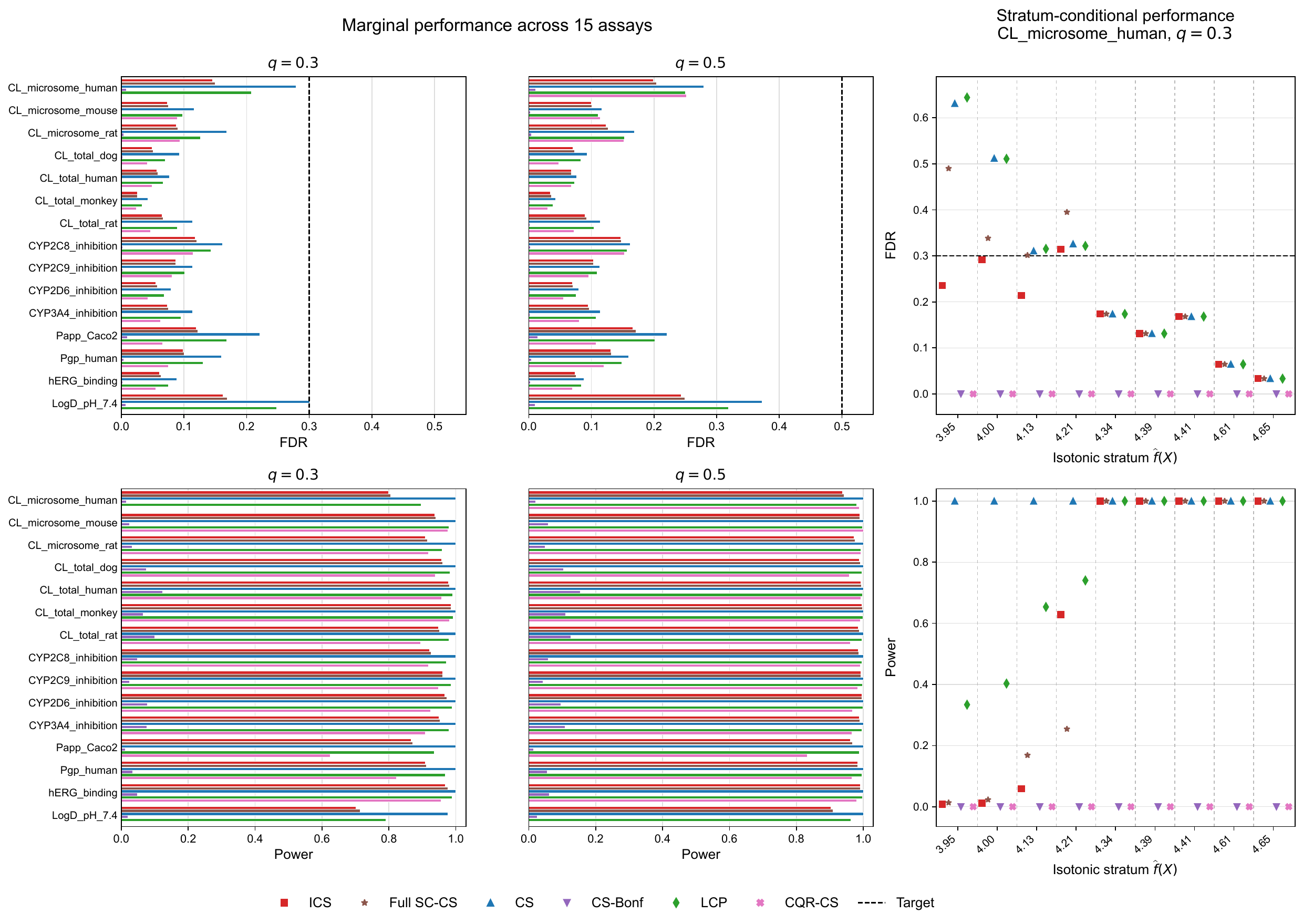}
    \caption{Clipped-score sensitivity analysis on the ADMET data. The left
    panels report marginal FDR and power across all 15 assays; the right panels
    report empirical stratum-conditional FDR and power for
    CL\_microsome\_human at \(q=0.3\). Only strata with average
    size greater than 125 and at least one nonzero entry are displayed.
    Method-specific symbols are grouped by stratum, with vertical gray
    dashed lines separating strata; black dashed lines denote the nominal
    FDR level. CQR-CS retains its quantile
    score.}
    \label{fig:clipped-real-data}
\end{figure*}

The right panels of Figure~\ref{fig:clipped-real-data} show nine informative
strata with average bootstrap sample size greater than 125. At \(q=0.3\),
ICS has empirical conditional FDR above the nominal level in one of these
strata, compared with four strata each for Full SC-CS, CS, and LCP. Their
maximum displayed FDR values are \(0.315\), \(0.490\), \(0.631\), and
\(0.644\), respectively. The slight ICS exceedance occurs at \(\hat f(X)\approx4.214\), where its empirical conditional FDR is \(0.315\) and its power is \(0.628\). Because Theorem~\ref{thm:ics-fdr} averages over \(\mathcal D_2\) while this evaluation holds a single realization of \(\mathcal D_2\) fixed, such a small deviation is expected under finite-sample variability. CS selects every observation in all nine displayed strata, giving power
one throughout, but with substantial FDR exceedances in the lower-prediction
strata. In the five higher-prediction strata, ICS, Full SC-CS, CS, and LCP
all select the entire stratum and have identical performance: power one
and empirical conditional FDR between \(0.033\) and \(0.174\). Their
differences are concentrated in the four lower-prediction strata, where
ICS and Full SC-CS generally make fewer discoveries than CS and LCP.
CS-Bonf and CQR-CS make no discoveries in the nine displayed strata.
Since the base regressor was tuned for squared-error prediction rather
than classification of \(Y>c\), this is a score-sensitivity analysis rather
than an optimized clipped-score procedure.
\end{document}